\documentclass[twoside,11pt]{article}
\usepackage{jmlr2e}
\usepackage{natbib}
\usepackage{lipsum,booktabs}
\usepackage{amsmath,mathrsfs,amssymb,amsfonts,bm,enumitem}
\usepackage{rotating}
\usepackage{pdflscape}
\usepackage{tcolorbox}
\usepackage{hyperref,url,dsfont,nicefrac}
\hypersetup{
    colorlinks,
    breaklinks,
    linkcolor = blue,
    citecolor = blue,
    urlcolor = black,
}
\allowdisplaybreaks
\usepackage{appendix}
\usepackage{multirow,makecell,tabularx}

\usepackage{algorithmic,algorithm}

\renewcommand{\tilde}{\widetilde}
\renewcommand{\hat}{\widehat}
\newcommand\given[1][]{\:#1\vert\:}
\newcommand\givenn[1][]{\:#1\middle\vert\:}

\def \A {\mathcal{A}}
\def \B {\mathbb{B}}
\def \B {\mathcal{B}}

\def \D {\mathcal{D}}
\def \E {\mathbb{E}}

\def \H {\mathcal{H}}
\def \I {\mathcal{I}}

\def \M {\mathcal{M}}

\def \O {\mathcal{O}}

\def \R {\mathbb{R}}

\def \T {\top}

\def \W {\mathcal{W}}
\def \X {\mathcal{X}}
\def \Y {\mathcal{Y}}

\def \a {\mathbf{a}}

\def \f {\tilde{f}}
\def \g {\mathbf{g}}

\def \p {\mathbf{p}}

\def \u {\mathbf{u}}
\def \v {\mathbf{v}}
\def \w {\mathbf{w}}
\def \x {\mathbf{x}}
\def \y {\mathbf{y}}
\def \z {\mathbf{z}}

\def \wh {\hat{\w}}

\def \Rcal {\mathcal{R}}
\def \Ot {\tilde{\O}}

\def \ellb {\boldsymbol{\ell}}

\def \meta {\mathtt{meta}\mbox{-}\mathtt{regret}}

\usepackage{mathtools}
   \let\norm\undefined \DeclarePairedDelimiter\norm{\lVert}{\rVert}
\DeclarePairedDelimiter\abs{\lvert}{\rvert}
\newcommand\inner[2]{\langle #1, #2 \rangle}
\newcommand\ceil[1]{\lceil #1 \rceil}

\DeclareMathOperator*{\argmin}{arg\,min}

\let\proof\relax

\usepackage{amsthm}
\let\proof\relax

\newenvironment{proof}{\par\noindent{\bf Proof\ }}{\hfill\BlackBox\\[2mm]}

\newtheorem{myThm}{Theorem}

\newtheorem{myCor}[myThm]{Corollary}

\newtheorem{myLemma}[myThm]{Lemma}

\newtheorem{myProp}[myThm]{Proposition}

\theoremstyle{definition}
\newtheorem{myAssumption}{Assumption}

\newtheorem{myDef}{Definition}

\newtheorem{myRemark}{Remark}

\usepackage{graphicx,subfigure,color} 

\definecolor{wine_red}{RGB}{228,48,64}
\definecolor{DSgray}{cmyk}{0,1,0,0}

\def \ellb {\boldsymbol{\ell}}
\def \p {\boldsymbol{p}}
\def \meta {\mathtt{meta}\text{-}\mathtt{regret}}
\def \base {\mathtt{base}\text{-}\mathtt{regret}}
\def \epsilon {\varepsilon}
\def \SC {Scream.Control}
\def \define {\triangleq}
\def \Ah {\hat{A}}
\def \Bh {\hat{B}}
\def \Sh {\hat{S}}
\def \wh {\hat{w}}
\def \Psih {\hat{\Psi}}
\def \Rcal {\psi}

\def \Gmeta {G_{\text{meta}}}
\def \wcirc {\mathring{\w}}
\def \vcirc {\mathring{\v}}

\newcommand{\op}{\mathrm{op}}
\newcommand{\loneop}{\mathrm{\ell_1,op}}

\newcommand \ind[1]{\mathbf{1}_{#1}}
\newcommand \term[1]{\mathtt{term}~(\mathtt{#1})}
\newcommand \DR[2]{\D_{\psi}(#1, #2)}

\newcommand \sbr[1]{\left( #1 \right)}
\newcommand \mbr[1]{\left[ #1 \right]}
\DeclarePairedDelimiter\Fnorm{\lVert}{\rVert_{\mathrm{F}}}
\DeclarePairedDelimiter\opnorm{\lVert}{\rVert_{\op}}
\DeclarePairedDelimiter\Loneopnorm{\lVert}{\rVert_{\loneop}}
 
\usepackage{lastpage}
\jmlrheading{24}{2023}{1-\pageref{LastPage}}{3/22}{6/23}{22-0218}{Peng Zhao, Yu-Hu Yan, Yu-Xiang Wang, and Zhi-Hua Zhou}
\ShortHeadings{Non-stationary Online Learning with Memory and Non-stochastic Control}{Zhao, Yan, Wang, and Zhou}
\firstpageno{1}

\begin{document}
\title{Non-stationary Online Learning with Memory and
Non-stochastic Control}

\author{\name Peng Zhao \email zhaop@lamda.nju.edu.cn \\
     \addr National Key Laboratory for Novel Software Technology\\ 
   Nanjing University, Nanjing 210023, China \AND
   \name Yu-Hu Yan \email yanyh@lamda.nju.edu.cn \\
     \addr National Key Laboratory for Novel Software Technology \\
     Nanjing University, Nanjing 210023, China\AND
   \name Yu-Xiang Wang \email yuxiangw@cs.ucsb.edu \\
   \addr Department of Computer Science\\
   University of California, Santa Barbara, CA 93106, USA \AND
   \name Zhi-Hua Zhou \email zhouzh@lamda.nju.edu.cn \\
     \addr National Key Laboratory for Novel Software Technology\\
   Nanjing University, Nanjing 210023, China}
\editor{Shipra Agrawal}
\maketitle

\begin{abstract}We study the problem of Online Convex Optimization (OCO) with memory, which allows loss functions to depend on past decisions and thus captures temporal effects of learning problems. In this paper, we introduce \emph{dynamic policy regret} as the performance measure to design algorithms robust to non-stationary environments, which competes algorithms' decisions with a sequence of changing comparators. We propose a novel algorithm for OCO with memory that provably enjoys an \emph{optimal} dynamic policy regret in terms of time horizon, non-stationarity measure, and memory length. The key technical challenge is how to control the \emph{switching cost}, the cumulative movements of player's decisions, which is neatly addressed by a novel switching-cost-aware online ensemble approach equipped with a new meta-base decomposition of dynamic policy regret and a careful design of meta-learner and base-learner that explicitly regularizes the switching cost. The results are further applied to tackle non-stationarity in \emph{online non-stochastic control}~\citep{ICML'19:online-control}, i.e., controlling a linear dynamical system with adversarial disturbance and convex cost functions. We derive a novel gradient-based controller with dynamic policy regret guarantees, which is the first controller provably competitive to a sequence of changing policies for online non-stochastic control.
\end{abstract}

\section{Introduction}
\label{sec:introduction}
Online Convex Optimization (OCO)~\citep{book'12:Shai-OCO,book'16:Hazan-OCO} is a versatile model of learning in adversarial environments, which can be regarded as a sequential game between a player and an adversary (environments). At each round, the player makes a prediction from a convex set $\w_t \in \W \subseteq \R^d$, the adversary simultaneously selects a convex loss $f_t: \W \mapsto \R$, and the player incurs a loss $f_t(\w_t)$. The goal of the player is to minimize the cumulative loss. The framework is found useful in a variety of disciplines including learning theory, game theory, and  optimization, etc~\citep{book/Cambridge/cesa2006prediction}.

The standard OCO framework considers only \emph{memoryless} adversary, in the sense that the resulting loss is only determined by the player's current prediction without involving past ones. In real-world applications, particularly those related to online decision making, it is often the case that past predictions/decisions would also contribute to the current loss, which makes the standard OCO framework not viable. To remedy this issue, Online Convex Optimization with Memory (OCO with Memory) was proposed as a simplified and elegant model to capture the temporal effects of learning problems~\citep{TIT'02:OCOwithMemory,NIPS'15:OCOmemory}. Specifically, at each round, the player makes a prediction $\w_t \in \W$, the adversary chooses a loss function $f_t: \W^{m+1} \mapsto \R$, and the player will then suffer a loss $f_t(\w_{t-m},\ldots,\w_t)$. Notably, now the loss function depends on both current and past predictions. The parameter $m$ is the memory length, and evidently the OCO with memory model reduces to the standard memoryless OCO when memory length $m=0$. The performance measure for OCO with memory is \emph{policy regret}~\citep{ICML'12:policy-regret}, defined as 
\begin{equation}
\label{eq:policy-Reg}
  \textnormal{Regret}_T = \sum_{t=1}^T f_t(\w_{t-m:t}) - \min_{\v \in  \W} \sum_{t=1}^T f_t(\v,\ldots,\v),
\end{equation}
where throughout the paper we adopt the notation $\a_{i:j}$ to denote the vector sequence $\a_i,\ldots,\a_j$. We start the index from $1$ for convenience. Recent studies apply online learners with provable low policy regret to a variety of related problems~\citep{COLT'18:smoothed-OCO,ICML'19:online-control,ALT'19:DanielyMansour,NIPS'20:chen-switch-constrain}. However, the policy regret~\eqref{eq:policy-Reg} only measures the performance versus a \emph{fixed} comparator and is thus not suitable for learning in non-stationary and open environments~\citep{book/mit/sugiyama2012machine,zhou:openML}. For instance, in the recommendation system, the users' interest may change when looking through the product pages; in the traffic flow scheduling, the traffic network pattern changes throughout the day. Therefore, it is necessary to design online decision-making algorithms with robustness to non-stationary environments. To this purpose, we introduce the \emph{dynamic policy regret} to guide algorithm design, measuring the competitive performance against an arbitrary sequence of \emph{time-varying} comparators $\v_1,\ldots,\v_T \in \W$, defined as 
\begin{equation}
\label{eq:dynamic-policy-Reg}
  \textnormal{D-Regret}_T(\v_{1:T}) = \sum_{t=1}^T f_t(\w_{t-m:t}) - \sum_{t=1}^{T} f_t(\v_{t-m:t}).
\end{equation}
The upper bound of $\textnormal{D-Regret}_T(\v_{1:T})$ should be a function of the comparator sequence $\v_{1:T}$, while the algorithm is agnostic to the choice of comparators. The proposed measure is very general---it subsumes  static policy regret~\eqref{eq:policy-Reg} as a special case when comparators become the best predictor in hindsight, i.e., $\v_{1:T} = \v^* \in \argmin_{\v \in  \W} \sum_{t=1}^T f_t(\v,\ldots,\v)$. Therefore, dynamic policy regret is a more stringent measure than standard policy regret and algorithms that optimize it are more robust to non-stationary environments.

The fundamental challenge of dynamic policy regret optimization is how to simultaneously compete with all comparator sequences with vastly different levels of non-stationarity. Our approach builds upon recent advance of non-stationary online learning~\citep{NIPS'18:Zhang-Ader,NIPS'20:sword, JMLR:sword++} to hedge the uncertainty via the meta-base online ensemble structure, along with several new ingredients specifically designed for the OCO with memory setting. In particular, it is essential to control the \emph{switching cost} for OCO with memory, the cumulative movement of player's predictions. The amount is relatively easy to control in static policy regret~\citep{NIPS'15:OCOmemory}, yet becomes much harder in dynamic policy regret and could even scale linearly due to the meta-base online ensembles structure. Intuitively, online algorithms minimizing dynamic regret necessitate maintaining a certain  probability of aggressive movement to catch up with potential changes within non-stationary environments, which results in tensions between dynamic regret and switching cost. We elegantly address the difficulty by proposing a \emph{switching-cost-aware online ensemble} approach. Our approach features a novel meta-base decomposition of dynamic policy regret and a switching-cost-regularized surrogate loss, which avoids directly handling switching cost altogether but regularizes the switching cost to meta-learner and base-learner instead. Our proposed online-ensemble algorithm provably enjoys an \emph{optimal} $\O(\sqrt{T(1+P_T)})$ dynamic policy regret, where $P_T = \sum_{t=2}^{T} \norm{\v_{t-1} - \v_t}_2$ denotes the unknown path length of comparators. As a byproduct, our result can serve as a solution for minimizing dynamic regret of \emph{online convex optimization with switching cost}, a variant of classic OCO setting by penalizing switching cost of returned decisions~\citep{MLJ'99:transaction-cost,COLT'14:higher-switching,COLT'18:smoothed-OCO}. Specifically, consider the OCO problem with online functions $h_1,\ldots,h_T$ with $h_t: \W \mapsto \R$. Denote by $\w_1,\ldots,\w_T$ the returned decisions by our algorithm. Then, we have $\sum_{t=1}^T h_t(\w_t) - \sum_{t=1}^T h_t(\v_t) + \lambda \sum_{t=2}^T \norm{\w_t - \w_{t-1}}_2 \leq \O(\sqrt{\lambda T(1+P_T)})$, where $P_T$ is the path length as defined above. We also establish the lower bound to show the \emph{minimax optimality} in terms of switching-cost coefficient $\lambda$, time horizon $T$, and path length $P_T$. Compared to our conference paper~\citep{AISTATS'22:scream}, the current result improves the dependence in the memory parameter $\lambda$ to be optimal, which is achieved via a novel usage of the laze update mechanism.
 
The results of OCO with memory yield an important application in online decision-making problems. Specifically, we investigate the problem of \emph{online non-stochastic control}~\citep{ICML'19:online-control}, i.e., controlling a linear dynamical system with adversarial (non-stochastic) disturbance and adversarial convex cost functions. Online non-stochastic control has attracted much recent research attention due to its relaxed assumptions on  disturbances and flexibility of cost functions. Existing studies mainly focus on optimizing static policy regret, whereas the optimal controller of each round would naturally change over iterations since the disturbances and cost functions both change adversarially. Therefore, it is necessary to investigate \emph{dynamic policy regret}, which competes controller's performance with time-varying benchmark controllers. By adopting the ``disturbance-action'' policy parameterization~\citep{ICML'19:online-control}, online non-stochastic control is reduced to OCO with memory, and thus its dynamic policy regret can be optimized by a similar meta-base online ensemble structure as developed before. Our designed controller attains an $\Ot(\sqrt{T(1+P_T)})$ dynamic policy regret, where $P_T$ measures the fluctuation of compared controllers. To the best of our knowledge, this is the first controller competitive to a sequence of changing ``disturbance-action'' policies. Given that our techniques for OCO with memory provide a provable way to handle the memory effects of past decisions, we anticipate that they would have broader applications in online decision-making problems.

The main contributions of this paper are summarized as follows.
\begin{itemize}[noitemsep,nolistsep]
  \item We introduce \emph{dynamic policy regret} as the performance measure to guide the algorithm design of OCO with memory and online non-stochastic control to enhance the robustness of online algorithms to non-stationary environments.
  \item We propose a novel algorithm for OCO with memory, which enjoys an \emph{optimal} dynamic policy regret of order $\O(\sqrt{T(1+P_T)})$. To achieve this, several key algorithmic ingredients are designed to handle unknown environments and control switching cost.
  \item The results are further applied to the problem of online non-stochastic control, yielding an online controller with $\Ot(\sqrt{T(1+P_T)})$ dynamic policy regret, which is the first online controller competitive with a sequence of \emph{time-varying} policies.
\end{itemize}

In the following, we first review related works in Section~\ref{sec:appendix-related-work} and then introduce some preliminaries in Section~\ref{sec:problem}. Next, we present the main results for OCO with memory and online non-stochastic control in Section~\ref{sec:OCOwMemory} and Section~\ref{sec:online-control}. Section~\ref{sec:appendix-experiment} reports the experiments. We finally conclude the paper in Section~\ref{sec:conclusion}. All the proofs are included in appendices. \section{Related Work}
\label{sec:appendix-related-work}
In this section, we briefly discuss related works on OCO with memory, online non-stochastic control, and dynamic regret minimization for online learning.

\paragraph{OCO with Memory.} OCO with memory is initiated by~\citet{TIT'02:OCOwithMemory}, who prove an $\O(T^{2/3})$ policy regret by a blocking technique. Later,~\citet{NIPS'15:OCOmemory} propose a simple gradient-based algorithm that provably achieves $\O(\sqrt{T})$ and $\O(\log T)$ policy regret for convex and strongly convex functions, respectively. Recent study discloses that the policy regret of OCO with memory over exp-concave functions is at least $\Omega(T^{1/3})$~\citep[Theorem 2.3]{NIPS'20:max-control}. One of the key concepts of OCO with memory is \emph{switching cost}, the cumulative movement of decisions, which is also concerned in smoothed online learning~\citep{COLT'18:smoothed-OCO,NIPS'19:smoothed-OCO,AISTATS'19:smooth-LQR}, online learning with switching budget~\citep{COLT'18:switch-cost,NIPS'20:chen-switch-constrain,COLT'21:lazy-OCO,NeurIPS'21:OCO-CSC}. Online learning with memory is also studied in the prediction with expert advice setting~\citep{COLT'10:OCOmemory,TIT'14:Neu-memory,NIPS'13:switch-cost,COLT'18:switch-cost} and bandit settings~\citep{ICML'12:policy-regret,STOC'14:bandit-switch-cost,COLT'18:switch-cost,NIPS'19:graph-switch-cost}.

\paragraph{Online Non-stochastic Control.} Recently, there is a surge of interest to apply modern statistical and algorithmic techniques to the control problem. 
Online non-stochastic control is proposed by~\citet{ICML'19:online-control}, where the regret is chosen as the performance measure and the disturbance is allowed to be adversarially chosen. When online cost functions are convex and Lipschitz,~\citet{ICML'19:online-control} obtain an $\O(\sqrt{T})$ policy regret for known linear dynamical system by introducing the DAC parameterization and reducing the problem to OCO with memory.~\citet{ALT'20:control-Hazan} show an $\O(T^{2/3})$ policy regret for unknown system via system identification. In addition,~\citet{ICML'20:log-control} propose the online learning with advantages technique and obtain logarithmic regret for known system with quadratic cost and adversarial disturbance, whose results are strengthened by~\citet{NIPS'20:max-control} to accommodate arbitrary changing costs. All mentioned results are developed for fully observed system, and~\citet{COLT'20:control-improper} present a clear picture for non-stochastic control with partially observed systems. We are still witnessing a variety of recent advances, for example, non-stochastic control with bandit feedback~\citep{NIPS'20:bandit-control,NIPS'20:bandit-Koren}, adaptive regret minimization~\citep{arXiv'20:adaptive-OCOmemory,AISTATS'22:SA-OCOwMemory,NIPS'22:CA-OCOwMemory}, etc. We will present more discussions on the relationship between these works for adaptive regret minimization and our work (for dynamic regret minimization) at the end of this section. There are other related works studying non-stationary online control from the lens of competitive ratio~\citep{NIPS'20:smoothed-OCO,TAC'22:competitive-control} and robust control~\citep{arXiv'20:Goel,L4DC'22:Goel}. In addition, there have been considerable efforts dedicated to the broader field of online (stochastic) control over the past several decades. While only a handful can be mentioned here~\citep{guo1995performance, COLT'97:PACcontrol, COLT'11:LQC,ICML'18:Cohen,DeanMMRT:Control20,COLT'22:efficientOLC,NIPS'22:optimalALC}, interested readers can refer to the references therein to explore more recent developments in this area.

\paragraph{Dynamic Regret.} Benchmarking the regret in term of changing comparators dates back to early development of prediction with expert advice~\citep{journals/ml/HerbsterW98,JMLR'01:Herbster}, in which they studied a special form of dynamic regret that supports the comparators change for at most $S$ times (often referred to as $S$-tracking/shifting/switching regret)~\citep{journals/ml/HerbsterW98,JMLR'01:Herbster,JMLR'02:bousquet-dynamic,conf/nips/Cesa-BianchiGLS12,ICML'16:GyorgyS-shiftregret,NIPS'16:Wei-non-stationary-expert,NIPS'19:Zheng,COLT'22:corral_T}. For online convex optimization, \citet{ICML'03:zinkvich} pioneers the study of dynamic regret and shows that OGD can attain an $\O(\sqrt{T}(1+P_T))$ dynamic regret.~\citet{ICML'18:zhang-dynamic-adaptive} show that the minimax lower bound is $\Omega(\sqrt{T(1+P_T)})$ and close the gap by proposing an algorithm with an $\O(\sqrt{T(1+P_T)})$ regret. Recent works achieve problem-dependent guarantees by exploiting smoothness and incorporating the optimistic online learning techniques~\citep{NIPS'20:sword,JMLR:sword++}, and other works obtain an improved rate by exploiting exp-concavity or strong convexity~\citep{COLT'21:baby-strong-convex,AISTATS'22:sc-proper}. More results for dynamic regret minimization have been developed in bandit convex optimization~\citep{JMLR'21:BCO}, Markov decision processes~\citep{ICML'22:mdp}, online label shift problems~\citep{NeurIPS'22:label_shift,arxiv'23:baby-OLS}, time-varying games~\citep{ICML'22:TVgame,ICML'23:SMontoneGame}, etc. We note that the dynamic regret measure studied in this paper is also called the \emph{universal} dynamic regret, in the sense that the regret guarantee holds universally against any comparator sequence in the domain. Another special variant called the \emph{worst-case} dynamic regret is frequently studied in the literature~\citep{OR'15:dynamic-function-VT,AISTATS'15:dynamic-optimistic,CDC'16:dynamic-sc,NIPS'17:zhang-dynamic-sc-smooth,NIPS'19:Wangyuxiang,UAI'20:simple,L4DC'21:sc_smooth}, which specifies comparators as the optimizers of online functions. The worst-case dynamic regret is less general than the universal one. Indeed, both worst-case dynamic regret and static regret are special cases of the universal dynamic regret with different choices of comparators, and we refer the reader to~\citep{JMLR:sword++} for more elaborations.

\paragraph{More Discussions.} Online non-stochastic control in non-stationary environments is also recently studied via the measure of \emph{adaptive regret}~\citep{ICML'09:Hazan-adaptive,ICML'15:Daniely-adaptive}---the regret compared to the best policy on any interval in the time horizon.~\citet{arXiv'20:adaptive-OCOmemory} propose the first controller with an $\Ot(\sqrt{T})$ expected adaptive regret on any interval in the total horizon. The result is strengthened in a recent work (concurrent to our paper)~\citep{AISTATS'22:SA-OCOwMemory}, which presents a strongly adaptive controller with an $\Ot(\sqrt{\abs{\I}})$ deterministic adaptive regret on any interval $\I \subseteq [T]$. The two papers and our work all study non-stationary online control, however, the concerned measures and used techniques are completely different. \textbf{(1)}~Measures: dynamic regret examines the global behavior to ensure a competitive performance with time-varying compared polices, whereas adaptive regret focuses on the local behavior with respect to a fixed strategy. Even though a black-box reduction from dynamic regret to adaptive regret has been known in the simpler setting of prediction with expert advice (i.e., online linear optimization over the simplex)~{\citep[Theorem~4]{COLT'15:Luo-AdaNormalHedge}}, the relationship between strongly adaptive regret and universal dynamic regret for online convex optimization over the general setup~\citep[Section~5]{IJCAI:2020:Zhang} remains highly \emph{unclear}, which is even more vague when further taking the switching cost into account. \textbf{(2)}~Techniques: optimizing either dynamic regret or adaptive regret requires the meta-base online ensemble structure to deal with uncertainty of the non-stationary environments. However, the specific techniques, especially the way to control switching cost, exhibit significant difference.~\citet{arXiv'20:adaptive-OCOmemory} leverage the Follow-the-Leading-History framework~\citep{ICML'09:Hazan-adaptive} with a shrinking technique~\citep{COLT'10:OCOmemory} to keep previous experts unchanged with a certain probability to reduce the switching cost, so their result holds in expectation only. The improved result of $\O(\sqrt{\abs{\I}})$ deterministic strongly adaptive regret bound~\citep{AISTATS'22:SA-OCOwMemory} is achieved by a very different framework drawn inspirations from parameter-free online learning~\citep{ICML'20:Ashok}. By contrast, the key ingredients of our approach are the novel meta-base decomposition and the switching-cost-regularized loss, which avoid explicitly handling the switching cost of final decisions but directly control the switching cost of meta-algorithm and individual base-algorithm. These mechanisms finally lead to a deterministic dynamic policy regret guarantee for our methods.  \section{Preliminaries}
\label{sec:problem}
This section introduces preliminaries for online convex optimization (OCO) with memory.

\paragraph{Problem Setup.} OCO with memory is a variant of standard OCO framework to capture the long-term effects of past decisions, whose protocol is shown below.

\begin{algorithmic}[1]
  \FOR{$t = m+1, \ldots, T$}
    \STATE the player chooses a decision $\w_t \in \W$;
    \STATE the adversary reveals the loss $f_t: \W^{m+1} \mapsto \R$ that applies to last $m+1$ decisions;
    \STATE the player suffers a loss of $f_t(\w_{t-m},\ldots,\w_t)$; 
  \ENDFOR
\end{algorithmic}
In above, $m$ is the memory length, and $f_t: \W^{m+1} \mapsto \R$ is convex in memory, which means its unary function $\f_t(\w) = f_t(\w,\ldots,\w)$ is convex in $\w$. Clearly, OCO with memory recovers the standard memoryless OCO when $m=0$. The standard measure is policy regret~\citep{ICML'12:policy-regret} as defined in~\eqref{eq:policy-Reg}. We introduce a strengthened measure called \emph{dynamic policy regret} to compete with changing comparators as defined in~\eqref{eq:dynamic-policy-Reg}. The dynamic policy regret upper bound usually involves the path length $P_T = \sum_{t=2}^{T} \norm{\v_{t-1} - \v_t}_2$, which measures the variation of comparators and thus captures the environmental non-stationarity. Throughout the paper, $\O(\cdot)$-notation is used to express regret upper bound as a function of $T$ and $P_T$, and $\Ot(\cdot)$-notation omits logarithmic factors in $T$. To make it clear, we mention that the $\O(\cdot)$-notation does not hide $\log \log P_T$ or $\log \log T$ terms, even though they are indeed small.

\paragraph{Assumptions.} Next, we introduce several standard assumptions~\citep{NIPS'15:OCOmemory}. For simplicity we focus on the $\ell_2$-norm and the extension to general primal-dual norms is straightforward.
\begin{myAssumption}[coordinate-wise Lipschitzness]
\label{assume:Lipschitz}
The online function $f_t: \W^{m+1} \mapsto \R$ is $L$-coordinate-wise Lipschitz, i.e., $\abs{f_t(\x_0,\ldots,\x_m) - f_t(\y_0,\ldots,\y_m)} \leq L \sum_{i=0}^{m} \norm{\x_i - \y_i}_2$.
\end{myAssumption}

\begin{myAssumption}[bounded gradient]
\label{assume:bounded-gradient}
The gradient norm of the unary loss is at most $G$, i.e., for all $\w \in \W$ and $t \in [T]$, $\norm{\nabla \f_t(\w)}_2 \leq G$.
\end{myAssumption}

\begin{myAssumption}[bounded domain]
\label{assume:bounded-norm}
The domain $\W$ is convex, closed, and satisfies $\norm{\w -\w'}_2 \leq D$ for all $\w, \w' \in \W$. For convenience, we also assume $\mathbf{0} \in \W$.
\end{myAssumption}

\paragraph{Static Regret of OCO with Memory.} This part briefly reviews the result of static policy regret.~\citet{NIPS'15:OCOmemory} propose a simple approach based on the gradient descent based on the observation that when online functions are coordinate-wise Lipschitz, the policy regret can be upper bounded by the switching cost and the vanilla regret over the unary loss, formally,
\begin{align*}
  \sum_{t=1}^T f_t(\w_{t-m:t}) - \min_{\v \in \W} \sum_{t=1}^T \f_t(\v)  \leq \lambda \sum_{t=2}^T \norm{\w_{t} - \w_{t-1}}_2 + \sum_{t=1}^T \f_t(\w_t) - \min_{\v \in \W} \sum_{t=1}^T \f_t(\v),
\end{align*}
where $\lambda = m^2 L$. The first term is the \emph{switching cost} measuring the cumulative movement of  decisions $\w_{1:T}$ and the remaining term is the standard regret of memoryless OCO. Consequently, it is natural to perform Online Gradient Descent (OGD)~\citep{ICML'03:zinkvich} over the unary loss $\tilde{f}_t$, i.e., $\w_{t+1} = \Pi_{\W}[\w_t - \eta \nabla \f_t(\w_t)]$, where $\eta > 0$ is the step size and $\Pi_{\W}[\cdot]$ denotes the projection onto the nearest point in $\W$. It is well-known that with an appropriate step size OGD enjoys an $\O(\sqrt{T})$ regret in memoryless OCO. Further,~\citet{NIPS'15:OCOmemory} show that the produced decisions move sufficiently slowly. Indeed, switching cost satisfies $\sum_{t=2}^T \norm{\w_{t} - \w_{t-1}}_2 \leq \O(\eta T)$, which will not affect the final regret order by choosing $\eta = \O(1/\sqrt{T})$. Combining both facts yields an $\O(\sqrt{T})$ static policy regret~\citep[Theorem 3.1]{NIPS'15:OCOmemory}. \section{OCO with Memory}
\label{sec:OCOwMemory}
This section presents dynamic policy regret of OCO with memory. We begin with the gentle case when the path length is known, and then handle the general case when it is unknown and present the overall result.

\subsection{A Gentle Start: known path length}

Similar to the static regret analysis mentioned in the last section, we first upper-bound the dynamic policy regret~\eqref{eq:dynamic-policy-Reg} in the following way:
\begin{align}
& \textnormal{D-Regret}_T(\v_{1:T}) \leq \sum_{t=1}^T \f_t(\w_t) -  \sum_{t=1}^T \f_t(\v_t) + \lambda \sum_{t=2}^T \norm{\w_{t} - \w_{t-1}}_2 + \lambda \sum_{t=2}^T \norm{\v_{t} - \v_{t-1}}_2. \label{eq:dynamic-regret-upper-bound}
\end{align}
There are three terms in the upper bound: dynamic regret of unary functions, switching cost of final decisions, and switching cost of comparators. Therefore, it is natural to deploy OGD over unary functions, and we can prove the following dynamic policy regret guarantee. The proof can be found in Appendix~\ref{sec:proof-OCOwMemory-static}.
\begin{myThm}
\label{thm:dynamic-regret-OCOwMemory}
Under Assumptions~\ref{assume:Lipschitz}--\ref{assume:bounded-norm}, running OGD over unary losses $\f_1,\ldots,\f_T$ ensures 
\begin{equation}
  \label{eq:dynamic-OGD-untuned}
  \textnormal{D-Regret}_T(\v_{1:T}) = \sum_{t=1}^T f_t(\w_{t-m:t}) - \sum_{t=1}^{T} f_t(\v_{t-m:t}) \leq \O\Big(\eta T + \frac{1+ P_T}{\eta} + P_T\Big)
\end{equation}
for any comparator sequence $\v_1,\ldots,\v_T \in \W$, where $P_T = \sum_{t=2}^{T} \norm{\v_t - \v_{t-1}}_2$ is the path length measuring fluctuation of the comparator sequence.
\end{myThm}
Suppose the value of path length $P_T$ were known a priori, Theorem~\ref{thm:dynamic-regret-OCOwMemory} indicates an optimal $\O(\sqrt{T(1+P_T)})$ dynamic policy regret by setting step size as $\eta = \O(\sqrt{(1+P_T)/T})$, matching the $\Omega(\sqrt{T(1+P_T)})$ lower bound of memoryless OCO~\citep{NIPS'18:Zhang-Ader}. However, this step size tuning is not realistic because we cannot attain the prior information of path length $P_T = \sum_{t=2}^{T} \norm{\v_{t-1} - \v_t}_2$. Indeed, since the dynamic policy regret measure holds for any comparator sequence $\v_1,\ldots,\v_T$ that can be arbitrarily selected in the feasible domain $\W$, the path length $P_T$ essentially captures the environmental non-stationarity and is \emph{unknown} to the player. In Section~\ref{sec:OCOmemory-challenge}, we will further elucidate the challenge of designing online algorithms that enjoy optimal dynamic policy regret and meanwhile do not require prior knowledge of environmental non-stationarity, especially due to the switching cost arising in OCO with memory. In Section~\ref{sec:new-decompose}, we will present our solution by introducing several novel algorithmic ingredients. Finally, in Section~\ref{sec:optimal-memory} we further improve the algorithm to achieve an optimal memory dependence along with the corresponding lower bound argument to show the minimax optimality of our results.

\subsection{Challenge: unknown path length and switching cost of OCO with memory}
\label{sec:OCOmemory-challenge}
As mentioned in the last paragraph, the fundamental difficulty of attaining optimal dynamic policy regret lies in the infeasible step size tuning that depends on the unknown comparator sequence $\v_1,\ldots,\v_T$. We emphasize that such an unpleasant dependence cannot be removed by the well-known doubling trick~\citep{JACM'97:doubling-trick}, because we cannot monitor the empirical value of path length, $P_t = \sum_{s=2}^t \norm{\v_s - \v_{s-1}}_2$, as comparators $\v_1,\ldots,\v_T$ can be arbitrarily chosen in the feasible domain $\W$ and are entirely unknown to the learner. Similar challenge also emerges in recent studies of memoryless non-stationary online learning~\citep{NIPS'18:Zhang-Ader,NIPS'20:sword}, inspired by which we employ the meta-base online ensemble framework to design a two-layer approach to optimize the dynamic policy regret. Below, we will first briefly review the framework and then elucidate the challenge of its application in OCO with memory, mainly due to the tension between dynamic regret and switching cost, which necessitates additional new ideas. 

\paragraph{Meta-base Online Ensemble Framework.} The framework admits a two-layer structure and is essentially an online ensemble method~\citep{book'12:ensemble-zhou,thesis:zhao2021-eng}. We first need to design an appropriate pool of candidate step sizes $\H = \{\eta_1,\ldots,\eta_N\}$ to ensure the existence of a step size $\eta_{i^*}$ that approximates optimal step size $\eta_*$ well. Then, multiple base-learners $\B_1,\ldots,\B_N$ are maintained, and each performs base-algorithm (for example, OGD) with a step size $\eta_i \in \H$ and generates the decision sequence $\w_{1,i},\w_{2,i},\ldots,\w_{T,i}$. Finally, a meta-learner, supposed to be able to track the best base-learner, is used to combine all intermediate results of base learners to produce final output $\w_{1},\w_{2},\ldots,\w_{T}$, where $\w_t = \sum_{i=1}^{N} p_{t,i} \w_{t,i}$. The final output of meta-base algorithm can well approximate the decision sequence of the best base-learner (the one with  near-optimal step size $\eta_{i^*}$) and thus ensure a good dynamic regret bound. 

Indeed, by employing OGD over unary functions $\f_1,\ldots,\f_T$ and designing a proper step size pool $\H$, it is not hard to prove a dynamic regret bound over unary functions, that is, $\sum_{t=1}^T \f_t(\w_t) -  \sum_{t=1}^T \f_t(\v_t) \leq \O(\sqrt{T(1+P_T)})$. Then, by~\eqref{eq:dynamic-regret-upper-bound} we have
\begin{align*}
  &\textnormal{D-Regret}_T(\v_{1:T}) \leq \O(\sqrt{T(1+P_T)}) + \O(P_T) + \sum_{t=2}^T \norm{\w_{t} - \w_{t-1}}_2.
\end{align*}
So we are in the position to control \emph{switching cost}. Below, we demonstrate that a vanilla deployment of the meta-base method may move too fast to achieve a sublinear switching cost and will ruin the overall policy regret bound, which necessitates additional novel algorithmic ingredients to better balance the dynamic regret and switching cost. 

\paragraph{Switching Cost.} The switching cost is the pivot of the analysis for OCO with memory.~\citet{NIPS'15:OCOmemory} demonstrate that many popular OCO algorithms for static regret minimization naturally produce slow-moving decisions, however, it becomes more difficult in dynamic regret. Intuitively, for dynamic online algorithms, it is necessary to keep some probability of aggressive movement in order to catch up with the potential changes of non-stationary environments, which results in \emph{tensions between dynamic regret and switching cost}. Formally, denote by $\w_t = \sum_{i=1}^{N} p_{t,i} \w_{t,i}$ the final decision returned by the two-layer approach, then the switching cost can be bounded by
\begin{align}
  {} & \sum_{t=2}^{T} \norm{\w_{t} - \w_{t-1}}_2 \leq D \sum_{t=2}^{T} \norm{\p_t - \p_{t-1}}_1 + \sum_{t=2}^{T} \sum_{i=1}^{N} p_{t,i} \norm{\w_{t,i} - \w_{t-1,i}}_2. \label{eq:sc-decompose}
\end{align}
A formal proof is presented in Appendix~\ref{sec:sc-decompose}. In the upper bound, the first term $\sum_{t=2}^{T} \norm{\p_t - \p_{t-1}}_1$ is the switching cost of meta-learner, which is at most $\O(\sqrt{T})$. However, the second term $\sum_{t=2}^{T} \sum_{i=1}^{N} p_{t,i} \norm{\w_{t,i} - \w_{t-1,i}}_2$, the weighted sum of switching cost of all base-learners, becomes the major barrier, which could be very large and even grow linearly over iterations. Specifically, for each base-learner $\B_i$ (OGD with step size $\eta_i$), its switching cost is at most $\O(\eta_i T)$; additionally, to ensure a coverage of the optimal step size, the pool of candidate step sizes is usually set as $\H = \{ \eta_i = \O(2^i \cdot T^{-1/2}), i \in [N]\}$ such that $\eta_1 = \O(T^{-1/2})$ and $\eta_N = \O(1)$. Therefore, the base-learner with larger step sizes would incur unacceptable switching cost, for instance, the switching cost of base-learner $\B_N$ could grow linearly, of order $\O(T)$. As a result, the term $\sum_{t=2}^{T} \sum_{i=1}^{N} p_{t,i} \norm{\w_{t,i} - \w_{t-1,i}}_2$ could be enlarged by base-learners whose step sizes are too large and therefore is difficult to control.

\subsection{Algorithmically Enforcing Low Switching Cost: a new meta-base decomposition}
\label{sec:new-decompose}
To resolve the challenge of switching cost in dynamic policy regret minimization, we propose a novel switching-cost-aware online ensemble approach. Specifically, we start with proposing the following new meta-base regret decomposition to avoid directly controlling switching cost of final predictions or controlling switching cost of every base-learner:
\begin{align}
  & \sum_{t=1}^T \f_t(\w_t) -  \sum_{t=1}^T \f_t(\v_t) + \lambda \sum_{t=2}^T \norm{\w_{t} - \w_{t-1}}_2 \label{eq:new-decompose}\\
  & \leq \sum_{t=1}^T \inner{\nabla \f_t(\w_t)}{\w_t - \v_t} + \lambda D \sum_{t=2}^{T} \norm{\p_t - \p_{t-1}}_1 + \lambda \sum_{t=2}^{T} \sum_{i=1}^{N} p_{t,i} \norm{\w_{t,i} - \w_{t-1,i}}_2 \nonumber\\
  & = \underbrace{\sum_{t=1}^T \big( \inner{\p_t}{\ellb_t}  - \ell_{t,i} \big) + \lambda D \sum_{t=2}^{T} \norm{\p_t - \p_{t-1}}_1}_{\meta}  + \underbrace{\sum_{t=1}^T  \big( g_t(\w_{t,i}) - g_t(\v_t) \big) + \lambda \sum_{t=2}^T \norm{\w_{t,i} - \w_{t-1,i}}_2}_{\base}.\nonumber
\end{align}
The first inequality follows from the convexity of unary functions and switching cost decomposition~\eqref{eq:sc-decompose}, and for convenience we introduce the notation of linearized loss $g_t(\w) = \inner{\nabla \f_t(\w_t)}{\w}$. The second equation is crucial, in which the key ingredient is the introduced \emph{switching-cost-regularized surrogate loss} $\ellb_t \in \R^N$ for the meta-algorithm, defined as 
\begin{equation}
\label{eq:surrogate-loss-meta}
\ell_{t,i} \define g_t(\w_{t,i}) + \lambda \norm{\w_{t,i} - \w_{t-1,i}}_2.
\end{equation}
Intuitively, the base-learner's switching cost is now taken into account when evaluating its performance---the meta-learner will impose more penalty on base-learners with larger switching cost. Technically, the key improvement upon previous analysis in~\eqref{eq:sc-decompose} lies in the switching cost term of the base-learner: we now only need to bound switching cost of a single base-learner $\sum_{t=2}^{T} \norm{\w_{t,i} - \w_{t-1,i}}_2$, which is to be contrasted to the switching cost of all the base-learners $\sum_{t=2}^{T} \sum_{i=1}^{N} p_{t,i} \norm{\w_{t,i} - \w_{t-1,i}}_2$. 

Furthermore, noting that the new meta-base decomposition~\eqref{eq:new-decompose} holds simultaneously for \emph{any} index $i \in [N]$, we can therefore choose the compared index as $i^*$ (the one with near-optimal step size) and the switching cost of this base-learner $\B_{i^*}$ is at most $\O(\eta_{i^*} T) = \O(\sqrt{T(1+P_T)})$. In other words, we successfully escape from those base-learners with unacceptably large step sizes, whose switching cost is too large to tolerate. 

Consequently, we can tackle switching cost in the meta-base methods with the help of the switching-cost-regularized technique. The rest is more or less standard. Specifically, the meta-base regret decomposition indicates the following requirements on the base-algorithm and meta-algorithm: 
\begin{itemize}
  \item base-algorithm needs to achieve low dynamic regret over unary functions and tolerate its own switching cost $\sum_{t=2}^T \norm{\w_{t,i} - \w_{t-1,i}}_2$;
  \item meta-algorithm needs to optimize the switching-cost-regularized loss to impose more penalty on base-learners with larger switching cost, and tolerate its own switching cost $\sum_{t=2}^T \norm{\p_{t} - \p_{t-1}}_1$.
\end{itemize}

Below, we outline the specific configurations of our switching-cost-aware online ensemble approach (including settings of step size pool, base-algorithm, and meta-algorithm) to fulfill  above requirements. 
\paragraph{Step Size Pool.} We initiate $N = \left\lceil \frac{1}{2} \log_2(1+T) \right\rceil + 1 = \O(\log T)$ base-learners, with step size pool set as
\begin{equation}
  \label{eq:step-size-pool}
  \H = \left\{\eta_i \givenn \eta_i = 2^{i-1}\cdot \sqrt{\frac{D^2}{(\lambda G + G^2)T}},\ i \in [N] \right\}.
\end{equation}

\paragraph{Base-algorithm.} The base-algorithm is chosen as OGD running over the linearized loss $\{g_t\}_{t=1:T}$. The switching cost of each base-learner can be safely controlled, as indicated by Theorem~\ref{thm:dynamic-regret-OCOwMemory}. More specifically, there are $N$ base-learners denoted by $\B_1,\ldots,\B_N$ and the base-learner $\B_i$ (with step size $\eta_i \in \H$) performs
\begin{equation*}
  \w_{t+1,i} = \Pi_{\W}[\w_{t,i} - \eta_i \nabla g_t(\w_{t,i})] = \Pi_{\W}[\w_{t,i} - \eta_i \nabla \tilde{f}_t(\w_t)].
\end{equation*}
The second equation is from $g_t(\w) = \inner{\nabla \f_t(\w_t)}{\w}$ and the update exhibits the computational advantage due to linearization: although multiple base-learners are performed, they share the same gradient and thus the algorithm only calculates one gradient per iteration, rather than $N$ gradients as was anticipated.

\paragraph{Meta-algorithm.} The meta-algorithm is set as the well-known Hedge algorithm~\citep{JCSS'97:boosting} \emph{running over the switching-cost-regularized loss}. The weight $\p_{t+1} \in \Delta_N$ is updated by $p_{t+1,i} \propto p_{t,i} \exp(-\epsilon \ell_{t,i})$, where $\ellb_{t} \in \R^N$ is the switching-cost-regularized surrogate loss defined in~\eqref{eq:surrogate-loss-meta} and $\epsilon > 0$ is the learning rate. Then, the meta-regret $\sum_{t=1}^T \big( \inner{\p_t}{\ellb_t} - \ell_{t,i} \big) + \lambda D \sum_{t=2}^{T} \norm{\p_t - \p_{t-1}}_1$, essentially the static regret with switching cost, can be well controlled with $\epsilon = \O(\sqrt{1/T})$. For technical reasons, we adopt a non-uniform initialization by setting $\p_1 \in \Delta_N$ with $p_{1,i} \propto 1/(i^2 + i)$. The dependence of learning rate on $T$ can be removed by either a time-varying tuning or doubling trick.

We finally remark that base-algorithm (OGD) and meta-algorithm (Hedge) can be understood in a unified view from the aspect of Online Mirror Descent (OMD)~\citep{book'93:nemirovskij-yudin,book'12:Shai-OCO,NIPS'11:universal-OMD}. OMD is a powerful online method accommodating general geometries and both OGD and Hedge are its special instances. We can generalize the dynamic policy regret of Theorem~\ref{thm:dynamic-regret-OCOwMemory} from OGD to OMD, and this can be used to extend all the results in this paper from $\ell_2$-norm to general primal-dual norms. More descriptions are supplied in Appendix~\ref{sec:appendix-proof-OMD}.

\begin{algorithm}[!t]
   \caption{\textbf{Scream}}
   \label{alg:scream}
\begin{algorithmic}[1]
  \REQUIRE{step size pool $\H = \{\eta_1,\ldots,\eta_N\}$, learning rate of meta-algorithm $\epsilon$}
  \STATE{Initialization: $\w_{1:m} \in \W$, $\w_{m,i} \in \W$, $\forall i \in [N]$; $\p_m \in \Delta_N$ with $p_{m,i} \propto 1/(i^2+i)$, $\forall i\in[N]$\\}
    \FOR{$t=m+1$ {\bfseries to} $T$}
      \STATE {Receive $\w_{t,i}$ from base-learner $\B_i$ for $i \in [N]$}
      \STATE {Submit the decision $\w_t = \sum_{i=1}^{N} p_{t,i}\w_{t,i}$}
      \STATE {Suffer a loss of $f_t(\w_{t-m},\ldots,\w_t)$}
      \STATE {Observe the online function $f_t:\W^{m+1} \mapsto \R$ that applies to last $m + 1$ decisions}\
      \STATE {Construct the linearized loss by $g_t(\w) = \inner{\nabla \f_t(\w_t)}{\w}$}
      \STATE {Construct the switching-cost-regularized loss $\ellb_t \in \R^N$ with $\ell_{t,i} = g_t(\w_{t,i}) + \lambda \norm{\w_{t,i} - \w_{t-1,i}}_2$ for $i \in [N]$}
      \STATE {Update the weight $\p_{t+1} \in \Delta_N$ according to $p_{t+1,i} \propto p_{t,i} \exp(-\epsilon \ell_{t,i})$}
      \STATE {Base-learner $\B_i$ updates the local decision by $\w_{t+1,i} = \Pi_{\W}[\w_{t,i} - \eta_i \nabla \f_t(\w_t)]$, $\forall i \in [N]$}
    \ENDFOR
\end{algorithmic}
\end{algorithm}

\paragraph{Overall Algorithm.} Combining all above ingredients, we propose the \underline{S}witching-\underline{C}ost-\underline{R}egularized \underline{E}nsemble \underline{A}lgorithm for OCO with \underline{M}emory (\textbf{Scream}) algorithm, which is based on online mirror descent and admits a two-layer meta-base online ensemble structure. Algorithm~\ref{alg:scream} presents overall procedures: each base-learner performs OGD with its step size as shown in Line 10; the meta-learner combines local decisions and updates the weight according to the switching-cost-regularized loss as described in Lines 4--9. The following theorem demonstrates that our algorithm can attain a favorable dynamic policy regret, striking a good balance between regret and switching cost.
\begin{myThm}
\label{thm:main-result-OCOmemory}
Under Assumptions~\ref{assume:Lipschitz}--\ref{assume:bounded-norm}, by setting the learning rate optimally of meta-algorithm as $\epsilon = \sqrt{2/((2\lambda + G)(\lambda + G)D^2T)}$ and the step size pool $\H$ as~\eqref{eq:step-size-pool}, our proposed Scream algorithm ensures that for any comparator sequence $\v_1, \ldots, \v_T \in \W$, we have
\begin{equation*}
  \textnormal{D-Regret}_T(\v_{1:T}) \leq \O \big(\sqrt{\lambda T(1+P_T)} + \lambda^{\frac{3}{4}} \sqrt{T} (1+ \log \log P_T) + \lambda P_T \big),
\end{equation*}
where $\lambda = m^2 L$ and $P_T = \sum_{t=2}^{T} \norm{\v_{t-1} - \v_t}_2$. So dynamic policy regret is $\O(\sqrt{T(1+P_T)})$. 
\end{myThm}

The proof of Theorem~\ref{thm:main-result-OCOmemory} is presented in Appendix~\ref{sec:appendix-proof-OCOmemory}.

\begin{myRemark}
Since the dynamic policy regret holds for any comparator sequence, by simply setting comparators as the fixed best decision in hindsight (now $P_T = 0$), our dynamic policy regret implies the $\O(\sqrt{T})$ static policy regret~\citep{NIPS'15:OCOmemory}. Second, when omitting the consideration of the $\lambda$-dependence, the dynamic regret bound simplifies to $\O(\sqrt{T(1+P_T)})$, which is \emph{minimax optimal} in terms of $T$ and $P_T$, as an $\Omega(\sqrt{T(1+P_T)})$ lower bound has been established for the dynamic regret of memoryless OCO~\citep{NIPS'18:Zhang-Ader}, which is a special case of OCO with memory when setting $m = 0$.
\end{myRemark} 

\begin{myRemark} We further examine the \emph{memory dependence} of the attained bounds. The dynamic policy regret in Theorem~\ref{thm:main-result-OCOmemory} exhibits a quadratic dependence on the memory length $m$ (i.e., linear dependence on $\lambda = m^2 L$). Recall that the dynamic policy regret is upper bounded by the dynamic regret of unary functions and switching cost of decisions (i.e., $\sum_{t=1}^T \f_t(\w_t) -  \sum_{t=1}^T \f_t(\v_t) + \lambda \sum_{t=2}^T \norm{\w_{t} - \w_{t-1}}_2$) as well as the switching cost/ path length of comparators (i.e., $\lambda \sum_{t=2}^T \norm{\v_{t} - \v_{t-1}}_2 = \lambda P_T$), namely,
\begin{align}
& \textnormal{D-Regret}_T(\v_{1:T}) \leq \underbrace{\sum_{t=1}^T \f_t(\w_t) -  \sum_{t=1}^T \f_t(\v_t) + \lambda \sum_{t=2}^T \norm{\w_{t} - \w_{t-1}}_2}_{\texttt{dynamic regret of OCO with switching cost}} + \underbrace{\lambda \sum_{t=2}^T \norm{\v_{t} - \v_{t-1}}_2}_{\texttt{path length}~(= \lambda P_T )}. \label{eq:dynamic-regret-upper-bound}
\end{align}
Notably, the last path length term is the variation of comparators and thus irrelevant to the algorithm, which already exhibits a quadratic memory dependence. As a result, in the following we will focus the memory dependence of the first two terms, which is essentially the \emph{dynamic regret of OCO with switching cost}. Indeed, our conference version~\citep{AISTATS'22:scream} gives an $\O(\sqrt{\lambda T(1+P_T)} + \lambda \sqrt{T} (1+ \log \log P_T)) \le \O(\lambda \sqrt{T(1+P_T)})$ regret bound,\footnote{Note that the $\log \log P_T$ term can be dominated by $\sqrt{P_T}$ and is thus absorbed within the $\O(\cdot)$-notation.} whereas Theorem~\ref{thm:main-result-OCOmemory} of this paper improves the result to $\O(\sqrt{\lambda T(1+P_T)} + \lambda^{3/4} \sqrt{T} (1+ \log \log P_T)) \le \O(\lambda^{3/4} \sqrt{T(1+P_T)})$ through a refined \emph{analysis} (there is no modification on the algorithm), achieving an $\lambda^{1/4}$ improvement.

As a benefit, when choosing a fixed comparator, Theorem~\ref{thm:main-result-OCOmemory} implies an $\O(\lambda^{3/4} \sqrt{T})$ static regret, improving upon the $\O(\lambda \sqrt{T})$ static regret implication based on the dynamic policy regret in the conference version~\citep{AISTATS'22:scream}, where $\lambda = \O(m^2)$ is the squared memory length. Nevertheless, the best static policy regret for OCO with switching cost is  $\O(\sqrt{\lambda T}) = \O(m\sqrt{T})$, which enjoys a linear dependence on the memory length~\citep{NIPS'15:OCOmemory} (see discussions in Appendix~\ref{sec:appendix-memory-dependency} for details), and our result still exhibits a gap here. Therefore, we are wondering what the optimal memory dependence of dynamic regret for OCO with switching cost is. We answer this question in the next subsection.
\end{myRemark}

\subsection{Improved Algorithm with an Optimal Memory Dependence}
\label{sec:optimal-memory}
In this part, we resolve the question raised at the end of the last subsection. Specifically, we first illustrate the failure of Scream algorithm in achieving optimal memory dependence; and then we propose an improved algorithm building upon Scream (Algorithm~\ref{alg:scream}) that enjoys an $\O(\sqrt{\lambda T(1+P_T)})$ dynamic regret for OCO with switching cost, hence matching the $\O(\sqrt{\lambda T})$ static regret~\citep{NIPS'15:OCOmemory} when choosing a fixed comparator such that $P_T = 0$. We finally supply the lower bound to demonstrate the minimax optimality of our attained upper bound in terms of the memory dependence.

\paragraph{Failure of Scream Algorithm.} Inspecting the proof of Theorem~\ref{thm:main-result-OCOmemory}, we can observe that the sub-optimality of memory dependence mainly comes from the meta-regret $\sum_{t=1}^T \inner{\p_t}{\ellb_t} - \sum_{t=1}^T \ell_{t,i} + \lambda D \sum_{t=2}^{T} \norm{\p_t - \p_{t-1}}_1$ (see the decomposition in~\eqref{eq:new-decompose} for more details). Specifically, consider the switching cost of meta-algorithm, which can be upper bounded as follows:
\begin{equation}
  \label{eq:meta-sc-slow}
  \lambda \sum_{t=2}^T \norm{\p_t - \p_{t-1}}_1 \leq \lambda \sum_{t=2}^T \epsilon \norm{\ell_t}_{\infty} \le \lambda \epsilon \Gmeta T \leq \O(\lambda^{\frac{3}{4}} \sqrt{T}),
\end{equation}
where the first inequality holds by the standard analysis on the meta-algorithm (see~\eqref{eq:meta-regret-analysis} for more details). The second inequality is by definition of $\Gmeta = \sup_{t \in [T], i\in [N]} \abs{\ell_{t,i}}$, that is, the maximum scale of the loss of meta-algorithm. The last inequality is due to the setting of $\epsilon = \O(1/\sqrt{T})$ and our analysis shows that $\Gmeta \leq \O(\sqrt{\lambda})$. 

From~\eqref{eq:meta-sc-slow}, we can see that the switching cost of meta-algorithm exhibits an undesirable memory dependence of order $\O(\lambda^{3/4}) = \O(m^{3/2})$, whereas our desired one is linear in $m$. Therefore, it is natural to ask for an improved meta-algorithm that can enjoy a better memory dependence. However, we present the following theorem to negatively show that when the loss of meta-algorithm lies in the range of $[-C,C]$ for some $C >0$, \emph{any} algorithm must incur a regret of $\Omega(\sqrt{\lambda C T})$. The proof is deferred to Appendix~\ref{sec:appendix-proof-PEA-SC-lower}.

\begin{myThm}
  \label{thm:PEA-SC-lower}
  Consider a $T$-round prediction with expert advice problem with $\lambda$-switching cost. Given $\lambda >0$ and $C > 0$, there exists a sequence of loss functions $\ellb_1,\ldots,\ellb_T$ satisfying $\ellb_t \in [-C,C]^N$ for all $t \in [T]$ such that any feasible expert algorithm (whose output is $\p_1,\ldots,\p_T \in \Delta_N$) incurs the following regret 
  \begin{equation*}
    \sum_{t=1}^T \langle \ellb_t, \p_t\rangle - \min_{i \in [N]}\sum_{t=1}^T \ell_{t,i}+ \lambda \sum_{t=2}^T \| \p_t - \p_{t-1} \|_1 \geq \Omega(\sqrt{\lambda C T}).
  \end{equation*}
\end{myThm}

In our case, we have $\abs{\ell_{t,i}} \leq GD+\sqrt{\lambda}$ (see the argument in~\eqref{eq:Gmeta} for details). Therefore, by applying Theorem~\ref{thm:PEA-SC-lower}, we know that the meta-algorithm will incur at least $\Omega(\lambda^{3/4} \sqrt{T})$ regret, which prohibits Scream from achieving the desired $\O(\sqrt{\lambda})$ memory dependence. 

\paragraph{An Improved Algorithm.} To address this memory dependence issue, we propose an improved algorithm called \textbf{Lazy Scream}, presented in Algorithm~\ref{alg:lazy-scream}, which is a simple variant of the vanilla Scream algorithm (see Algorithm~\ref{alg:scream}). Specifically, Lazy Scream builds upon Scream with episodic updates, and proceeds in $K$ epochs (Line 2). The $k$-th epoch is denoted by $\I_k$ such that $\abs{\I_k} = \Delta$, for all $k \in [K]$. Specifically, the algorithm updates at the epoch-level, for each epoch $\I_k$, the learner submits the same decision (Line~5) and computes the cumulative loss gradient (Line~7), and at the end of each epoch, the learner sends the cumulative gradient to the original Scream algorithm (Algorithm~\ref{alg:scream}) for update (Line~9). The next theorem shows that Lazy Scream  attains an improved dynamic policy regret in terms of memory length, whose proof can be found in Appendix~\ref{sec:appendix-proof-optimal-memory}. 

\begin{algorithm}[!t]
  \caption{\textbf{Lazy Scream}}
  \label{alg:lazy-scream}
  \begin{algorithmic}[1]
  \REQUIRE{Scream $\A$ (Algorithm~\ref{alg:scream}), epoch number $B$, epoch length $\Delta$}
  \STATE{Initialization: $\w_{1}$ from Scream $\A$}
  \FOR{$k=1$ {\bfseries to} $K$}
    \STATE {Initialize $\nabla_k=\mathbf{0}$}
    \FOR{$t=(k-1)\Delta + 1$ {\bfseries to} $k\Delta$}
      \STATE {Submit the decision $\w_t = \wcirc_k$}
      \STATE {Suffer a loss of $f_t(\w_{t-m:t})$}
      \STATE {$\nabla_k = \nabla_k + \nabla \f_t(\w_t) = \nabla_k + \nabla \f_t(\wcirc_k)$}
    \ENDFOR
    \STATE Send $\nabla_k$ to Scream $\A$ for update and receive $\wcirc_{k+1}$
  \ENDFOR
  \end{algorithmic}
\end{algorithm}

\begin{myThm}
  \label{thm:optimal-memory}
  Under the same assumptions as Theorem~\ref{thm:main-result-OCOmemory}, by setting the learning rate of meta-algorithm optimally and the step size pool $\H$ as~\eqref{eq:step-size-pool}, our proposed Lazy Scream (Algorithm~\ref{alg:lazy-scream}) with epoch length $\Delta = \sqrt{\lambda}$ ensures that
  \begin{equation*}
    \textnormal{D-Regret}_T(\v_{1:T}) \le \O \big(\sqrt{\lambda T(1+P_T)} + \lambda P_T \big),
  \end{equation*}
  for any comparator sequence $\v_1, \ldots, \v_T \in \W$, where $\lambda = m^2 L$ and $P_T = \sum_{t=2}^{T} \norm{\v_{t-1} - \v_t}_2$.
\end{myThm}
Theorem~\ref{thm:optimal-memory} implies an $\O(\sqrt{\lambda T(1+P_T)})$ dynamic regret for OCO with switching cost. Below we further prove that our result is \emph{minimax optimal} in switching-cost coefficient $\lambda$, time horizon $T$, and path length $P_T$. 
\begin{myThm}
  \label{thm:lower-dynamic}
  Given a real value $\tau \in [0, DT]$ and a parameter $\lambda > 0$, there exist (1) a sequence of convex loss functions $h_1,\ldots,h_T$ with $h_t: \W \mapsto \R$ for $t \in [T]$, which satisfy Assumption~\ref{assume:bounded-gradient} and some feasible domain $\W \subseteq \R^d$ with Assumption~\ref{assume:bounded-norm}; and (2) a sequence of comparators $\v_1,\ldots,\v_T \in \R^d$ whose path length $P_T(\v_1,\ldots,\v_T) = \sum_{t=2}^T \norm{\v_{t} - \v_{t-1}}_2 \leq \tau$, such that any online algorithm returning $\w_1,\ldots,\w_T \in \W$ satisfies 
  \begin{equation}
    \label{eq:lower-bound}
    \sum_{t=1}^T h_t(\w_t) -  \sum_{t=1}^T h_t(\v_t) + \lambda \sum_{t=2}^T \norm{\w_{t} - \w_{t-1}}_2 \geq \Omega(\sqrt{\lambda \tau T}).
  \end{equation}
\end{myThm}

Theorem~\ref{thm:lower-dynamic} demonstrates the minimax optimality of the obtained $\O(\sqrt{\lambda T(1+P_T)})$ dynamic regret bound for OCO with switching cost, which is optimal in terms of switching-cost coefficient $\lambda$, time horizon $T$, and path length $P_T$. The corresponding proof can be found in Appendix~\ref{sec:appendix-proof-lower-bound}. \section{Online Non-stochastic Control}
\label{sec:online-control}
In this section, we apply the results of OCO with memory to an important online decision-making problem, online non-stochastic control~\citep{ICML'19:online-control}, which draws much attention from researchers in online learning and control theory communities~\citep{ICML'19:online-control,COLT'20:control-improper,ALT'20:control-Hazan,NIPS'20:max-control,NIPS'20:bandit-control,NIPS'20:bandit-Koren,arXiv'20:adaptive-OCOmemory,AISTATS'22:SA-OCOwMemory}.

\subsection{Problem Statement}
\paragraph{Problem Setting.} We study the online control of the linear dynamical system (LDS) governed by
\begin{equation}
  \label{eq:dynamics}
  x_{t+1} = A x_{t} + B u_t + w_t,
\end{equation}
where at iteration $t$, the controller provides the control $u_t$ upon the observed dynamical state $x_t$ and suffers a cost $c_t(x_t,u_t)$ with convex function $c_t: \R^{d_x} \times \R^{d_u} \mapsto \R$. Following the notational convention of previous works, throughout the section we will use unbold fonts to denote vectors (including control signal, state, disturbance, etc.). We focus on online \emph{non-stochastic} control~\citep{ICML'19:online-control}, that is, the disturbance can be generated arbitrarily and no statistical assumption is imposed on its distribution; additionally, cost functions can be chosen adversarially. The adversarial nature of the disturbance and online cost functions hinders an a priori computation of the optimal policy as in settings of classical control theory~\citep{kalman1960contributions} and therefore requires techniques from modern online learning to tackle adversarial environments. 

\paragraph{Policy Regret.} The standard measure for online non-stochastic control is the \emph{policy regret}~\citep{ICML'19:online-control}, defined as the difference between cumulative loss of the designed controller $\mathcal{A}$ and that of the compared controller $\pi \in \Pi$, namely,
\begin{equation}
  \label{eq:control-Reg}
  \mbox{Regret}_T = \sum_{t=1}^{T} c_t(x_t,u_t) - \min_{\pi \in \Pi} \sum_{t=1}^{T} c_t(x_t^{\pi},u_t^{\pi}).
\end{equation}
The comparator could be chosen with complete foreknowledge of the disturbance and loss functions. Recently, a variety of control algorithms have been proposed to optimize this measure under different settings~\citep{ICML'19:online-control,ALT'20:control-Hazan,COLT'20:control-improper,NIPS'20:bandit-Koren,NIPS'20:bandit-control,ICML'20:log-control}. However, we argue that competing with a fixed controller may be not appropriate, especially because the unknown disturbances and cost functions can change arbitrarily in the non-stochastic control setting so that the optimal controller of each round would also change accordingly. Therefore, it is necessary to enable the online controller to compete with \emph{time-varying} controllers to adapt to those changes. To this end, we generalize the standard measure~\eqref{eq:control-Reg} to the \emph{dynamic policy regret} to benchmark the algorithm with a sequence of \emph{time-varying} controllers $\pi_1,\ldots,\pi_T \in \Pi$, formally,
\begin{equation}
\label{eq:control-dynamic-Reg}
  \textnormal{D-Regret}_T(\pi_{1:T})= \sum_{t=1}^{T} c_t(x_t,u_t) - \sum_{t=1}^{T} c_t(x_t^{\pi_t},u_t^{\pi_t}).
\end{equation}
The measure clearly subsumes the standard policy regret~\eqref{eq:control-Reg} when choosing the compared controllers as a fixed one, i.e., $\pi_* \in \argmin_{\pi \in \Pi} \sum_{t=1}^{T} c_t(x_t^{\pi},u_t^{\pi})$. In this work, the benchmark set $\Pi$ is chosen as the class of disturbance-action controllers (see Definition~\ref{def:DAC}), which encompasses many controllers of interest.  

\subsection{Reduction to OCO with Memory}
\label{sec:control-reduction}
Following the pioneering work~\citep{ICML'19:online-control}, we will work on the \emph{Disturbance-Action Controller} (DAC) policy class, which parametrizes the executed action as a linear function of the past disturbances. By doing so, we can reduce online non-stochastic control to OCO with memory so that the results of Section~\ref{sec:OCOwMemory} can be leveraged to design robust controllers with provable dynamic policy regret guarantee. 
\begin{myDef}[Disturbance-Action Controller, DAC]
\label{def:DAC}
A disturbance-action controller, denoted by $\pi(K,M)$, with memory length $H$ is specified by a fixed matrix $K$ and parameters $M = (M^{[1]},\ldots,M^{[H]})$. At each iteration $t$, the controller $\pi(K,M)$ chooses the action as a linear map of the past disturbances with an offset linear controller, formally, $u_t = -Kx_t + \sum_{i=1}^{H} M^{[i]} w_{t-i}$. 
\end{myDef}
For convenience, we define $w_i = 0$ for $i < 0$. The DAC policy is implementable because the disturbance can be recovered by $w_{t} = x_{t+1} - A x_{t} - B u_{t}$ as system dynamics $A$ and $B$ are supposed to be known. Our method can also extend to the scenario of online non-stochastic control with unknown systems, which is presented at the end of this section. The following proposition by~\citet{ICML'19:online-control} presents an important property of DAC policy.
\begin{myProp}[Lemma 4.3 of \citet{ICML'19:online-control}]
\label{proposition:DAC-state}
Suppose the initial state is $x_0 = 0$ and one chooses the DAC controller $\pi(K,M_t)$ at iteration $t$, the reaching state and the corresponding DAC control are 
\begin{align*}
x_t^K(M_{0:t-1}) = {}& \sum_{i=0}^{H + t-1} \Psi_{t-1,i}^{K,t-1}(M_{0:t-1}) w_{t-1-i},\\
u_t^K(M_{0:t}) = {}& -Kx_t^K(M_{0:t-1}) + \sum_{i=1}^{H} M_t^{[i]} w_{t-i},
\end{align*}
where $\tilde{A}_K = A - BK$ and $$\Psi_{t,i}^{K,h}(M_{t-h:t}) = \tilde{A}_K^i \ind{i \leq h} + \sum_{j=0}^{h} \tilde{A}_K^j B M_{t-j}^{[i-j]} \ind{1 \leq i-j \leq H}.$$
\end{myProp}
Evidently, both state $x_t$ and control $u_t$ are linear functions of DAC parameters $M_{0:t}$, so the cost $c_t(x_t^K(M_{0:t-1}),u_t^K(M_{0:t}))$ is a function of historical parameters $M_{0:t}$. Thereby, the remaining challenge is to handle this \emph{memory} issue due to the state transition of online control, which can be addressed by OCO with memory studied in Section~\ref{sec:OCOwMemory}. Note that there is one big caveat in applying the technique---the current memory length is not fixed but growing with time, which is not feasible in OCO with memory. To this end,~\citet{ICML'19:online-control} further propose a truncation operation that truncates the state with a fixed memory length $H$ and defines the following truncated loss.
\begin{myDef}[Truncated Loss]
\label{def:truncated-loss}
For the cost function $c_t: \R^{d_x} \times \R^{d_u} \mapsto \R$ and DAC policies $\{\pi(K,M_t)\}_{t=1,\ldots,T}$, given memory length $H$, the induced truncated loss $f_t: \M^{H+2} \mapsto \R$ is defined as 
\begin{equation*}
  f_t(M_{t-1-H:t}) = c_t(y_t^K(M_{t-1-H:t-1}), v_t^K(M_{t-1-H:t})),
\end{equation*}
where the truncated state and truncated DAC control are $$y_{t+1}^K = \sum_{i=0}^{2H} \Psi_{t,i}^{K,H}(M_{t-H:t}) w_{t-i}, ~\text{ and } v_{t+1}^K = - K y_{t+1}^K(M_{t-H:t}) + \sum_{i=1}^{H} M_{t+1}^{[i]} w_{t+1-i}.$$ 
\end{myDef}
It can be proved that the error introduced by the truncation operation (the gap between $f_t$ and $c_t$) can be precisely controlled. Therefore, by feeding the truncated loss $f_t$ to the OCO with memory framework with a memory length of $H+2$, we finish the reduction from online non-stochastic control to OCO with memory.

\subsection{Dynamic Policy Regret of Online Non-stochastic Control}
\label{sec:control-algorithm-theory}
The above reduction enables us to leverage results of OCO with memory (Section~\ref{sec:OCOwMemory}) to design online controllers competitive with time-varying compared policies. We propose the \textbf{Scream.Control} algorithm, consisting of the following two components:
\begin{itemize}
  \item[(1)] DAC parameterization for reduction: using DAC control $u_t = \pi(K, M_t)$ for parameterization and define the unary loss of the truncated loss, i.e., $\tilde{f}_t: \mathcal{M} \mapsto \R$ with $\tilde{f}_t(M) = f_t(M,\ldots,M)$ (see Definition~\ref{def:truncated-loss}).
  \item[(2)] meta-base online ensemble structure for OCO with memory:  performing Scream algorithm of Section~\ref{sec:OCOwMemory} over unary loss $\tilde{f}_t$, and using meta-algorithm to combine intermediate parameters $M_{t,1},\ldots,M_{t,N}$ from all base-learners to produce the final $M_t$.
\end{itemize}

\begin{algorithm}[!t]
   \caption{\textbf{Scream.Control}}
   \label{alg:control}
\begin{algorithmic}[1]
  \REQUIRE{step size pool $\H = \{\eta_1,\ldots,\eta_N\}$; learning rate of meta-learner $\epsilon$; memory length $H$; linear controller $K$; feasible domain $\M$}
  \STATE{Initialization: $u_1,\ldots,u_H$, any feasible output control signals for the first $H$ rounds;}
  \STATE{Initialization: base decisions of the $H$-th round $M_{H,1}, M_{H,2}, \ldots M_{H,N} \in \mathcal{M}$; non-uniform weight $\p_{H+1} \in \Delta_N$ with $p_{H+1,i} \propto 1/(i^2 + i)$, $\forall i \in [N]$}
    \FOR{$t=H+1$ {\bfseries to} $T$}
      \STATE {Receive $M_{t,i}$ from base-learner $\B_i$ for $i \in [N]$}
      \STATE {Obtain the policy parameter $M_t = \sum_{i=1}^{N} p_{t,i}M_{t,i}$}
      \STATE {Output $u_t = - Kx_t + \sum_{i=1}^{H} M_t^{[i]} w_{t-i}$}
      \STATE {Suffer a loss of $c_t(x_t,u_t)$ and observe the cost function $c_t: \R^{d_x} \times \R^{d_u} \mapsto \R$}
      \STATE {Construct the truncated loss $f_t: \M^{H+2} \mapsto \R$ by Definition~\ref{def:truncated-loss} and the linearized loss by $g_t(M) = \inner{\nabla \f_t(M_t)}{M}$}
      \STATE {Compute the switching-cost-regularized loss $\ellb_t \in \R^N$ with $\ell_{t,i} = \lambda \Fnorm{M_{t,i} - M_{t-1,i}} + g_t(M_{t,i})$ for $i \in [N]$}
      \STATE {Update the weight to $\p_{t+1} \in \Delta_N$ via $p_{t+1,i} \propto p_{t,i} \exp(-\epsilon \ell_{t,i})$}
      \STATE {Base-learner $\B_i$ updates the local parameter by $M_{t+1,i} = \Pi_{\mathcal{M}}[M_{t,i} - \eta_i \nabla \tilde{f}_t(M_t)]$} 
      \STATE Observe the new state $x_{t+1}$ and calculate the disturbance $w_t = x_{t+1} - Ax_t - Bu_t$
    \ENDFOR
\end{algorithmic}
\end{algorithm}

Algorithm~\ref{alg:control} describes our proposed algorithm for optimizing dynamic policy regret of online non-stochastic control. We further provide its theoretical guarantee. We begin with several standard assumptions used in the literature~\citep{ICML'19:online-control,ALT'20:control-Hazan,NIPS'20:bandit-control} and next present the main result.
\begin{myAssumption}
\label{assume:bound-noise}
The system matrices are bounded, i.e., $\opnorm{A} \leq \kappa_A$ and $\opnorm{B} \leq \kappa_B$. Besides, the disturbance $\norm{w_t} \leq W$ holds for any $t \in [T]$.
\end{myAssumption}

\begin{myAssumption}
\label{assume:cost-bound}
The cost function $c_t(x,u)$ is convex. Further, when $\norm{x},\norm{u} \le D$, it holds that $|c_t(x,u)| \leq \beta D^2$ and $\norm{\nabla_x c_t(x,u)}, \norm{\nabla_{u} c_t(x,u)} \leq G_c D$.
\end{myAssumption}

\begin{myAssumption}
\label{assume:strongly-stable}
DAC controller $\pi(K,M)$ satisfies:
\begin{itemize}
\item[(1)] $K$ is $(\kappa, \gamma)$-strongly stable, whose precise definition is in Definition~\ref{def:controller-set} of Appendix~\ref{sec:appendix-notions}; 
\item[(2)] $M \in \M$ where $\M = \{M = (M^{[1]},\ldots,M^{[H]}) \mid \opnorm{M^{[i]}} \leq \kappa_B \kappa^3 (1-\gamma)^i \}$.
\end{itemize}
\end{myAssumption}

\begin{myThm}
\label{thm:main-result-control}
Under Assumptions~\ref{assume:bound-noise}--\ref{assume:strongly-stable}, we set learning rate optimally and the step size pool $\H$ as
\begin{equation}
  \label{eq:step-size-pool-control}
  \H = \left\{ \eta_i \givenn \eta_i = 2^{i-1}\cdot \sqrt{\frac{D_f^2}{(\lambda G_f + G_f^2)T}}, i \in [N] \right\},
\end{equation}
where $N = \left\lceil \frac{1}{2} \log_2(1+T) \right\rceil + 1 = \O(\log T)$ is the number of base-learners, and $\lambda = (H+2)^2L_f$. The parameters $L_f, G_f, D_f$ are defined in Lemma~\ref{lemma:support-2-f-property} and only depend on natural parameters of the linear dynamical system and truncated memory length $H$. By choosing $H = \Theta(\log T)$, our \SC~algorithm enjoys
\begin{align*} 
\sum_{t=1}^{T} c_t(x_t,u_t) - \sum_{t=1}^{T} c_t(x_t^{\pi_t},u_t^{\pi_t}) \leq \Ot \big(\sqrt{T(1+P_T)}\big),
\end{align*}
where $\pi_1,\ldots,\pi_T \in \Pi$ can be any comparator sequence in the compared DAC policy class $\Pi = \{\pi(K,M) \mid M \in \M \}$ with $\pi_t = \pi(K,M_t^*)$ for $t \in [T]$. The path length $P_T = \sum_{t=2}^{T} \Fnorm{M_{t-1}^* - M_t^*}$ measures the cumulative variation of comparators.
\end{myThm}

\begin{algorithm}[!t]
\caption{System Identification via Random Inputs~\citep{ALT'20:control-Hazan}}
\label{alg:system-identification}
\begin{algorithmic}[1]
\REQUIRE{rounds of exploration $T_0$.}
\FOR{$t=1,\ldots,T_0$}
    \STATE {Execute the control $u_t = -K x_t + \tilde{u}_t$ with $\tilde{u}_t \sim_{i.i.d.} \{\pm 1\}^{d_u}$}
    \STATE {Record the observed state $x_{t+1}$}
\ENDFOR
\STATE {Declare $N_j = \frac{1}{T_0-k}\sum_{t=0}^{T_0-k-1} x_{t+j+1} \tilde{u}_t^\T$, for all $j \in [k]$}
\STATE {Define $\hat{C}_0 = \mbr{N_0,\ldots,N_{k-1}}, \hat{C}_1 = \mbr{N_1,\ldots,N_k}$ and return estimation $\Ah,\Bh$ as}
\begin{equation*}
    \Bh = N_0,\quad \Ah_K \define \hat{C}_1 \hat{C}_0^\T \sbr{\hat{C}_0 \hat{C}_0^\T}^{-1},\quad \Ah = \Ah_K + \Bh K.
\end{equation*}
\end{algorithmic}
\end{algorithm}

Till now, we assume the knowledge of the underlying system $A$ and $B$. By further adopting the system identification via random inputs developed by~\citet{ALT'20:control-Hazan}, our result can be extended to online non-stochastic control with unknown systems. Indeed, when the system is unknown, i.e., $A$ and $B$ are not known in advance, we follow the explore-then-commit method of~\citet{ALT'20:control-Hazan} to identify the underlying dynamics and then deploy the control algorithm based on the estimated system dynamics. The algorithmic descriptions are summarized in Algorithm~\ref{alg:system-identification}. In the exploration phase, the identification algorithm~\citep[Algorithm 2]{ALT'20:control-Hazan} uses some random inputs to approximately recover the system dynamics. Specifically, given an estimation budget $T_0 < T$, in the first $T_0$ rounds, we input the control signal $u_t = -K x_t + \tilde{u}_t$ with the random inputs $\tilde{u}_t \sim \{\pm 1\}^{d_u}$ and then observe the corresponding state $x_{t+1}$. Then, by the estimation method presented in Line 6 of Algorithm~\ref{alg:system-identification}, we can show that the estimation regret overhead is $\Ot(T^{2/3})$ when choosing $T_0 = \Theta(T^{2/3})$. 

To give the formal regret analysis and ensure finite-sample convergence rate, we focus on the system with strong controllability following the work of~\citet{ALT'20:control-Hazan}.

\begin{myDef}[Strong Controllability]
  \label{def:unknown-strong-controllability}
  For a linear dynamical system and a strongly stable linear controller $K$, for $k \ge 1$, define a matrix $C_k \in \R^{d_x \times k d_u}$ as 
  \begin{equation}
    \label{eq:C_k}
      C_k = \mbr{B, \tilde{A}_KB, \ldots, \tilde{A}_K^{k-1}B},
  \end{equation}
  where $\tilde{A}_K=A-BK$. A linear dynamical system is controllable with controllability index $k$ if $C_k$ has full row-rank. In addition, such a system is also $(k,\kappa_c)$-strongly controllable if $\norm{\sbr{C_k C_k^\T}^{-1}} \le \kappa_c$.
\end{myDef}

\begin{myAssumption}[Strong Controllability]
  \label{assume:strong-controllability}
  The dynamical system $x_{t+1} = A x_{t} + B u_t + w_t$ is $(k,\kappa_c)$-strongly controllable.
\end{myAssumption}

\begin{myThm}
\label{thm:unknown-system}
Under the same assumptions of Theorem~\ref{thm:main-result-control} except that system matrices $A$ and $B$ are now unknown, and suppose the systems are strongly controllable (see Assumption~\ref{assume:strong-controllability}) and the time horizon $T$ is sufficiently large, \SC~with system identification (Algorithm~\ref{alg:system-identification}) ensures that with high probability,
\[
  \sum_{t=1}^{T} c_t(x_t,u_t) - \sum_{t=1}^{T} c_t(x_t^{\pi_t},u_t^{\pi_t}) \leq \Ot(\sqrt{T(1+P_T)} + T^{2/3}),
\] 
where $\pi_1,\ldots,\pi_T \in \Pi$ can be any comparator sequence in the compared DAC policy class $\Pi = \{\pi(K,M) \mid M \in \M \}$ with $\pi_t = \pi(K,M_t^*)$ for $t \in [T]$. The path length $P_T = \sum_{t=2}^{T} \Fnorm{M_{t-1}^* - M_t^*}$ measures the cumulative variation of comparators.
\end{myThm}

Finally, we note that our obtained dynamic policy regret bound in Theorem~\ref{thm:main-result-control} can recover the $\Ot(\sqrt{T})$ static policy regret for non-stochastic control with known systems~\citep{ICML'19:online-control}, and the result in Theorem~\ref{thm:unknown-system} implies an $\Ot(T^{2/3})$ high-probability static policy regret for non-stochastic control with unknown systems~\citep{ALT'20:control-Hazan}.
\begin{myCor}
\label{corollary:static-regret-control}
For known systems, under the same assumptions of Theorem~\ref{thm:main-result-control}, it holds that \SC~enjoys a static policy regret at most $$\sum_{t=1}^{T} c_t(x_t,u_t) - \min_{\pi \in \Pi} \sum_{t=1}^{T} c_t(x_t^{\pi},u_t^{\pi}) \leq \Ot(\sqrt{T}).$$ 

For unknown systems, under the same assumptions of Theorem~\ref{thm:unknown-system}, \SC~with system identification ensures that with high probability,
$$\sum_{t=1}^{T} c_t(x_t,u_t) - \min_{\pi \in \Pi} \sum_{t=1}^{T} c_t(x_t^{\pi},u_t^{\pi}) \leq \Ot(T^{2/3}).$$
In above, the comparator set $\Pi$ can be chosen as either the set of DAC policies or the set of strongly linear controllers.
\end{myCor} \section{Experiment}
\label{sec:appendix-experiment}
Although our paper mainly focuses on the theoretical investigation, in this section, we further present empirical studies to support our theoretical findings. We report the results of OCO with memory in Section~\ref{sec:experiment-OCOM} and online non-stochastic control in Section~\ref{sec:experiment-control}.

\subsection{OCO with Memory}
\label{sec:experiment-OCOM}
Since OCO with memory is essentially tackled by optimizing the upper bound of the policy regret, which consists of the vanilla regret over the unary functions and the switching cost, as explained in~\eqref{eq:dynamic-regret-upper-bound} for dynamic policy regret. Thus, in the empirical studies, we directly investigate the performance of different algorithms in optimizing this upper bound, i.e., the unary regret with switching cost. More specifically, we consider the following OCO with switching cost problem: at each round, the player predicts $\w_t \in \W$ and the environments choose the loss function $f_t: \W \mapsto \R$. The player will then suffer a loss of $f_t(\w_t)$ as well as a switching cost of $\norm{\w_t - \w_{t-1}}_2$, and thus the overall loss is $f_t(\w_t) + \lambda \norm{\w_t - \w_{t-1}}_2$ with some $\lambda > 0$ as the trade-off parameter.

\paragraph{Settings.} We simulate the online learning scenario by the following setting: the player sequentially receives the feature of data item and then predicts its label. The data item of each round is denoted by $(\x_t, y_t) \in \X \times \Y$, where $\X$ is a $d$-dimensional  ball with diameter $\Gamma$ and $\Y \in \R$ is the space of real values. The time horizon is set as $T=50000$ and the dimension is set as $d=10$. To simulate the distribution changes, we generate the output according to $\y_t = \x_t^\T \w_t^* + \epsilon_t$, where $\w_t^* \in \R^d$ is the underlying model and $\epsilon_t \in [0,0.1]$ is a random noise. The underlying model $\w_t^*$ will change every $1000$ rounds, randomly sampled from a $d$-dimensional ball with diameter $D/2$, so there are in total $S = 50$ changes. We the squared loss as loss functions, defined as $f_t(\w) = \frac{1}{2} (\w^\T\x_t - y_t)^2$ and thus the gradient is $\nabla f_t(\w) = (\w^\T\x_t - y_t) \cdot \x_t$. The feasible set $\W$ is also set as $d$-dimensional ball with diameter $D/2$, and thus from all above settings, we know that $\norm{\x_t}_2 \leq \Gamma$, $\norm{\w}_2 \leq D/2$, and $\norm{\nabla f_t(\w)}_2 \leq D\Gamma^2$. We set $\Gamma=1$ and $D=2$, so the gradient norm is upper bounded by $G = D\Gamma^2 = 2$. 

\paragraph{Contenders and Measure.} We benchmark our proposed Scream algorithm with the following two algorithms: (1) OGD~\citep{ICML'03:zinkvich}, is the online gradient descent algorithm. The work of~\citet{NIPS'15:OCOmemory} proves that this simple \emph{static} regret minimization algorithm also enjoys a low switching cost when choosing the step size as $\eta = \O(1/\sqrt{T})$. (2) Ader~\citep{NIPS'18:Zhang-Ader}, is the online algorithm designed in non-stationary online convex optimization. Ader is also in a meta-base structure to optimize the dynamic regret, but the algorithm does \emph{not} consider the switching cost. Thus its switching cost might be huge (as analyzed in Section~\ref{sec:OCOmemory-challenge}). 

We examine the performance of all compared algorithms via the following three measures: (1) the overall cost $\sum_{t=1}^{T} f_t(\w_t) + \lambda \sum_{t=2}^{T} \norm{\w_t - \w_{t-1}}_2$, (2) the cumulative loss $\sum_{t=1}^{T} f_t(\w_t)$, and (3) the switching cost $\lambda \sum_{t=2}^{T} \norm{\w_t - \w_{t-1}}_2$. Here, we set the regularizer coefficient $\lambda = \alpha G$, where $G$ is the gradient norm upper bound, with the purpose of matching the magnitude of cumulative loss and the switching cost. We consider three cases with different regularizer coefficients that impose different levels of penalty on the switching cost: 
\begin{enumerate}[noitemsep,nolistsep]
    \item[(i)] small regularizer ($\alpha = 0.1$): in this case the switching cost is small so that optimizing the dynamic regret would dominate the performance;
    \item[(ii)] medium regularizer ($\alpha = 1$): in this case the algorithm needs to have a good balance of dynamic regret and switching cost in order to behave well;
    \item[(iii)] large regularizer ($\alpha = 2$): in this case dynamic regret is small so that optimizing the switching cost would dominate the performance.
\end{enumerate}
We repeat the experiments five times and report the mean and standard deviation of different algorithms with respect to three performance measures (overall loss, cumulative loss, and switching cost).

\begin{figure}[!t]
\centering
    \subfigure[{overall loss ($\alpha = 0.1$)}]{ 
    \label{fig:small-overall}
    \includegraphics[clip, trim=1.1cm 0.2cm 1.3cm 1.3cm,height=0.24\textwidth]{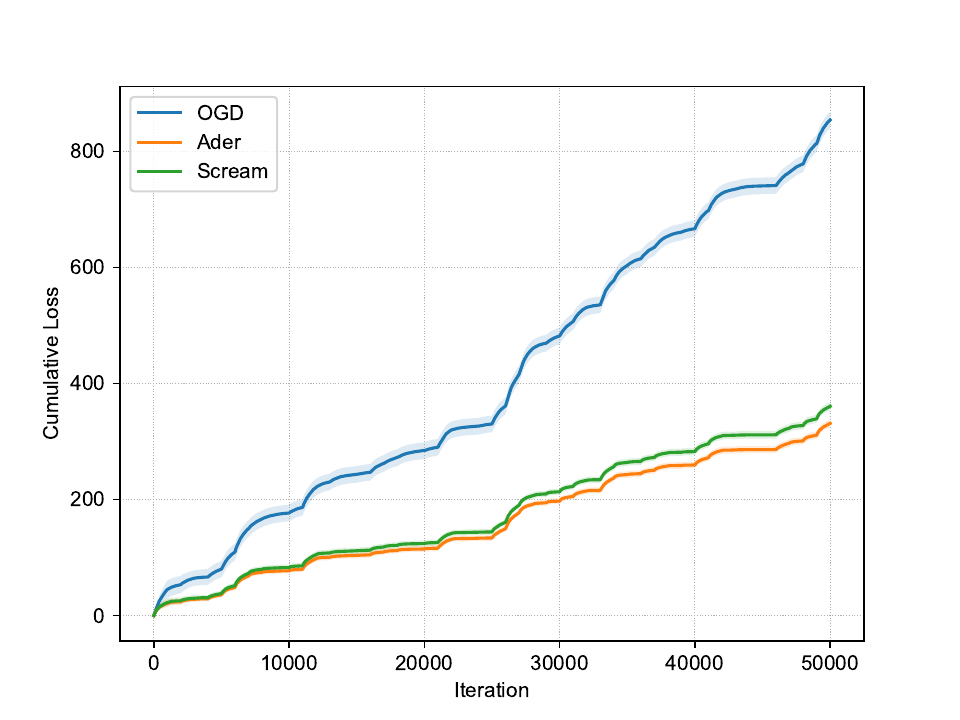}}
    \subfigure[{cumulative loss ($\alpha = 0.1$)}]{ 
    \label{fig:small-loss}
    \includegraphics[clip, trim=1.1cm 0.2cm 1.3cm 1.3cm,height=0.24\textwidth]{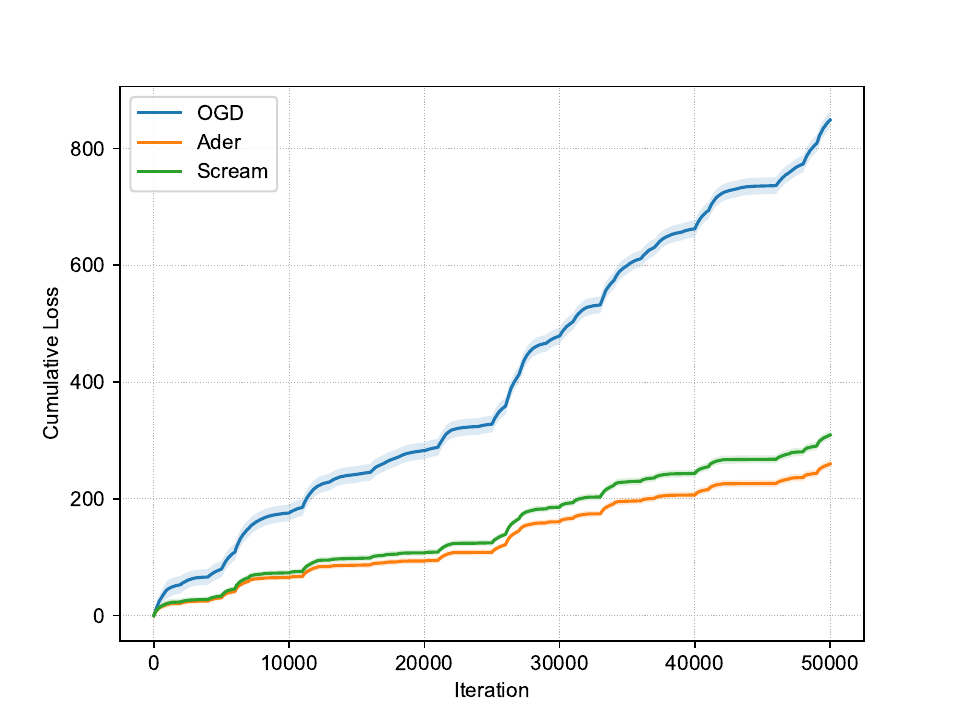}}
    \subfigure[{switching cost ($\alpha = 0.1$)}]{ 
    \label{fig:small-sc}
    \includegraphics[clip, trim=1.2cm 0.2cm 1.3cm 1.3cm,height=0.24\textwidth]{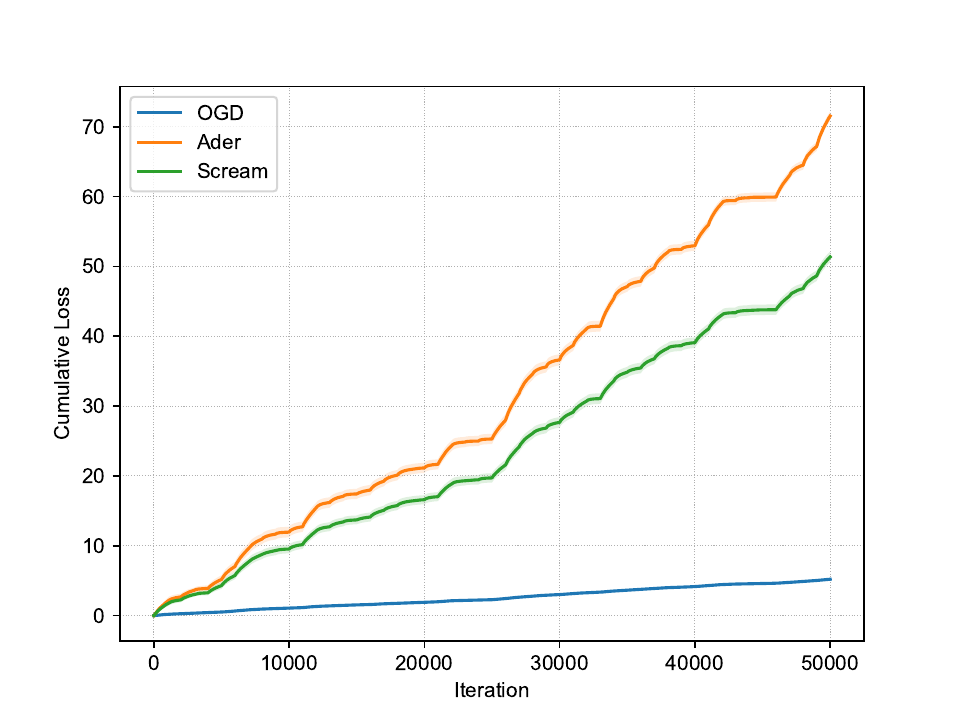}}
    \subfigure[{overall loss ($\alpha = 1$)}]{ 
    \label{fig:medium-overall}
    \includegraphics[clip, trim=1.1cm 0.2cm 1.3cm 1.3cm,height=0.24\textwidth]{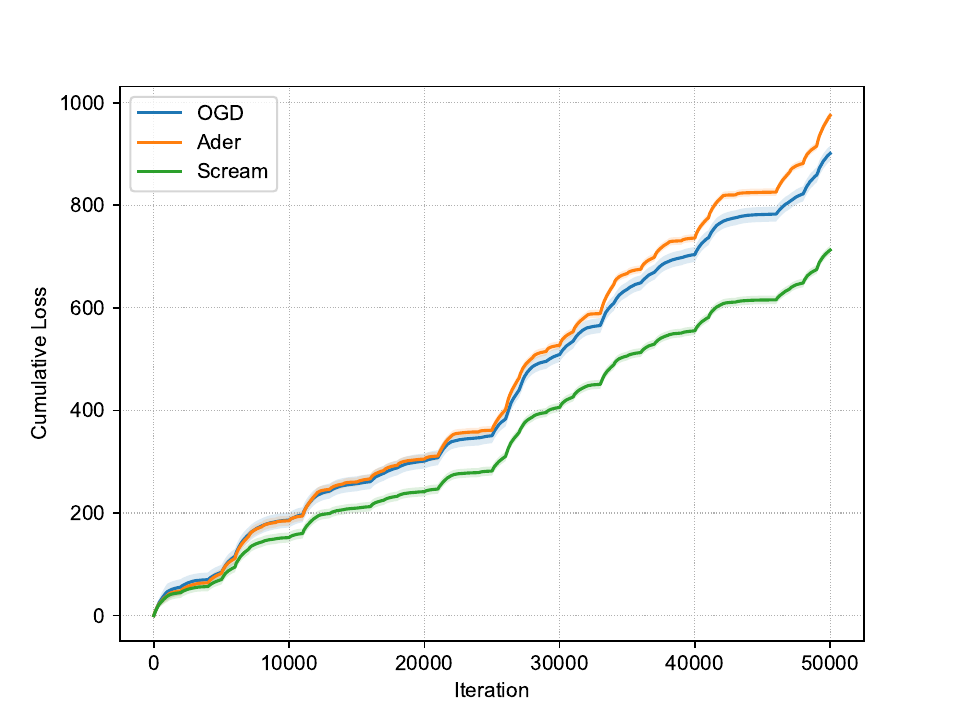}}
    \subfigure[{cumulative loss ($\alpha = 1$)}]{ 
    \label{fig:medium-loss}
    \includegraphics[clip, trim=1.1cm 0.2cm 1.3cm 1.3cm,height=0.24\textwidth]{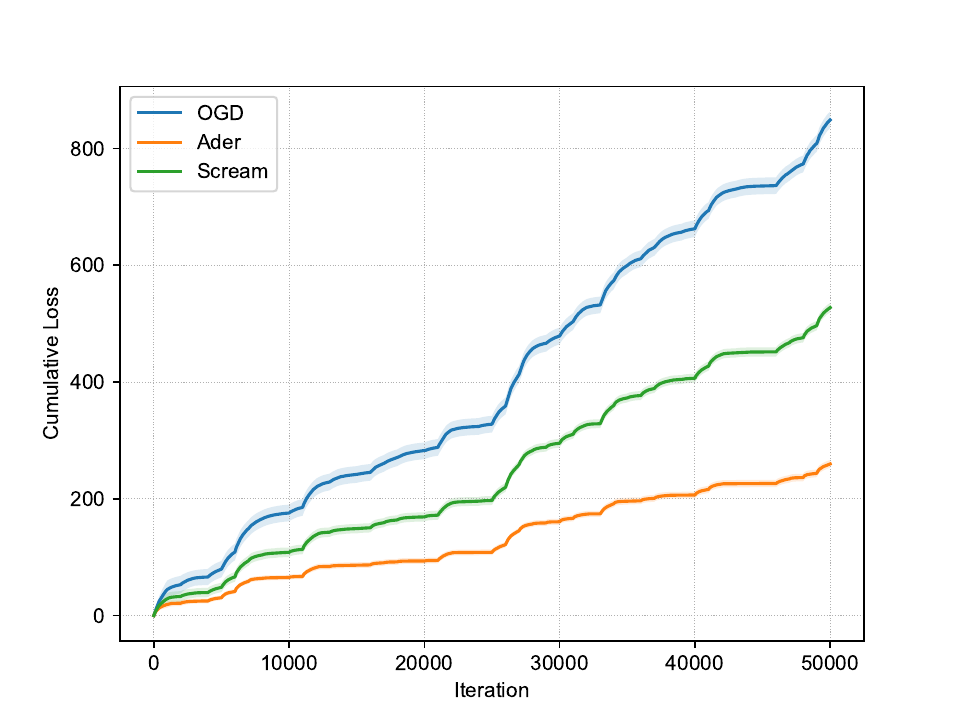}}
    \subfigure[{switching cost ($\alpha = 1$)}]{ 
    \label{fig:medium-sc}
    \includegraphics[clip, trim=1.1cm 0.2cm 1.3cm 1.3cm,height=0.24\textwidth]{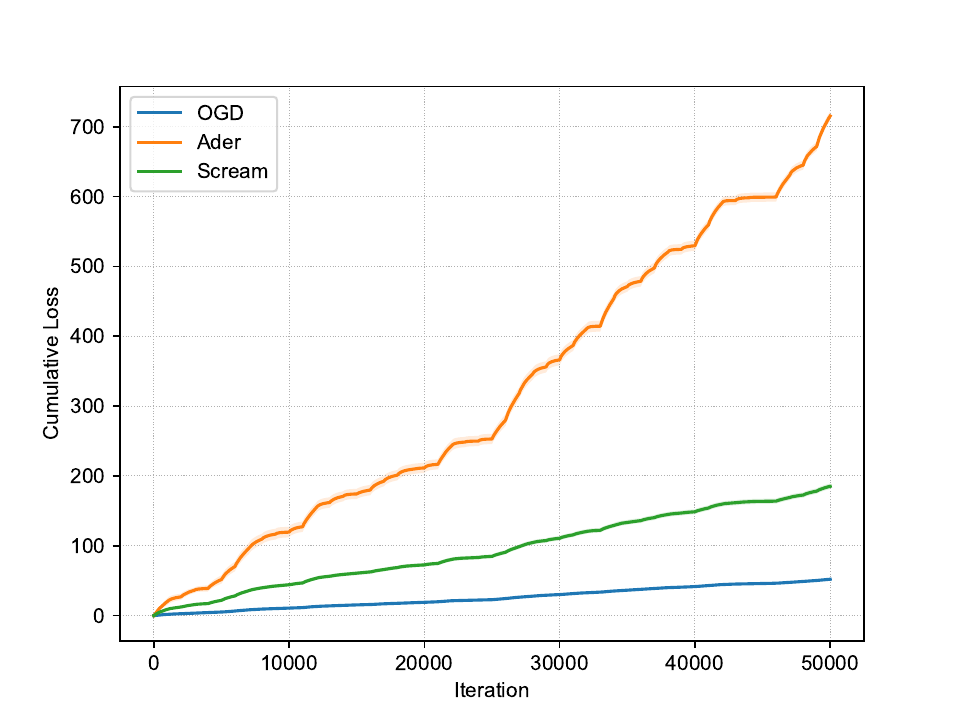}}
    \subfigure[{overall loss ($\alpha = 2$)}]{ 
    \label{fig:large-overall}
    \includegraphics[clip, trim=1cm 0.2cm 1.3cm 1.3cm,height=0.24\textwidth]{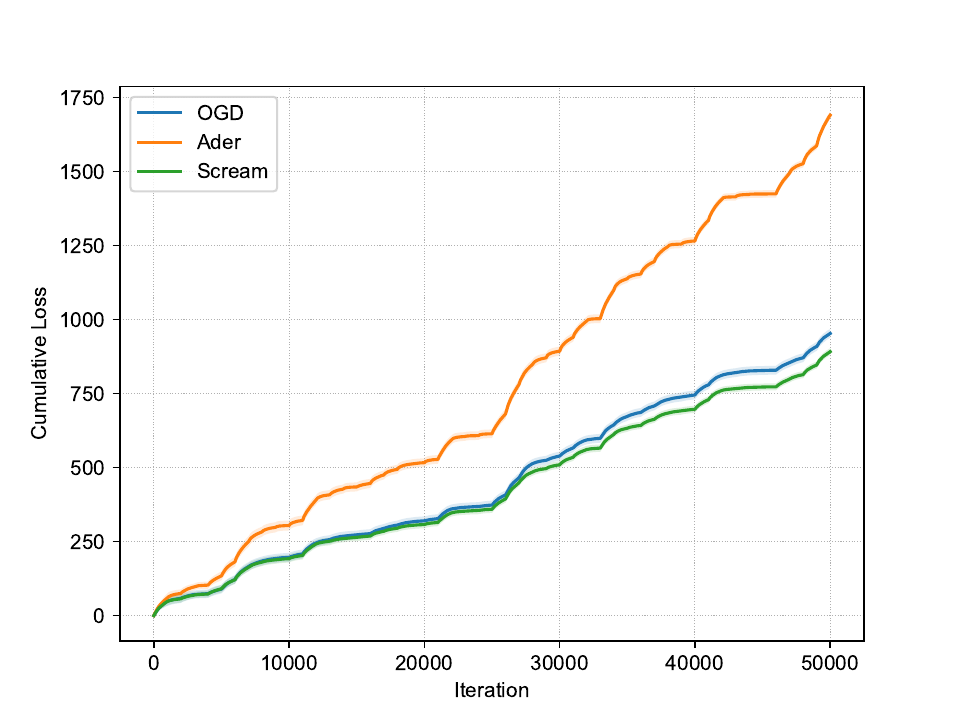}}
    \subfigure[{cumulative loss ($\alpha = 2$)}]{ 
    \label{fig:large-loss}
    \includegraphics[clip, trim=1.1cm 0.2cm 1.3cm 1.3cm,height=0.24\textwidth]{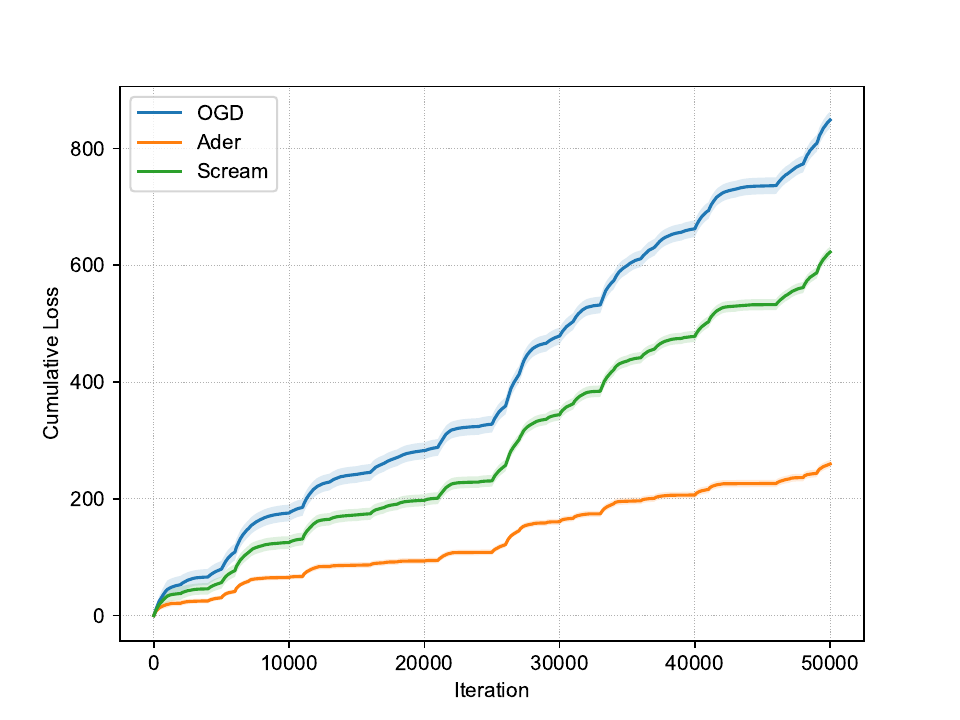}}
    \subfigure[{switching cost ($\alpha = 2$)}]{ 
    \label{fig:large-sc}
    \includegraphics[clip, trim=1cm 0.2cm 1.3cm 1.3cm,height=0.24\textwidth]{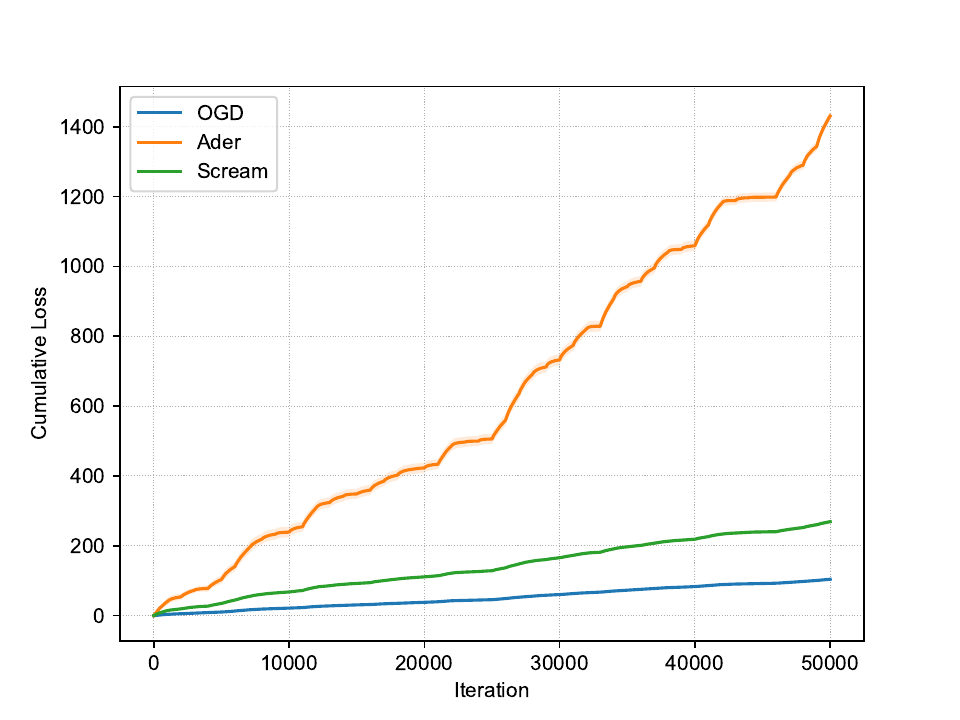}}
  \caption{Performance comparisons of OGD, Ader, Scream, under different regularizer coefficients ($\lambda = \alpha G$, $G$ is the gradient norm upper bound). The performance is evaluated by three measures: overall loss, cumulative loss, and switching cost.}
  \label{fig:comparison}
\end{figure}

\paragraph{Results.} 
Figure~\ref{fig:comparison} plots performance comparisons of three algorithms (OGD, Ader, Scream) under different regularizer coefficients. There are in total nine sub-figures, where each row presents the performance under a particular regularizer coefficient ($\alpha = 0.1, 1, 2$), and each column reports the performance in terms of a specific measure (overall loss, cumulative loss, and switching cost). For instance, Figure~\ref{fig:medium-overall} plots the overall loss under the setting of $\lambda = \alpha G$ with $\alpha = 0.1$. We first focus on the measure of overall loss. From the results of overall loss (Figures~\ref{fig:small-overall}, \ref{fig:medium-overall}, \ref{fig:large-overall}), we can see that under the case of small regularizer ($\alpha = 0.1$), Ader achieves the best, and Scream is comparable, while the performance of OGD is not good; with the medium regularizer ($\alpha = 1$), Scream evidently ranks the first, whereas Ader and OGD are not well-behaved; under the case of large regularizer ($\alpha = 2$), OGD performs surprisingly well, and Scream is comparable, whereas the performance of Ader is not desired. The results accord to our theory well, especially after a further examination of corresponding cumulative loss (Figures~\ref{fig:small-loss}, \ref{fig:medium-loss}, \ref{fig:large-loss}) and switching cost (Figures~\ref{fig:small-sc}, \ref{fig:medium-sc}, \ref{fig:large-sc}). Indeed, we can observe that Ader focuses on optimizing the dynamic regret (i.e., cumulative loss) but fails to control the switching cost; and OGD indeed yields a sequence of slow-moving decisions, but it fails to optimize the dynamic regret. Consequently, when the regularizer is small, one can optimize the overall loss by simply forgetting about the switching cost, and this is why Ader could behave well in this setting. Moreover, the switching cost plays a more important role in the overall loss with a large regularizer. Therefore, the algorithm can optimize the overall loss by simply producing a sequence of slow-moving decisions regardless of regret minimization. This is why OGD could achieve a surprisingly good performance in this setting. However, under the non-degenerate settings (for example, with medium regularizer), the two compared methods behave badly and Scream achieves the best. It is because our proposed Scream algorithm strikes a good balance between minimizing the dynamic regret and controlling the switching cost, owing to the novel  online ensemble structure via the introduced switching-cost-regularized loss. Therefore, the above empirical studies demonstrate the effectiveness of our proposed algorithm and its algorithmic components.

\subsection{Online Non-stochastic Control}
\label{sec:experiment-control}
This part further examines the performance of our proposed algorithm in online non-stochastic control.

\paragraph{Settings.} We conduct the experiments in synthetic linear dynamical system (LDS) environments and a real inverted pendulum environment. For the synthetic environment, we consider a time-varying LDS governed by $x_{t+1} = A_t x_t + B_t u_t + w_t$, where $w_t$ is the Gaussian noise, $A_t$ and $B_t$ are the time-varying system matrices to be specified later. It is generally challenging to control time-varying systems, and we here consider a special case that can be handled by the online non-stochastic control framework. Specifically, we design the system matrices as $A_t=A+\Delta_{t,A}$ and $B_t=B+\Delta_{t,B}$, where $A$ and $B$ are fixed, and $\Delta_{t,A},\Delta_{t,B}$ are time-varying zero-mean Gaussian random matrices. Notably, when applying online non-stochastic control methods, we only need to access $A$ and $B$, and the changes of system matrices can be treated as a part of disturbance. Indeed, we have $x_{t+1} = A x_t + B u_t + (w_t + \Delta_{t,A} x_t + \Delta_{t,B} u_t) = A x_t + B u_t + \tilde{w}_t$, where $\tilde{w}_t$ is the effective disturbance of this time-varying system. Moreover, we choose the quadratic loss as the online cost function, defined as $c_t(x_t, u_t) = x_t^\T Q_t x_t + u_t^\T R_t u_t$, where $Q_t = a_t I$ and $R_t = b_t I$ change over time. By setting different $a_t$ and $b_t$, we simulate the following two environments. (1) gradual change: in which $a_t = \sin (t / (10 \pi))$ and $b_t = \sin (t / (20 \pi))$; (2) abrupt change: the whole time horizon is divides into five stages, and the cost functions only change between different stages. In addition, we examine the performance in the real inverted pendulum environment, which is a commonly used benchmark consisting of a nonlinear and unstable system. The goal of this task is to balance the inverted pendulum by applying torque that will stabilize it in a vertically upright position. The state is a 2-dimensional vector denoted by $x_t = [\theta_t, \dot{\theta}_t]^\T$, where the first entry $\theta_t$ is the deviation angle normalized between $[-\pi,\pi]$ and the second entry $\dot{\theta}_t$ is the rotational velocity. The action is a 1-dimensional $u_t = \ddot{\theta}_t$ representing the torque applied on the system. The inverted pendulum environment is a non-linear dynamical system with transitions 
\begin{equation*}
    x_{t+1} = \begin{bmatrix} \theta_{t+1} \\ \dot{\theta}_{t+1} \end{bmatrix} = \begin{bmatrix} \theta_{t} + c\dot{\theta}_t \\ \dot{\theta}_{t} + a \sin(\theta_t+\pi) + b \ddot{\theta}_t \end{bmatrix}.
\end{equation*}
and the online cost function is set as $c_t(x_t, u_t) = a_t \theta_t^2 + b_t \dot{\theta}_t^2 + c_t \ddot{\theta}_t^2$, where $a_t=\sin (t / (10 \pi))$, $b_t=\sin (t / (20 \pi))$, and $c_t=\sin (t / (20 \pi))$ are slowly evolving parameters.

\begin{figure}[!t]
\centering
    \subfigure[LDS, gradual]{ \label{fig:lds-slow}
        \includegraphics[clip, trim=0.7cm 0.2cm 1.3cm 1.3cm,height=0.24\textwidth]{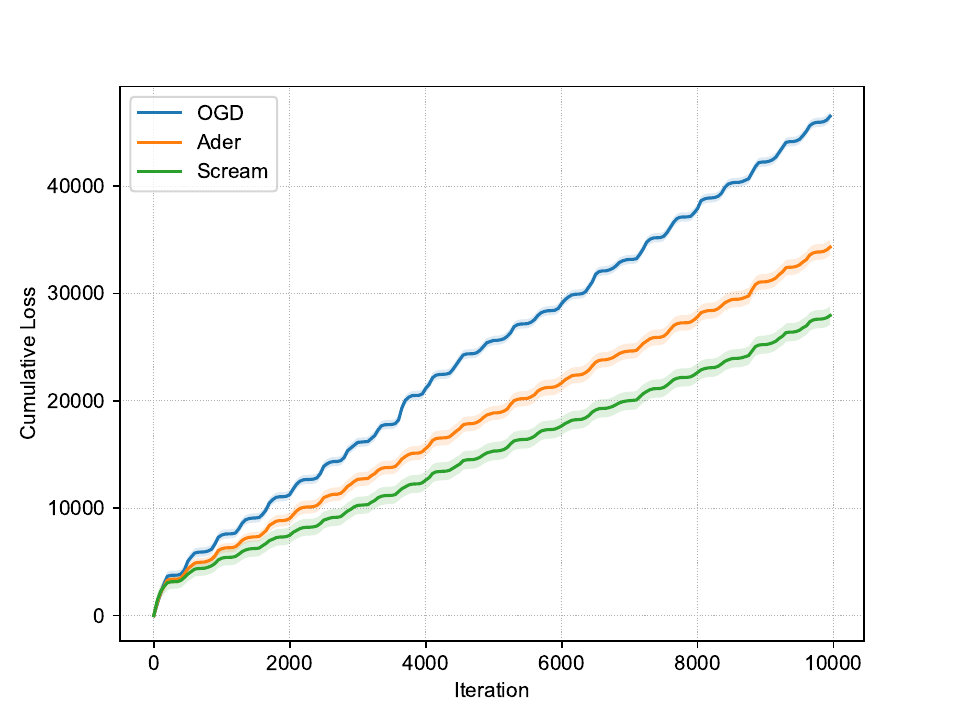}}
    \subfigure[LDS, abrupt]{ \label{fig:lds-abrupt}
        \includegraphics[clip, trim=0.7cm  0.2cm 1.3cm 1.3cm,height=0.24\textwidth]{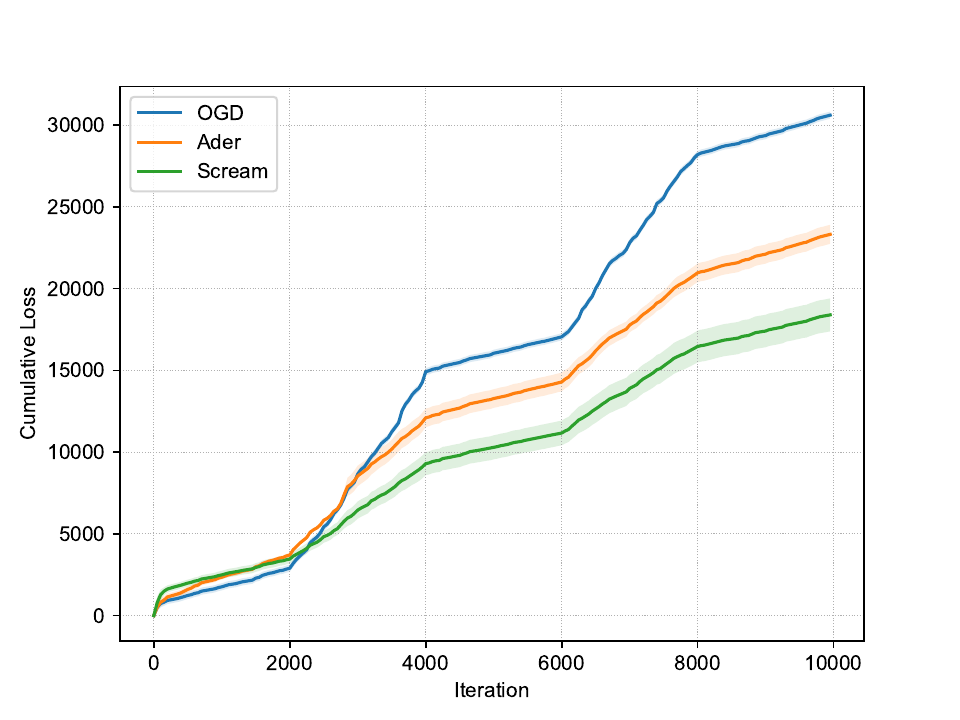}}
    \subfigure[Inverted Pendulum]{ \label{fig:pendulum}
        \includegraphics[clip, trim=1cm 0.2cm 1.3cm 1.3cm,height=0.24\textwidth]{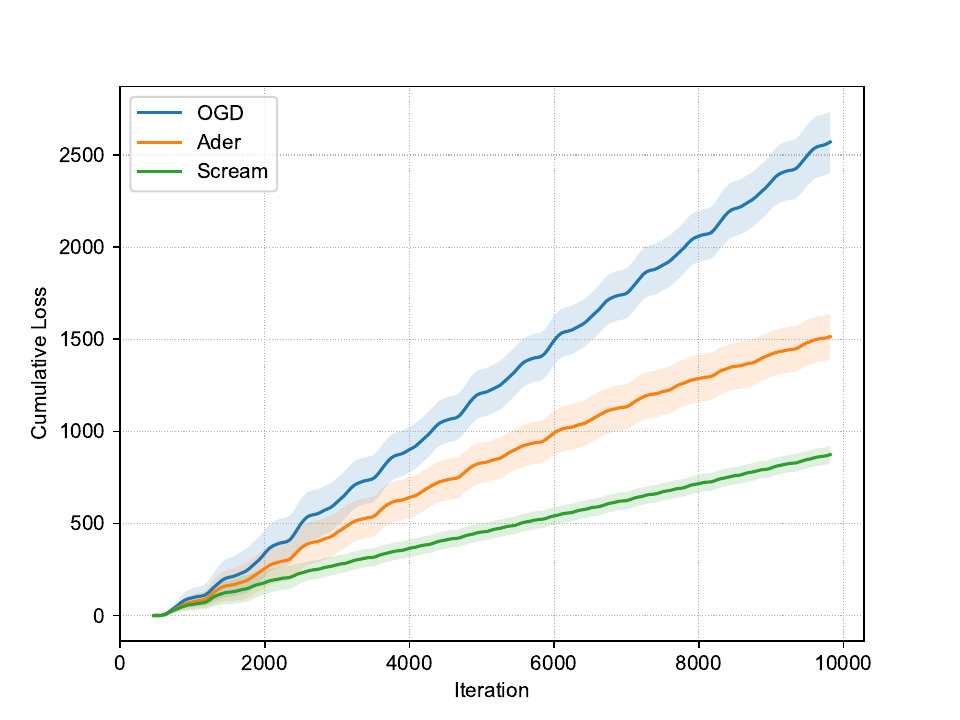}}
  \caption{Performance comparisons of different algorithms. The performance is measured by the cumulative loss, the smaller the better. From left to right: (a) synthetic time-varying LDS with gradual changes; (b) synthetic time-varying LDS with abrupt changes; (c) real pendulum environments.}
  \label{fig:comparison-control}
\end{figure}

\paragraph{Contenders and Measure.} We benchmark our proposed Scream.Control algorithm with the following two algorithms: (1) OGD.Control, which uses the OGD algorithm for the online non-stochastic control~\citep{ICML'19:online-control}; (2) Ader.Control, Ader is an OCO algorithm~\citep{NIPS'18:Zhang-Ader} that admits a two-layer structure and enjoys dynamic regret guarantee. Although it cannot deal with the OCO with memory problem (see discussions in Section~\ref{sec:OCOmemory-challenge}), we apply it for online-non-stochastic control, serving to validate the effectiveness of our proposed switching-cost-regularized surrogate loss. We denote the three control algorithms simply as ``OGD'', ``Ader'', and ``Scream'' when there is no confusion. We record the cumulative loss as the performance measure, namely, $\sum_{t=1}^T c_t(x_t, u_t)$. We repeat the experiments five times and report the mean and standard deviation.

\paragraph{Results.} Figure~\ref{fig:comparison-control} plots the performance comparison of three algorithms (OGD, Ader, Scream) in terms of the cumulative cost. The result shows that our proposed algorithm outperforms the other two contenders, which validates that the meta-base structure (compared with OGD) and the switching-cost-regularizer (compared with Ader) are necessary for online non-stochastic control problems in non-stationary environments. 
 \section{Conclusion}
\label{sec:conclusion}
This paper investigates the dynamic policy regret of online convex optimization with memory and online non-stochastic control. For OCO with memory, we propose the Scream algorithm and prove an optimal $\O(\sqrt{T(1+P_T)})$ dynamic policy regret, where $P_T$ is the path length of comparators that reflects the environmental non-stationarity. Our approach admits the meta-base online ensemble structure to handle uncertain environments and introduces a novel meta-base decomposition via switching-cost regularized loss to algorithmically address the tension between dynamic regret and switching cost. The approach is further used to design robust controllers for online non-stochastic control, where the underlying disturbance and cost functions could be chosen adversarially. We adopt the DAC parameterization and design the Scream.Control algorithm that provably achieves an $\Ot(\sqrt{T(1+P_T)})$ dynamic policy regret, where $P_T$ is the path length of compared controllers. Minimizing dynamic policy regret facilitates our controller with more robustness, since it can compete with any sequence of time-varying controllers instead of a fixed one.

In the future, we will explore the possibility of extension to \emph{bandit} feedback, where the only feedback to the controller is the loss value~\citep{NIPS'20:bandit-Koren,NIPS'20:bandit-control}. Moreover, it would be also intriguing to investigate whether dynamic policy regret can be improved when the cost functions are \emph{strongly convex} or \emph{exponentially concave}~\citep{ICML'20:log-control,COLT'21:baby-strong-convex,AISTATS'22:sc-proper}.  
\section*{Acknowledgments}
Peng Zhao, Yu-Hu Yan, and Zhi-Hua Zhou were supported by the National Science Foundation of China (61921006, 62206125) and JiangsuSF (BK20220776) and the Collaborative Innovation Center of Novel Software Technology and Industrialization. Peng Zhao was also supported by National Postdoctoral Program for Innovative Talent. Yu-Xiang Wang was supported by a startup grant from UCSB CS Department. Part of this work was conducted while Peng Zhao remotely visited UCSB in summer 2020. The authors thank Ming Yin and Dheeraj Baby for helpful discussions. We are also grateful for the anonymous reviewers for their insightful comments.

\appendix
\section{Preliminaries}
\label{sec:preliminary}
In this section, we present the preliminaries, including the dynamic regret results of memoryless online convex optimization, additional notions, and some technical lemmas. 

\subsection{Dynamic Regret of Memoryless OCO}
\label{sec:appendix-dynamic}

In this part we present the dynamic regret analysis of the online gradient descent (OGD) algorithm for memoryless online convex optimization~\citep{ICML'03:zinkvich,NIPS'18:Zhang-Ader,NIPS'20:sword}.

We first specify the problem settings and notations of memoryless online convex optimization. Specifically, the player iteratively selects a decision $\w \in \W$ from a convex set $\W \subseteq \R^d$ and then suffers a loss of $f_t(\w_t)$, in which the loss function $f_t :\W \mapsto \R$ is assumed to be convex and chosen adversarially by the environments. The performance measure we are concerned with is the \emph{dynamic regret}, defined as
\[
  \textnormal{D-Regret}_T(\v_1,\ldots,\v_T) = \sum_{t=1}^T f_t(\w_t) - \sum_{t=1}^T f_t(\v_t),
\]
where $\v_1,\ldots,\v_T \in \W$ is the comparator sequence arbitrarily chosen in the domain by the environments. The critical advantage of the above measure is that it supports to compete with a sequence of \emph{time-varying} comparators, instead of a fixed one as specified in the standard (static) regret.

In the development of dynamic regret of memoryless OCO, one of the most crucial building blocks is the well-known Online Gradient Descent (OGD) algorithm~\citep{ICML'03:zinkvich}, which starts from any $\w_1 \in \W$ and performs the following update, 
\begin{equation}
  \label{eq:appendix-OGD}
  \w_{t+1} = \Pi_{\W}[\w_t - \eta \nabla f_t(\w_t)].
\end{equation}
Here, $\eta > 0$ is the step size and $\Pi_{\W}[\cdot]$ denotes the Euclidean projection onto the nearest point in the feasible domain $\W$. The standard textbooks of online convex optimization~\citep{book'12:Shai-OCO,book'16:Hazan-OCO} show that OGD can achieves an optimal $\O(\sqrt{T})$ static regret for convex functions, providing with appropriate step size settings. Furthermore, such a simple algorithm actually also enjoys the following dynamic regret guarantee~\citep[Theorem 2]{ICML'03:zinkvich}, and we supply the proof for self-containedness.
\begin{myThm}
\label{thm:appendix-dynamic-regret-OCO}
Let $\W \in \R^d$ be a bounded convex and compact set in Euclidean space, and we denote by $D$ an upper bound of the diameter of the domain, i.e., $\norm{\w - \w'}_2 \leq D$ holds for any $\w, \w' \in \W$. Suppose the gradient norm of $f_t$ over $\W$ is bounded by $G$, i.e., $\norm{\nabla f_t(\w)}_2 \leq G$ holds for any $\w \in \W$ and $t \in [T]$. Then, OGD~\eqref{eq:appendix-OGD} enjoys the following dynamic regret,
\begin{align*}
    \textnormal{D-Regret}_T(\v_1,\ldots,\v_T) \leq \frac{\eta}{2} G^2 T + \frac{1}{2\eta}(D^2 + 2DP_T),
\end{align*}
which holds for any comparator sequence $\v_1,\ldots,\v_T \in \W$, and $P_T = \sum_{t=2}^{T} \norm{\v_{t-1} - \v_{t}}_2$ is the path length that measures the cumulative movements of the comparator sequence.
\end{myThm}

\begin{proof}
Since the online functions are convex, we have 
\begin{equation*}
  \textnormal{D-Regret}_T(\v_1,\dots,\v_T) = \sum_{t=1}^T f_t(\w_t) - \sum_{t=1}^{T} f_t(\v_t) \leq \sum_{t=1}^{T} \inner{\nabla f_t(\w_t)}{\w_t - \v_t}.
\end{equation*}

Thus, it suffices to bound the sum of $\inner{\nabla f_t(\w_t)}{\w_t - \v_t}$ over iterations. Note that from the update rule in \eqref{eq:OGD-update},
\[
  \begin{split}
    \norm{\w_{t+1} - \v_t}_2^2 &=  \left\lVert\Pi_{\mathcal{X}}[\w_t - \eta \nabla f_t(\w_t)] - \v_t \right\rVert_2^2\\
    & \leq \norm{\w_t - \eta \nabla f_t(\w_t) - \v_t}_2^2\\
    & = \eta^2 \norm{\nabla f_t(\w_t)}_2^2 - 2\eta \inner{\nabla f_t(\w_t)}{\w_t - \v_t} + \norm{\w_t - \v_t}_2^2 
  \end{split}
\]
The inequality holds due to Pythagorean theorem~\citep[Theorem 2.1]{book'16:Hazan-OCO}. After rearranging, we obtain
\[
  \inner{\nabla f_t(\w_t)}{\w_t - \v_t} \leq \frac{\eta}{2} \norm{\nabla f_t(\w_t)}_2^2 + \frac{1}{2\eta}\left(\norm{\w_t - \v_t}_2^2 - \norm{\w_{t+1} - \v_t}_2^2 \right).
\]
Summing the above inequality from $t=1$ to $T$ yields,
\begin{align*}
    \textnormal{D-Regret}_T(\v_1,\ldots,\v_T) \leq \frac{\eta}{2}\sum_{t=1}^{T}\norm{\nabla f_t(\w_t)}_2^2 + \frac{1}{2\eta}\sum_{t=1}^{T}\left(\norm{\w_t - \v_t}_2^2 - \norm{\w_{t+1} - \v_t}_2^2 \right).
\end{align*}
We further provide an upper bound for the second term on the right-hand side. Indeed,
\begin{align*}
  & \sum_{t=1}^{T}\left(\norm{\w_t - \v_t}_2^2 - \norm{\w_{t+1} - \v_t}_2^2 \right) \\
  \leq {} & \sum_{t=1}^{T} \norm{\w_t - \v_t}_2^2 - \sum_{t=2}^{T} \norm{\w_t - \v_{t-1}}_2^2\\
  \leq {} & \norm{\w_1 - \v_1}_2^2 + \sum_{t=2}^{T} \left(\norm{\w_t - \v_{t}}_2^2 - \norm{\w_t - \v_{t-1}}_2^2\right)\\
  = {} & \norm{\w_1 - \v_1}_2^2 + \sum_{t=2}^{T} \inner{\v_{t-1} - \v_{t}}{2\w_t - \v_{t-1} - \v_{t}} \leq D^2 + 2D \sum_{t=2}^{T} \norm{\v_{t-1} - \v_{t}}_2.
\end{align*}
Combining all above inequalities, we have 
\begin{align*}
    \textnormal{D-Regret}_T(\v_1,\ldots,\v_T) \leq {} & \frac{\eta}{2}\sum_{t=1}^{T}\norm{\nabla f_t(\w_t)}_2^2 +  \frac{1}{2\eta}\left(D^2 + 2D \sum_{t=2}^{T} \norm{\v_{t-1} - \v_{t}}_2\right) \\
    \leq {} & \frac{\eta}{2} G^2 T + \frac{1}{2\eta}(D^2 + 2DP_T).
\end{align*}
Hence, we complete the proof.
\end{proof}

\subsection{Additional Notions}
\label{sec:appendix-notions}

We introduce the formal definition of strongly stable linear controllers~\citep{ICML'18:Cohen,ICML'19:online-control}. Indeed, the stable condition can guarantee the convergence, but nothing can be ensured about the rate of convergence. While working on the class of strongly stable controllers, we can establish the non-asymptotic convergence rate.

\begin{myDef}
\label{def:controller-set}
A linear controller $K$ is $(\kappa, \gamma)$-strongly stable if there exist matrices $L, H$ satisfying $A - BK = HLH^{-1}$, such that the following two conditions are satisfied:
\begin{itemize}
  \item[(i)] The spectral norm of $L$ satisfies $\norm{L} \leq 1 -\gamma$.
  \item[(ii)] The controller and transforming matrices are bounded, i.e., $\norm{K}, \norm{H}, \norm{H^{-1}} \leq \kappa$. 
\end{itemize}
\end{myDef}

\subsection{Technical Lemmas}
\label{appendix:tech-lemmas}
The following lemmas are important in analyzing algorithms based on the mirror descent.

\begin{myLemma}[Lemma 3.2 of~\citet{OPT'93:Bregman}]
\label{lemma:bregman-divergence}
Let $\X$ be a convex set in a Banach space $\mathcal{B}$ and $f: \X \mapsto \R$ be a closed proper convex function on $\X$. Given a convex regularizer $\Rcal:\X \mapsto \R$ and its induced Bregman divergence $\D_\Rcal(\cdot,\cdot)$, any update of the form
\[
  \x_k = \argmin_{\x \in \X} \{ f(\x) + \D_\Rcal(\x,\x_{k-1})\}
\]
satisfies the following inequality for any $\u \in \X$,
\begin{equation*}
  f(\x_k) - f(\u) \leq \D_\Rcal(\u, \x_{k-1}) - \D_\Rcal(\u, \x_{k}) - \D_\Rcal(\x_k, \x_{k-1}).
\end{equation*}
\end{myLemma}

\begin{myLemma}
\label{lemma:bregman-divergence-sc}
If the regularizer $\Rcal:\X \mapsto \R$ is $\lambda$-strongly convex with respect to a norm $\| \cdot \|$, then the induced Bregman divergence is lower-bounded as $\D_\Rcal(\x,\y) \geq \frac{\lambda}{2} \norm{\x - \y}^2$.
\end{myLemma}
\begin{proof}
 By the definition of strong convexity, we know that for any $\x,\y \in \X$, $\psi(\x) \geq \psi(\y) + \nabla \psi(\y)^\T (\x - \y) + \frac{\lambda}{2} \norm{\x - \y}^2$. Reformulating the inequality and combining the definition of Bregman divergence, we know that $D_{\psi}(\x,\y) \define \psi(\x) - \psi(\y) + \nabla \psi(\y)^\T (\x - \y) \geq \frac{\lambda}{2} \norm{\x - \y}^2$, which ends the proof.
\end{proof}

The following concentration inequality is used in analyzing dynamic policy regret for non-stochastic control with unknown systems.
\begin{myLemma}[Azuma-Hoeffding's Inequality for Vectors~{\citep[Theorem 1.8]{vector-Azuma}}]
  \label{suplem:vector-azuma}
  Suppose that $S_m = \sum_{t=1}^m X_t$ is a martingale where $X_1,\ldots,X_m$ take values in $\R^n$ and are such that $\E[X_t]=\mathbf{0}$ and $\norm{X_t}_2 \le D$ for all $t$, for $t >0$. Then for every $\epsilon >0$, 
  \begin{equation*}
      \Pr[\norm{S_m}_2 \ge \epsilon] \le 2e^2 e^{-\frac{\epsilon^2}{2mD^2}}.
  \end{equation*}
\end{myLemma} \section{Omitted Details for Section~\ref{sec:OCOwMemory} (OCO with Memory)}
\label{sec:appendix-OCOwithMemory}
In this section, we present omitted details for Section~\ref{sec:OCOwMemory} OCO with memory, including proofs of Theorem~\ref{thm:dynamic-regret-OCOwMemory} (in Appendix~\ref{sec:proof-OCOwMemory-static}) and Theorem~\ref{thm:main-result-OCOmemory} (in Appendix~\ref{sec:appendix-proof-OCOmemory}). Moreover, we provide the proof of the switching cost decomposition~\eqref{eq:sc-decompose} in Appendix~\ref{sec:sc-decompose} and supply more details for the online mirror descent in Appendix~\ref{sec:appendix-proof-OMD}. The proofs of Theorem~\ref{thm:main-result-OCOmemory}, Theorem~\ref{thm:PEA-SC-lower}, Theorem~\ref{thm:optimal-memory}, Theorem~\ref{thm:lower-dynamic} are listed in the following sections. We finally discuss the memory dependence in Appendix~\ref{sec:appendix-memory-dependency}.

\subsection{Proof of Theorem~\ref{thm:dynamic-regret-OCOwMemory}}
\label{sec:proof-OCOwMemory-static}
\begin{proof}
The coordinate-Lipschitz continuity of $f_t$ (Assumption~\ref{assume:Lipschitz}) implies that
\begin{align*}
  \abs{f_t(\w_{t-m},\ldots,\w_t) - \f_t(\w_t)} \leq L \cdot \sum_{i=1}^m \norm{\w_t - \w_{t-i}}_2 \leq mL \sum_{i=1}^m \norm{\w_{t-i+1} - \w_{t-i}}_2.
\end{align*}
Therefore, we have
\begin{equation}
  \label{eq:memory-inequality}
  \sum_{t=m}^{T} f_t(\w_{t-m},\ldots,\w_t) - \sum_{t=m}^{T} \f_t(\w_t) \leq m^2 L \sum_{t=m}^{T} \norm{\w_t - \w_{t-1}}_2,
\end{equation}
and the dynamic policy regret can be thus upper bounded by
\begin{equation}
\label{eq:dynamic-regret-split}
  \begin{split}
  {} & \textnormal{D-Regret}_T(\v_1,\ldots,\v_T) = \sum_{t=1}^{T} f_t(\w_{t-m},\ldots,\w_t) - \sum_{t=1}^{T} f_t(\v_{t-m},\ldots,\v_t) \\
  \overset{\eqref{eq:memory-inequality}}{\leq} {} & \underbrace{\sum_{t=1}^{T} \f_t(\w_t) - \sum_{t=1}^{T} \f_t(\v_t)}_{\mathtt{dynamic~regret~over~unary~loss}} + \underbrace{\lambda \sum_{t=1}^{T} \norm{\w_{t} - \w_{t-1}}_2}_{\mathtt{switching~cost~of~decisions}} + \underbrace{\lambda \sum_{t=1}^{T} \norm{\v_{t} - \v_{t-1}}_2}_{\mathtt{switching~cost~of~comparators}}, 
  \end{split}
\end{equation}
where we define $\lambda \define  m^2L$ for notational convenience. Note that the first term is the dynamic regret over the unary loss, which is optimized by OGD over the unary loss. Since the sequence of unary loss $\{\f_t\}_{t=1}^T$ is convex and \emph{memoryless}, from the standard dynamic regret analysis~\citep{ICML'03:zinkvich,NIPS'18:Zhang-Ader}, as shown in Theorem~\ref{thm:appendix-dynamic-regret-OCO}, we get
\begin{equation}
  \label{eq:dynamic-regret-bound}
  \sum_{t=1}^T \f_t(\w_t) - \sum_{t=1}^T \f_t(\v_t) \leq \frac{\eta}{2} G^2 T + \frac{1}{2\eta} (D^2 + 2DP_T),
\end{equation}
where $P_T = \sum_{t=2}^{T} \norm{\v_t - \v_{t-1}}_2$ is the path length measuring the fluctuation of the comparator sequence $\v_1, \v_2,\ldots,\v_{T}$. Next, the last term of~\eqref{eq:dynamic-regret-split} is the switching cost of the comparators, which is exactly the path length $\lambda P_T$. 

So we only need to further examine the switching cost of the decisions, i.e., $\sum_{t=2}^{T} \norm{\w_{t-1} - \w_{t}}_2$, as well as the dynamic regret over the unary loss, i.e., $\sum_{t=1}^{T} \f_t(\w_t) - \sum_{t=1}^{T}\f_t(\v_t)$. By the non-expansive property of the projection operator, we can derive an upper bound for the switching cost:
\begin{equation}
  \label{eq:switching-cost-bound}
  \sum_{t=1}^T \norm{\w_{t} - \w_{t-1}}_2 = \sum_{t=1}^T \norm{\Pi_{\W}[\w_{t-1} - \eta \nabla \f_t(\w_{t-1})] - \w_{t-1}}_2 \leq \eta \sum_{t=1}^T \norm{\nabla \f_t(\w_{t-1})}_2 \leq \eta GT.
\end{equation}
Combining above two inequalities~\eqref{eq:switching-cost-bound} and~\eqref{eq:dynamic-regret-bound} yields
\[
  \sum_{t=1}^{T} f_t(\w_{t-m},\ldots,\w_t) - \sum_{t=1}^{T} f_t(\v_{t-m},\ldots,\v_t) \leq \frac{\eta}{2}(G^2 + 2\lambda G) T + \frac{1}{2\eta} (D^2 + 2DP_T) + \lambda P_T,
\]
with $\lambda = m^2 L$. We thus compete the proof.
\end{proof}

\subsection{{Proof of Switching Cost Decomposition}}
\label{sec:sc-decompose}
The following lemma restates the switching cost decomposition presented in~\eqref{eq:sc-decompose}. 

\begin{myLemma}
\label{lemma:sc-decompose}
The switching cost of meta-base outputs can be upper bounded as
\[
  \sum_{t=2}^{T} \norm{\w_{t} - \w_{t-1}}_2 \leq D \sum_{t=2}^{T} \norm{\p_t - \p_{t-1}}_1 + \sum_{t=2}^{T} \sum_{i=1}^{N} p_{t,i} \norm{\w_{t,i} - \w_{t-1,i}}_2.
\]
\end{myLemma}

\begin{proof}
By the meta-base structure, the final decision of each round is $\w_t = \sum_{i=1}^{N} p_{t,i} \w_{t,i}$. Therefore, we can expand the switching cost of the final prediction sequence as
\begin{align}
  \norm{\w_{t} - \w_{t-1}}_2 = {}& \left\| \sum_{i=1}^{N} p_{t,i}\w_{t,i} - \sum_{i=1}^{N} p_{t-1,i}\w_{t-1,i} \right\|_2 \nonumber\\
  \leq {}& \left\| \sum_{i=1}^{N} p_{t,i}\w_{t,i} - \sum_{i=1}^{N} p_{t,i}\w_{t-1,i} \right\|_2 + \left\| \sum_{i=1}^{N} p_{t,i}\w_{t-1,i} - \sum_{i=1}^{N} p_{t-1,i}\w_{t-1,i} \right\|_2 \nonumber \\
  \leq {} & \sum_{i=1}^{N} p_{t,i} \norm{\w_{t,i} - \w_{t-1,i}}_2 + D \sum_{i=1}^{N} \abs{p_{t,i} - p_{t-1,i}}\nonumber \\
  ={} & \sum_{i=1}^{N} p_{t,i} \norm{\w_{t,i} - \w_{t-1,i}}_2 + D \norm{\p_t - \p_{t-1}}_1, \label{eq:sc-upper}
\end{align}
where the second step holds due to the triangle inequality and the third step is true due to the boundedness of the feasible domain (Assumption~\ref{assume:bounded-norm}). Hence, we complete the proof.
\end{proof}

\subsection{Additional Results for Online Mirror Descent}
\label{sec:appendix-proof-OMD}
In this section, we present additional results and descriptions for Online Mirror Descent (OMD), which enables a unified view for algorithm design of both meta-algorithm and base-algorithm. 

Consider the standard online convex optimization setting, and the sequence of online convex functions are $\{h_t\}_{t=1,\ldots,T}$ with $h_t :\W \mapsto \R$. Online mirror descent starts from any $\w_1 \in \W$, and at iteration $t$, the algorithm performs the following update:
\begin{equation}
  \label{eq:OMD}
  \w_{t+1} = \argmin_{\w \in \W} \eta \inner{\nabla h_t(\w_t)}{\w} + \D_{\Rcal}(\w,\w_t),
\end{equation}
where $\eta > 0$ is the step size. The regularizer $\Rcal: \W \mapsto \R$ is a differentiable convex function defined on $\W$ and is assumed (without loss of generality) to be $1$-strongly convex w.r.t. some norm $\| \cdot \|$ over $\W$. The induced Bregman divergence $\D_{\Rcal}$ is defined by $\D_{\Rcal}(\x,\y) = \Rcal(\x) - \Rcal(\y) - \inner{\nabla \Rcal(\y)}{\x - \y}$.

The following generic result gives an upper bound of  dynamic regret with switching cost of OMD, which can be regarded as a generalization of Theorem~\ref{thm:dynamic-regret-OCOwMemory} from gradient descent (for Euclidean norm) to  mirror descent (for general primal-dual norm).
\begin{myThm}
\label{thm:OMD}
Online Mirror Descent~\eqref{eq:OMD} satisfies that
\begin{equation}
  \label{eq:OMD-dynamic-Reg}
  \sum_{t=1}^{T} h_t(\w_t) - \sum_{t=1}^{T} h_t(\v_t) + \lambda \sum_{t=2}^{T}\norm{\w_t - \w_{t-1}} \leq \frac{1}{\eta}\left( R^2 + \gamma P_T \right) + \eta(\lambda G + G^2) T,
\end{equation}
provided that $\D_{\Rcal}(\x,\z)-\D_\Rcal(\y,\z) \leq \gamma\norm{\x-\y}$ holds for any $\x,\y,\z\in\W$. In above, $R^2 = \sup_{\x,\y\in\W} \D_\Rcal(\x,\y)$, and $G = \sup_{\w \in \W, t \in [T]} \norm{\nabla h_t(\w)}_*$. Note that the above result holds for any comparator sequence $\v_1,\ldots,\v_T \in \W$. 
\end{myThm}

\begin{myRemark}
The dynamic regret of Theorem~\ref{thm:OMD} holds against \emph{any} comparator sequence in the domain. In particular, we can set them as the best fixed decision in hindsight and thus obtain  static regret with switching cost, $\sum_{t=1}^{T} h_t(\w_t) - \sum_{t=1}^{T} h_t(\w^*) + \lambda \sum_{t=2}^{T}\norm{\w_t - \w_{t-1}} \leq R^2/\eta + \eta(\lambda G + G^2) T$, that holds for any $\w^* \in \W$. A technical caveat is that  when deriving the static regret, the Bregman divergence is not required to satisfy the Lipschitz condition.
\end{myRemark}

Theorem~\ref{thm:OMD} exhibits a general analysis for the dynamic regret and switching cost of OMD. By flexibly choosing the regularizer $\Rcal$ and comparator sequence $\v_1,\ldots,\v_T$, we have the following two implications, which correspond to base-regret (dynamic regret with switching cost of OGD) and meta-regret (static regret with switching cost of Hedge) respectively.

Before presenting the proof of Theorem~\ref{thm:OMD}, we first analyze the switching cost of the online mirror descent, as demonstrated in the following stability lemma.
\begin{myLemma}
\label{lemma:OMD-switching-cost}
For Online Mirror Descent~\eqref{eq:OMD}, the instantaneous switching cost is at most
\begin{equation}
  \label{eq:OMD-switching-cost}
  \norm{\w_t - \w_{t+1}} \leq \eta \norm {\nabla h_t(\w_t)}_*.
\end{equation}
\end{myLemma}
\begin{proof}
From the update procedure of OMD~\eqref{eq:OMD} and Lemma~\ref{lemma:bregman-divergence}, we know that 
\[
  \inner{\w_{t+1} - \w_t}{\eta \nabla h_t(\w_t)} \leq \DR{\w_t}{\w_{t}} - \DR{\w_t}{\w_{t+1}} - \DR{\w_{t+1}}{\w_t},
\]
which implies
\[
  \DR{\w_t}{\w_{t+1}} + \DR{\w_{t+1}}{\w_t} \leq \inner{\w_{t} - \w_{t+1}}{\eta \nabla h_t(\w_t)}.
\]
Since the regularizer $\Rcal$ is chosen as a $1$-strongly convex function with respect to the norm $\| \cdot \|$, by Lemma~\ref{lemma:bregman-divergence-sc} we have
\[
  \DR{\w_t}{\w_{t+1}} + \DR{\w_{t+1}}{\w_t} \geq \norm{\w_t - \w_{t+1}}^2.
\]
Combining above two inequalities and further applying the Hölder's inequality, we obtain
\[
  \norm{\w_t - \w_{t+1}}^2 \leq \inner{\w_t - \w_{t+1}}{\eta \nabla h_t(\w_t)} \leq \norm{\w_t - \w_{t+1}} \norm{\eta \nabla h_t(\w_t)}_*.
\]
Therefore, we conclude that $\norm{\w_t - \w_{t+1}} \leq \eta \norm{\nabla h_t(\w_t)}_*$ and finish the proof.
\end{proof}

Based on the above stability lemma, we can now prove  Theorem~\ref{thm:OMD} regarding dynamic regret with switching cost for \textsc{OMD}.

\begin{proof}[{of Theorem~\ref{thm:OMD}}]
Notice that the dynamic regret can be decomposed as follows:
\begin{align*}
  \sum_{t=1}^{T} h_t(\w_t) - \sum_{t=1}^{T} h_t(\v_t) \leq {} & \sum_{t=1}^{T} \inner{\nabla h_t(\w_t)}{\w_t - \v_t} \\
  = {} & \underbrace{\sum_{t=1}^{T} \inner{\nabla h_t(\w_t)}{\w_{t} - \w_{t+1}}}_{\term{a}} + \underbrace{\sum_{t=1}^{T} \inner{\nabla h_t(\w_t)}{\w_{t+1} - \v_t}}_{\term{b}}.
\end{align*}
From Lemma~\ref{lemma:OMD-switching-cost} and Hölder's inequality, we have
\begin{equation}
  \label{eq:OMD-term-a}
  \term{a} \leq \sum_{t=1}^{T} \norm{\nabla h_t(\w_t)}_* \norm{\w_t - \w_{t+1}} \leq \eta \sum_{t=1}^{T} \norm{\nabla h_t(\w_t)}_*^2.
\end{equation}
Next, we investigate the term (b):
\begin{align}
  \term{b} \leq {} & \frac{1}{\eta} \sum_{t=1}^{T} \left( \DR{\v_t}{\w_t} - \DR{\v_t}{\w_{t+1}} - \DR{\w_{t+1}}{\w_t}\right) \nonumber\\
  \leq {} & \frac{1}{\eta} \sum_{t=2}^{T} \left( \DR{\v_t}{\w_t} - \DR{\v_{t-1}}{\w_t}\right) + \DR{\v_1}{\w_1} \nonumber\\
  \leq {} & \frac{\gamma}{\eta} \sum_{t=2}^{T} \norm{\v_t - \v_{t-1}} + \frac{1}{\eta} R^2, \label{eq:OMD-term-b}
\end{align}
where the first inequality holds due to Lemma~\ref{lemma:bregman-divergence}, and the second inequality makes uses of the non-negativity of the Bregman divergence. The last inequality holds due to the assumption of Lipschitz property that $\D_{\Rcal}(\x,\z)-\D_\Rcal(\y,\z) \leq \gamma\norm{\x-\y}$ holds for any $\x,\y,\z\in\W$. Furthermore, the switching cost can be bounded by Lemma~\ref{lemma:OMD-switching-cost},
\begin{equation}
  \label{eq:OMD-term-switching-cost}
  \sum_{t=2}^{T} \norm{\w_t - \w_{t-1}} \leq \eta \sum_{t=2}^{T} \norm{\nabla h_{t-1}(\w_{t-1})}_*.
\end{equation}
Combining~\eqref{eq:OMD-term-a},~\eqref{eq:OMD-term-b}, and~\eqref{eq:OMD-term-switching-cost}, we can attain that
\begin{align*}
    {} & \lambda \sum_{t=2}^{T} \norm{\w_t - \w_{t-1}} + \sum_{t=1}^{T} h_t(\w_t) - \sum_{t=1}^{T} h_t(\v_t) \\
  \leq {} & \frac{1}{\eta}(R^2 + \gamma P_T) + \eta \sum_{t=1}^{T} (\lambda \norm{\nabla h_{t}(\w_{t})}_* + \norm{\nabla h_{t-1}(\w_{t-1})}_*^2) \\
  \leq {} &  \frac{1}{\eta}(R^2 + \gamma P_T) + \eta (\lambda G + G^2) T,
\end{align*}
which finishes the proof.
\end{proof}

As we mentioned earlier, Theorem~\ref{thm:dynamic-regret-OCOwMemory} can be regarded as a corollary of Theorem~\ref{thm:OMD}, by specifying the Euclidean norm and $\psi(\w) = \frac{1}{2} \norm{\w}_2^2$. We give a formal statement in the following corollary.
\begin{myCor}
\label{corollary:OGD}
Setting the $\ell_2$ regularizer $\psi(\w) = \frac{1}{2} \norm{\w}_2^2$ and step size $\eta > 0$ for OMD, suppose $\norm{\nabla \f_t(\w)}_2 \leq G$ and $\norm{\w - \w'}_2 \leq D$ hold for all $\w \in \W$ and $t \in [T]$, then we have
\begin{equation}
  \label{eq:corollary-OGD}
  \lambda \sum_{t=2}^{T}\norm{\w_t - \w_{t-1}}_2 + \sum_{t=1}^{T} \f_t(\w_t) - \sum_{t=1}^{T} \f_t(\v_t) \leq (G^2 + \lambda G) \eta T + \frac{1}{2\eta}(D^2 + 2DP_T),
\end{equation}
which holds for any comparator sequence $\v_1,\ldots,\v_T \in \W$, and $P_T = \sum_{t=2}^{T} \norm{\v_{t-1} - \v_{t}}_2$ is the path length that measures the cumulative movements of the comparator sequence.
\end{myCor}

Further, we present a corollary regarding the static regret with switching cost for the meta-algorithm, which is essentially a specialization of \textsc{OMD} algorithm by setting the negative-entropy regularizer.
\begin{myCor}
\label{corollary:Hedge}
Setting the negative-entropy regularizer $\Rcal(\p) = \sum_{i=1}^{N} p_i \log p_i$ and learning rate $\epsilon> 0$ for OMD, suppose $\norm{\ellb_t}_\infty \leq G$ holds for any $t \in [T]$ and the algorithm starts from the initial weight $p_1 \in \Delta_N$, then we have
\begin{equation}
  \label{eq:corollary-non-uniform-Hedge}
  \lambda \sum_{t=2}^{T}\norm{\p_t - \p_{t-1}}_1 + \sum_{t=1}^{T} \inner{\p_t}{\ellb_t} - \sum_{t=1}^{T} \ell_{t,i} \leq \frac{\ln(1/p_{1,i})}{\epsilon} + \epsilon(\lambda G + G^2) T.
\end{equation}
\end{myCor}
\begin{proof}
From the proof of Theorem~\ref{thm:OMD}, we can easily obtain that
\[
  \lambda \sum_{t=2}^{T}\norm{\p_t - \p_{t-1}}_1 + \sum_{t=1}^{T} \inner{\p_t}{\ellb_t} - \sum_{t=1}^{T} \ell_{t,i}  \leq \frac{\DR{\mathbf{e}_i}{\p_1}}{\epsilon} + \epsilon (\lambda G + G^2) T.
\] 
When choosing the negative-entropy regularizer, the induced Bregman divergence becomes Kullback-Leibler divergence, i.e., $\DR{\boldsymbol{q}}{\p} = \mbox{KL}(\boldsymbol{q}, \p) = \sum_{i=1}^{N} q_i \ln (q_i/p_i)$. Therefore, $\DR{\boldsymbol{e}_i}{\p_1} = \ln (1/p_{1,i})$, which implies the desired result.
\end{proof}

\subsection{Proof of Theorem~\ref{thm:main-result-OCOmemory}}
\label{sec:appendix-proof-OCOmemory}
\begin{proof}
As indicated in~\eqref{eq:dynamic-regret-split}, the dynamic policy regret can be upper bounded by three terms, including dynamic regret over the unary regret, switching cost of decisions, and switching cost of comparators. The third term is essentially the path length of the comparators, and we focus on the first two terms.
\begin{align*}
  {} & \sum_{t=1}^T \f_t(\w_t) -  \sum_{t=1}^T \f_t(\v_t) + \lambda \sum_{t=2}^T \norm{\w_{t} - \w_{t-1}}_2\\
  \overset{\eqref{eq:sc-decompose}}{\leq} {} & \sum_{t=1}^T \inner{\nabla \f_t(\w_t)}{\w_t - \v_t} + \lambda D \sum_{t=2}^{T} \norm{\p_t - \p_{t-1}}_1 + \lambda \sum_{t=2}^{T} \sum_{i=1}^{N} p_{t,i} \norm{\w_{t,i} - \w_{t-1,i}}_2\\
  = {} & \sum_{t=1}^T \sum_{i=1}^{N} p_{t,i} \Big(\inner{\nabla \f_t(\w_t)}{\w_{t,i}} + \lambda \norm{\w_{t,i} - \w_{t-1,i}}_2\Big) - \sum_{t=1}^T \Big(\inner{\nabla \f_t(\w_t)}{\w_{t,i}} + \lambda \norm{\w_{t,i} - \w_{t-1,i}}_2\Big) \\
  {} & \qquad + \lambda D \sum_{t=2}^{T} \norm{\p_t - \p_{t-1}}_1 + \sum_{t=1}^T \Big(\inner{\nabla \f_t(\w_t)}{\w_{t,i}} - \inner{\nabla \f_t(\w_t)}{\v_t}\Big) + \lambda \sum_{t=2}^T \norm{\w_{t,i} - \w_{t-1,i}}_2\\
  = {} & \underbrace{\sum_{t=1}^T \big( \inner{\p_t}{\ellb_t}  - \ell_{t,i} \big) + \lambda D \sum_{t=2}^{T} \norm{\p_t - \p_{t-1}}_1}_{\meta} + \underbrace{\sum_{t=1}^T  \big( g_t(\w_{t,i}) - g_t(\v_t) \big) + \lambda \sum_{t=2}^T \norm{\w_{t,i} - \w_{t-1,i}}_2}_{\base},
\end{align*}
where the last step uses the convexity of $\f_t$ and the definition of linearized loss $g_t(\w) = \inner{\nabla \f_t(\w_t)}{\w}$. We will formally prove that our proposed algorithm optimizes the right-hand side of above inequality.

\paragraph{Bounding Meta-regret.} Denote by $\mathbf{e}_i$ the $i$-th standard basis of $\R^N$-space and by $\lambda' = \lambda D$ for simplicity. Denote by $G_{\text{meta}} = \max_{t \in [T]} \norm{\ellb_t}_{\infty}$ the maximum scale of the loss of meta-algorithm. Since the meta-algorithm actually performs Hedge over the switching-cost-regularized loss $\ellb_t \in \R^N$, Corollary~\ref{corollary:Hedge} implies that for any $i \in [N]$, 
\begin{equation} 
\label{eq:meta-regret-analysis}
\begin{split}  
  \sum_{t=1}^{T}\inner{\p_t}{\ellb_t} - \sum_{t=1}^{T} \ell_{t,i} + \lambda' \sum_{t=2}^{T}\norm{\p_t - \p_{t-1}}_1 \leq {} & \epsilon (\lambda' G_{\text{meta}} + G_{\text{meta}}^2) T + \frac{\DR{\mathbf{e}_i}{\p_1}}{\epsilon}\\
  = {} & \epsilon (\lambda D +G_{\text{meta}})G_{\text{meta}} T + \frac{\ln(1/p_{1,i})}{\epsilon}\\
  \leq {} & \epsilon (\lambda D +G_{\text{meta}})G_{\text{meta}} T + \frac{2 \ln(i+1)}{\epsilon},
\end{split}
\end{equation}
where the last step holds because we adopt a non-uniform weight initialization with the initial weight $\p_1 \in \Delta_N$ set as $p_{1,i} = \frac{1}{i(i+1)} \cdot \frac{N+1}{N}$ for any $i \in [N]$. By choosing the learning rate as $\epsilon = \epsilon^* = \sqrt{\frac{2}{\Gmeta (\lambda D + \Gmeta)T}}$, we can obtain the following upper bound for the meta-regret, 
\begin{equation}
  \label{eq:meta-upper-final}
  \sum_{t=1}^{T}\inner{\p_t}{\ellb_t} - \sum_{t=1}^{T} \ell_{t,i} + \lambda' \sum_{t=2}^{T}\norm{\p_t - \p_{t-1}}_1 \leq \sqrt{2 \Gmeta (\lambda D + \Gmeta)T}\left( 1 + \ln(i+1) \right).
\end{equation}
Note that the dependence of learning rate tuning on $T$ can be removed by either a time-varying tuning or doubling trick. We now present an upper bound for $G_{\text{meta}}$, indeed,
\begin{align}
  \ell_{t,i} = {} & \inner{\nabla \f_t(\w_t)}{\w_{t,i}} + \lambda \norm{\w_{t,i} - \w_{t-1,i}}_2 \leq \inner{\nabla \f_t(\w_t)}{\w_{t,i}} + \lambda \eta_i \norm{\nabla \f_t(\w_t)}_2 \notag\\
  \leq {} & GD + \lambda \eta_i G \leq GD + \lambda \eta_N G \leq GD \left(1 + 2\lambda \sqrt{\frac{1}{\lambda G + G^2}}\right) = \O(\sqrt{\lambda}) \label{eq:Gmeta}.
\end{align}

\paragraph{Bounding Base-regret.} As specified by our algorithm, there are multiple base-learners, each performing OGD over the linearized loss with a particular step size $\eta_i \in \H$ for base-learner $\B_i$:
\begin{equation*}
  \w_{t+1,i} = \Pi_{\W}[\w_{t,i} - \eta_i \nabla g_t(\w_{t,i})] = \Pi_{\W}[\w_{t,i} - \eta_i \nabla \f_t(\w_t)].
\end{equation*}
As a result, Theorem~\ref{thm:OMD} implies that the base-regret satisfies
\begin{equation}
  \label{eq:expert-upper-final}
  \sum_{t=1}^T g_t(\w_{t,i}) - \sum_{t=1}^T g_t(\v_t) + \lambda \sum_{t=2}^T \norm{\w_{t,i} - \w_{t-1,i}}_2 \leq (G^2 + \lambda G) \eta_i T + \frac{1}{2\eta_i} (D^2 + 2DP_T),
\end{equation}
which holds for any comparator sequence $\v_1,\ldots,\v_T \in \W$ as well as any base-learner $i \in [N]$.

\paragraph{Bounding Overall Dynamic Regret.} Due to the boundedness of the path length, we know that the optimal step size $\eta_*$ provably lies in the range of $[\eta_1, \eta_N]$. Furthermore, by the construction of the pool of candidate step sizes, we can confirm that there exists an index $i^* \in [N]$ ensuring $\eta_{i^*} \leq \eta_* \leq \eta_{i^*+1} = 2\eta_{i^*}$. Therefore, we have  
\begin{equation}
  \label{eq:optimal-index}
  i^* \leq \Big \lceil \frac{1}{2} \log_2 \sbr{1 + \frac{2P_T}{D}} \Big \rceil + 1.
\end{equation}
Notice that the meta-base decomposition at the beginning of the proof holds for any index of base-learners $i \in [N]$. Thus, in particular, we can choose the index $i^*$ and achieve the following result by using the upper bounds of meta-regret~\eqref{eq:meta-upper-final} and base-regret~\eqref{eq:expert-upper-final}.
\begin{align}
  {} & \sum_{t=1}^T \f_t(\w_t) -  \sum_{t=1}^T \f_t(\v_t) + \lambda \sum_{t=2}^T \norm{\w_{t} - \w_{t-1}}_2 \notag\\
  \leq {} & \underbrace{\sum_{t=1}^T \big( \inner{\p_t}{\ellb_t}  - \ell_{t,i^*} \big) + \lambda D \sum_{t=2}^{T} \norm{\p_t - \p_{t-1}}_1}_{\meta} + \underbrace{\sum_{t=1}^T  \big( g_t(\w_{t,i^*}) - g_t(\v_t) \big) + \lambda \sum_{t=2}^T \norm{\w_{t,i^*} - \w_{t-1,i^*}}_2}_{\base} \notag\\
  \leq {} & \sqrt{2 \Gmeta (\lambda D + \Gmeta)T}\left( 1 + \ln(i^*+1) \right) + (G^2 + \lambda G) \eta_{i^*} T + \frac{1}{2\eta_{i^*}} (D^2 + 2DP_T) \notag\\
  \leq {} & \sqrt{2 \Gmeta (\lambda D + \Gmeta)T}\left( 1 + \ln(i^*+1) \right) + (G^2 + \lambda G) \eta_{*} T + \frac{1}{\eta_{*}} (D^2 + 2DP_T) \notag\\
  \lesssim {} & \sqrt{2 (GD+\sqrt{\lambda}) (\lambda D + GD+\sqrt{\lambda})T}\left( 1 + \ln(i^*+1) \right) + \sqrt{(G^2 + \lambda G)(D^2 + 2DP_T)T} \label{eq:Scream}\\
  \leq {} & \O\left(\lambda^{\frac{3}{4}} \sqrt{T} (1+ \log \log P_T)\right) + \O\left(\sqrt{\lambda T(1+P_T)}\right)\notag,
\end{align}
where in \eqref{eq:Scream}, we use $a \lesssim b$ to represent $a = \O(b)$. Therefore, we have 
\begin{align*}
\textnormal{D-Regret}_T(\v_{1:T}) \leq {}& \sum_{t=1}^T \f_t(\w_t) -  \sum_{t=1}^T \f_t(\v_t) + \lambda \sum_{t=2}^T \norm{\w_{t} - \w_{t-1}}_2 + \lambda \sum_{t=2}^T \norm{\v_{t} - \v_{t-1}}_2\\
\leq {}& \O\left(\lambda^{\frac{3}{4}} \sqrt{T} (1+ \log \log P_T) + \sqrt{\lambda T(1+P_T)} + \lambda P_T\right) \leq \O(\sqrt{T(1+P_T)}).
\end{align*}
The last step omits the dependence on $\lambda$. Moreover, the inequality holds due to the following observation:
\begin{align*}
  \textnormal{D-Regret}_T(\v_{1:T}) \leq {}&\O(\sqrt{T(1+P_T)}) + \O(P_T) \\
  \leq {} & \O(\sqrt{T(1+P_T) + P_T^2}) \tag{$\sqrt{a} + \sqrt{b} \leq \sqrt{2(a+b)}$}\\
  = {} & \O(\sqrt{T + (T + P_T) P_T}) \\
  \leq {} &\O(\sqrt{T(1+P_T)}),
\end{align*}
where the last step holds as $P_T = \sum_{t=2}^T \norm{\v_t - \v_{t-1}}_2 \leq DT$ due to the boundedness of the domain. We hence complete the proof of Theorem~\ref{thm:main-result-OCOmemory}.
\end{proof}

\subsection{Proof of Theorem~\ref{thm:PEA-SC-lower}}
\label{sec:appendix-proof-PEA-SC-lower}
\begin{proof}
  First, we show that for any $\lambda > 0$ and $C > 0$, the original online learning problem can be reduced to optimize the following one through a shifting operation, 
  \begin{equation}
    \label{eq:shifting-PEA}
    \sum_{t=1}^T \langle \ellb^\prime_t, \p^\prime_t\rangle - \sum_{t=1}^T \ell^\prime_{t,i^*}+ \lambda \sum_{t=2}^T \| \p^\prime_t - \p^\prime_{t-1} \|_1,
  \end{equation}
  where $\ell^\prime_{t,i} \define \ell_{t,i} + C, \p_t^\prime \define \p_t$ for all $t \in [T], i\in [N]$, and evidently $\ell^\prime_{t,i} \in [0, 2C]$. The above equivalence can be simply proven by plugging the definition of $\ell^\prime_{t,i}$ and $\p_t^\prime$ into \eqref{eq:shifting-PEA}. Formally,
  \begin{align*}
    & \sum_{t=1}^T \langle \ellb^\prime_t, \p^\prime_t\rangle - \sum_{t=1}^T \ell^\prime_{t,i^*}+ \lambda \sum_{t=2}^T \| \p^\prime_t - \p^\prime_{t-1} \|_1\\
    = {} & \sum_{t=1}^T \langle \ellb_t + [C,\ldots,C]^\top, \p_t\rangle - \sum_{t=1}^T (\ell_{t,i^*} + C) + \lambda \sum_{t=2}^T \| \p_t - \p_{t-1} \|_1\\
    = {} & \sum_{t=1}^T \langle \ellb_t, \p_t\rangle - \sum_{t=1}^T \ell_{t,i^*}+ \lambda \sum_{t=2}^T \| \p_t - \p_{t-1} \|_1.
  \end{align*}
  Next, we prove that there exists a sequence of loss functions $\ellb^\prime_1,\ldots,\ellb^\prime_T$ satisfying $\ellb^\prime_t \in [0,2C]^N$ for all $t \in [T]$ such that any feasible expert algorithm (whose output is $\p^\prime_1,\ldots,\p^\prime_T \in \Delta_N$) incurs the following regret
  \begin{equation*}
    \sum_{t=1}^T \langle \ellb^\prime_t, \p^\prime_t\rangle - \sum_{t=1}^T \ell^\prime_{t,i^*}+ \lambda \sum_{t=2}^T \| \p^\prime_t - \p^\prime_{t-1} \|_1 \ge \Omega(\sqrt{\lambda CT}).
  \end{equation*}
  The remaining proof borrows the intuition from Theorem 13 of~\citet{COLT'18:switch-cost}. First we give a hard constraint on the switching cost, e.g., $\sum_{t=2}^T \| \p^\prime_t - \p^\prime_{t-1} \|_1 = S$. Then we divide the time horizon $T$ into $B = 4S^2 / (a^2\log N)$ blocks, each of uniform length $T/B$, where $a$ is some constant to be specified later. For each block $b \in [B]$, assign to each expert $i \in [N]$ a loss sampled from $2C \cdot \mbox{Ber}(1/2)$, i.e., $2C$ with probability $1/2$ and otherwise $0$, for each iteration in that block. Clearly this adversary is oblivious.
  
  Note that the cumulative loss of the $i$-th expert, namely, $\sum_{t=1}^T \ell^\prime_t(i)$, is equal in distribution to $T/B$ times a $2C \cdot \mbox{Bin}(B, 1/2)$ random variable. In the following, we first consider the expected cumulative loss of the best expert. Suppose there are $N$ variables drawn i.i.d. from $\operatorname{Bin}(B, 1/2)$, then the minimum one has the following upper bound.
    \begin{myLemma}
      \label{lem:bin}
      There exists a universal constant $c>0$ such that for all $B, N \in \mathbb{N}_{+}$,
      \begin{equation*}
        \mathbb{E}\left[\min _{i \in[N]} Z_{i}\right] \leq \frac{B}{2}-c \sqrt{B \log N}
      \end{equation*}
      where $\left\{Z_{i}\right\}_{i \in[N]}$ are i.i.d. from $\operatorname{Bin}(B, 1/2)$.
    \end{myLemma}
    The adversary chooses $a$ to be the constant that makes Lemma~\ref{lem:bin} holds. Thus the loss of the best expert satisfies that
    \begin{equation}
      \label{eq:lower-best}
      \begin{aligned}
        & \mathbb{E}\left[\sum_{t=1}^{T} \ell^\prime_{t}(i^*)\right] \leq 2C \cdot \frac{T}{B}\left(\frac{B}{2}-a \sqrt{B \log N}\right)\\
        =  {} & 2C \cdot \sbr{\frac{T}{2} - aT \sqrt{\frac{\log N}{B}}}= 2C \cdot \sbr{\frac{T}{2}- \frac{a^2 T \log N}{2S}}.
      \end{aligned}
    \end{equation}
    Now let us compute the expected loss of any algorithm $\A$ whose switching cost is at most $S$. It is simple to
    see that the following strategy is optimal: in the first round of each block, randomly assign the weights since there is no information about the losses of the experts; then convert the weight on the bad experts (with loss $2C$) to the good experts (with loss $0$) if the current switching cost is still less than $S$. Let the random variable $W$ denote the total weights that the algorithm assigns to the bad experts in the blocks' first iteration. Clearly $\E[W] = B/2$. Then the random variable $\min\{W, S/2\}$ is equal to the weights that algorithm $\A$ can convert from bad experts to good expert ($S/2$ dues to that converting weight of $S/2$ will suffers $S$ switching cost). Thus, we have
    \begin{align}
      & \mathbb{E}[\mbox{cumulative loss of } \A] = 2C \cdot \mathbb{E}[\A\mbox{'s weights on bad experts}] \notag\\
      = {} & 2C \cdot \mathbb{E}\left[\min\left\{W, \frac{S}{2}\right\} + \frac{T}{B} \cdot \sbr{W-\min\left\{W, \frac{S}{2}\right\}}\right] \notag\\
      \ge {} & 2C \cdot \frac{T}{B} \cdot \mathbb{E} \mbr{W-\frac{S}{2}} \ge 2C \cdot \frac{T}{B} \left(\frac{B}{2} - 2S\right)\notag\\
      = {} & 2C \cdot \sbr{\frac{T}{2}-\frac{2ST}{B}} = 2C \cdot \sbr{\frac{T}{2}-\frac{a^2 T \log N}{2S}} \label{eq:lower-alg}.
    \end{align}
    Combining \eqref{eq:lower-best} and \eqref{eq:lower-alg}, we conclude that any algorithm for $\lambda$-switching cost and $S$-switching cost budget suffers an expected regret at least $a^2 C T \log N / S = \Omega(CT/S)$. As a result, the regret of \eqref{eq:shifting-PEA} is at least $\Omega(CT/S + \lambda S) = \Omega(\sqrt{\lambda C T})$, which finishes the proof.
\end{proof}

\subsection{Proof of Theorem~\ref{thm:optimal-memory}}
\label{sec:appendix-proof-optimal-memory}
\begin{proof} 
We begin the proof by decomposing the dynamic regret of OCO with switching cost, and will then prove the theorem by exploiting the property of Scream algorithm.

\paragraph{Regret Decomposition.}
  We divide the time horizon $T$ into $K$ epochs of equal length $\Delta$, where the $k$-th epoch is denoted by $\I_k \define \{t_{k,1},\ldots, t_{k,\Delta}\}$ ($\Delta,K$ to be specified later). Without loss of generality, we assume $T=K \cdot \Delta$. Since in Algorithm~\ref{alg:lazy-scream}, the meta-learner and base-learners do \emph{not} update within each epoch, we denote by $\wcirc_1,\ldots,\wcirc_K$ the decisions of $K$ epochs. Thus the dynamic regret of OCO with switching cost can be decomposed as
   \begin{align*}
      & \sum_{t=1}^T \f_t(\w_t) - \sum_{t=1}^T \f_t(\v_t) + \lambda \sum_{t=2}^T \norm{\w_{t} - \w_{t-1}}_2 \le \sum_{t=1}^T \inner{\nabla_t}{\w_t - \v_t} + \lambda \sum_{t=2}^T \norm{\w_t - \w_{t-1}}_2\\
       = {} & \sum_{k=1}^K \sum_{t \in \I_k} \inner{\nabla_t}{\wcirc_k - \v_t} + \lambda \sum_{k=2}^K \norm{\wcirc_k - \wcirc_{k-1}}_2\\
       = {} & \underbrace{\sum_{k=1}^K \left\langle \sum_{t \in \I_k} \nabla_t, \wcirc_k - \vcirc_k \right\rangle + \lambda \sum_{k=2}^K \norm{\wcirc_k - \wcirc_{k-1}}_2}_{\term{A}} + \underbrace{\sum_{k=1}^K \sum_{t \in \I_k} \inner{\nabla_t}{\vcirc_k - \v_t}}_{\term{B}},
   \end{align*}
   where $\nabla_t \define \nabla \f_t(\w_t)$ and the $k$-th comparator $\vcirc_k \define \v_{t_{k,1}}$ is chosen as the first one in the $k$-th epoch. Define $\g_k \define \sum_{t \in \I_k} \nabla_t$ the loss of the $k$-epoch. Intuitively, term~(A) is the dynamic regret of OCO with switching cost in $K$ rounds with the loss sequence $\g_{1:K}$ and comparator sequence $\vcirc_{1:K}$. Since the new comparator sequence is artificially constructed, we need to measure its difference from the original sequence $\v_{1:T}$, i.e., term~(B). Term~(B) can be simply bounded using the sub-additivity property of vector norms, formally,
    \begin{align*}
        \term{B} \le {} & G \sum_{k=1}^K \sum_{t \in \I_k} \norm{\vcirc_k - \v_t}_2 \le G \sum_{k=1}^K |\I_k| \sum_{t \in \I_k} \norm{\v_{t} - \v_{t-1}}_2 \le G\Delta P_T.
    \end{align*}
    
    \paragraph{Black-box Use of Scream.}
    Term~(A) is actually the dynamic regret of OCO with switching cost in $K$ rounds. Plugging in the regret bound of Scream~\eqref{eq:Scream}, it holds that
    \begin{align*}
        & \term{A} = \sum_{k=1}^K \inner{\g_k}{\wcirc_k - \vcirc_k} + \lambda \sum_{k=2}^K \norm{\wcirc_k - \wcirc_{k-1}}_2\\ 
        \lesssim {} & \sqrt{2 (G^\prime D+\sqrt{\lambda}) (\lambda D + G^\prime D+\sqrt{\lambda})K}\left( 1 + \ln(i^*_K + 1) \right) + \sqrt{\sbr{{G^\prime}^2 + \lambda G^\prime}(D^2+2DP_K)K}\\
        \lesssim {} & \sqrt{2 (\Delta G D+\sqrt{\lambda}) (\lambda D +\Delta G D)\frac{T}{\Delta}}\left( 1 + \ln(i^*_K +1) \right) + \sqrt{\sbr{\Delta^2 G^2 + \lambda \Delta G}(D^2+2DP_K)\frac{T}{\Delta}}\\
        = {} & \sqrt{2D (\Delta G D+\sqrt{\lambda}) \sbr{\frac{\lambda}{\Delta}+G} T} \left( 1 + \ln(i^*_K+1) \right) + \sqrt{\sbr{\Delta G^2 + \lambda G}(D^2+2DP_K)T}\\
        \le {} & \O(\sqrt{\lambda T(1+P_T)}),
    \end{align*}
    where the path length in $K$ epochs $P_K \define \sum_{k=2}^K \norm{\vcirc_k - \vcirc_{k-1}}_2 \le P_T$, the gradient upper bound $G^\prime = \max_{k \in [K]} \norm{\g_k}_2 \le \Delta G$ and $a \lesssim b$ means $a = \O(b)$. The last step is due to the property of the best base learner, that is,
    \begin{equation*}
      i^*_K \overset{\eqref{eq:optimal-index}}{\le} \Big \lceil \frac{1}{2} \log_2 \big(1 + \frac{2P_K}{D}\big) \Big \rceil + 1 \le \Big \lceil \frac{1}{2} \log_2 \big(1 + \frac{2P_T}{D}\big) \Big \rceil + 1,
    \end{equation*}
    and by choosing $\Delta = \sqrt{\lambda}$. Combining the above inequality with the upper bound of $\term{B} \le G\sqrt{\lambda} P_T = \O(\sqrt{\lambda T (1+P_T)})$ finishes the proof.
\end{proof}

\subsection{Proof of Theorem~\ref{thm:lower-dynamic}}
\label{sec:appendix-proof-lower-bound}
\begin{proof}
  Overall the proof consists of two parts. First, we propose a lower bound for static regret of OCO with switching cost. Second, building upon the static regret lower bound, we give a lower bound for dynamic regret of OCO with switching cost to complete the proof.
  \paragraph{Static Regret Lower Bound.} 
  To give a static regret lower bound, we first consider a $T$-round prediction with expert advice problem with $\lambda$-switching cost. Theorem~\ref{thm:PEA-SC-lower} shows that given $\lambda >0$ and $C > 0$, there exists a sequence of loss functions $\ellb_1,\ldots,\ellb_T$ satisfying $\ellb_t \in [-C,C]^N$ for all $t \in [T]$ such that any feasible expert algorithm (whose output is $\p_1,\ldots,\p_T \in \Delta_N$) incurs the following regret 
  \begin{equation}
    \label{eq:lower-static-PEA}
    \sum_{t=1}^T \langle \ellb_t, \p_t\rangle - \min_{i \in [N]} \sum_{t=1}^T \ell_{t,i}+ \lambda \sum_{t=2}^T \| \p_t - \p_{t-1} \|_1 \ge \Omega(\sqrt{\lambda CT}).
  \end{equation}
  Consequently, given a parameter $\lambda > 0$, we choose the feasible domain as $\W = C_1 \Delta_N$, where $C_1 = \min\{1, D/\sqrt{2}\}$. It is easy to observe that $\W$ satisfies Assumption~\ref{assume:bounded-norm}, because for any $\p_1, \p_2 \in \Delta_N$, $\norm{C_1\p_1 - C_1\p_2}_2 \le C_1 \cdot \sqrt{2} \le D$ holds. Choose $C_2 = G/\sqrt{N}$ and loss functions as $h_t(\w) = \inner{\ellb_t}{\w}$, where $\ellb_{1:T}$ is the loss sequence that makes \eqref{eq:lower-static-PEA} holds given $C_2$ and $\lambda$. Since for any $\w \in \W$, $\norm{\nabla h_t(\w)}_2 = \norm{\ellb_t}_2 \le C_2\sqrt{N} \leq G$, the loss functions $h_1,\ldots,h_T$ satisfy Assumption~\ref{assume:bounded-gradient}. Thus any online algorithm returning $\w^\prime_1 \define C_1\w_1,\ldots,\w^\prime_T \define C_1\w_T \in \W$ satisfies 
  \begin{align}
    & \sum_{t=1}^T h_t(\w^\prime_t) -  \min_{\v \in \W} \sum_{t=1}^T h_t(\v) + \lambda \sum_{t=2}^T \norm{\w^\prime_{t} - \w^\prime_{t-1}}_2 \notag\\
    = {} & C_1 \sbr{\sum_{t=1}^T \inner{\ellb_t}{\w_t} - \min_{\v \in \Delta_N} \sum_{t=1}^T \inner{\ellb_t}{\v} + \lambda \sum_{t=2}^T \| \w_t - \w_{t-1} \|_2} \notag\\
    \ge {} & C_1 \sbr{\sum_{t=1}^T \langle \ellb_t, \w_t\rangle - \min_{i \in [N]} \sum_{t=1}^T \ell_{t,i} + \frac{\lambda}{\sqrt{d}} \sum_{t=2}^T \| \w_t - \w_{t-1} \|_1} \notag\\
    \overset{\eqref{eq:lower-static-PEA}}{\ge} {} & C_1 \Omega(\sqrt{\lambda C_2 T}) = \Omega(\sqrt{\lambda T})\label{eq:lower-static},
  \end{align}
  where the first step is by plugging in the definition of $h_1,\ldots,h_T$ and $\w^\prime_1,\ldots,\w^\prime_T$, the second step is because the optimizer in a simplex is on one of its vertices and the relationship between $\ell_1$-norm and $\ell_2$-norm, formally, $\norm{\x-\y}_1 \le \sqrt{d} \cdot \norm{\x-\y}_2$, for any $\x,\y \in \R^d$, where $d$ denotes the dimension.

  \paragraph{Dynamic Regret Lower Bound.}
  We consider two cases according to the value of $\tau$. When $\tau \le D$, we can always find a comparator sequence $\v_1,\ldots,\v_T \in \W$ such that
  \begin{align*}
    & \sum_{t=1}^T h_t(\w_t) - \sum_{t=1}^T h_t(\v_t) + \lambda \sum_{t=2}^T \norm{\w_{t} - \w_{t-1}}_2\\
    \ge {} & \sum_{t=1}^T h_t(\w_t) - \min_{\v \in \W} \sum_{t=1}^T h_t(\v) + \lambda \sum_{t=2}^T \norm{\w_{t} - \w_{t-1}}_2 \overset{\eqref{eq:lower-static}}{\ge} \Omega(\sqrt{\lambda T}) = \Omega(\sqrt{\lambda \tau T}),
  \end{align*} 
  where the last step holds since $\tau \le D$ can be seen as a constant and thus will not affect the order. Next, we consider the case $\tau \in (D, DT]$. Without loss of generality, we assume $\ceil{\tau}$ divides $T$ and let $K = T / \ceil{\tau}$. To proceed, we construct the following piecewise-stationary comparator sequence $\v_1,\ldots,\v_T$: for any $i \in [\ceil{\tau}]$, denote by $\I_i = [(i-1)K+1, iK]$ the $i$-th interval, the comparators within the interval are set as
  \begin{equation*}
    \v_{(i-1)K+1} = \v_{(i-1)K+2} = \dots = \v_{iK} \in \argmin_{\v \in \W} \sum_{t \in \I_i} h_t(\v).
  \end{equation*}
  Note that the path length of this comparator sequence does not exceeds $\tau D$. Thus, the dynamic regret competing with the comparator sequence $\v_1,\ldots,\v_T$ can be evaluated as,
  \begin{align*}
    & \sum_{t=1}^T h_t(\w_t) - \sum_{t=1}^T h_t(\v_t) + \lambda \sum_{t=2}^T \norm{\w_{t} - \w_{t-1}}_2\\
    \ge {} & \sum_{i=1}^{\ceil{\tau}} \sbr{\sum_{t \in \I_i} h_t(\w_t) - \min_{\v \in \W} \sum_{t \in \I_i} h_t(\v) + \lambda \sum_{t = (i-1)K+2}^{iK} \norm{\w_{t} - \w_{t-1}}_2}\\
    \overset{\eqref{eq:lower-static}}{\gtrsim} {} & \sum_{i=1}^{\ceil{\tau}} \sqrt{\lambda |\I_i|} = \ceil{\tau} \sqrt{\lambda \cdot \frac{T}{\ceil{\tau}}} \ge \sqrt{\lambda \tau T},
  \end{align*}
  where the first inequality is true by ignoring the switching cost between two consecutive pieces. In addition, $a \gtrsim b$ means $a = \Omega(b)$. Hence, we complete the proof.
\end{proof}

\subsection{Discussion on Memory Dependence}
\label{sec:appendix-memory-dependency}
In this part, we examine a subtle issue: the memory dependence of our static policy regret bound (an implication of the dynamic policy regret bound in Theorem~\ref{thm:main-result-OCOmemory}) and that of existing work~\citep{NIPS'15:OCOmemory}.

First, we state our attained static policy regret for OCO with memory via performing OGD over the unary loss with an optimal step size tuning (which is feasible as there is no dependence on the path length $P_T$).
\begin{myThm}
\label{thm:static-regret}
Under Assumptions~\ref{assume:Lipschitz}--\ref{assume:bounded-norm}, running OGD over the unary loss achieves $$\sum_{t=1}^T f_t(\w_{t-m:t}) - \min_{\v \in \W} \sum_{t=1}^T \f_t(\v) \leq (G^2 + m^2LG) \eta T + \frac{2D^2}{\eta}.$$ Setting the step size optimally as $\eta = \eta^* = \sqrt{\frac{2D^2}{(G^2 + m^2LG) T}}$, we attain an $\O(m\sqrt{T})$ static policy regret. 
\end{myThm}

\citet{NIPS'15:OCOmemory} present an $\O(m^{3/4} \sqrt{T})$ static policy regret for OCO with memory, which seems better than ours at the first glance. However, we point it out that this is due to the different assumptions imposing over the Lipschitz continuity. Their assumption is presented as follows.
\begin{myAssumption}[{Lipschitzness of~\citet{NIPS'15:OCOmemory}}]
\label{assume:Lipschitz-2}
The function $f_t: \W^{m+1} \mapsto \R$ is $\bar{L}$-Lipschitz, i.e.,
\[
  \abs{f_t(\x_0,\ldots,\x_m) - f_t(\y_0,\ldots,\y_m)} \leq \bar{L} \norm{(\x_0,\ldots,\x_m) - (\y_0,\ldots,\y_m)}_2 = \bar{L} \sqrt{\sum_{i=0}^{m} \norm{\x_i - \y_i}^2_2}.
\]
\end{myAssumption}
We compare this definition of Lipschitzness with the version used in our paper, namely, the coordinate-wise Lipschitzness defined in Assumption~\ref{assume:Lipschitz}. Indeed, their definition imposes a \emph{stronger} requirement on the function than ours. Clearly, when the online function $f_t$ satisfies $\bar{L}$-Lipschitz assumption as specified in Assumption~\ref{assume:Lipschitz-2}, it is also $\bar{L}$-coordinate-wise Lipschitz due to the simple fact that $\sqrt{\sum_{i=0}^{m} \norm{\x_i - \y_i}^2_2} \leq \sum_{i=0}^{m} \norm{\x_i - \y_i}_2$. On the other hand, when the online function $f_t$ is $L$-coordinate-wise Lipschitz as required by Assumption~\ref{assume:Lipschitz}, we thus conclude that it is Lipschitz in the sense of Assumption~\ref{assume:Lipschitz-2} with the Lipschitz coefficient $\bar{L} = \sqrt{m} L$, due to the following inequality (by Cauchy-Schwarz inequality) $L \sum_{i=0}^{m} \norm{\x_i - \y_i}_2 \leq L \sqrt{m} \sqrt{\sum_{i=0}^{m} \norm{\x_i - \y_i}_2}$.

In the following, we restate the static regret bound of~\citet{NIPS'15:OCOmemory} under Assumption~\ref{assume:Lipschitz-2}. We adapt their results to our notations to ease the understanding.

\begin{myThm}[{Theorem 3.1 of~\citet{NIPS'15:OCOmemory}}]
Under Assumptions~\ref{assume:bounded-gradient},~\ref{assume:bounded-norm}, and the assumption that the online functions are $\bar{L}$-Lipschitz (Assumption~\ref{assume:Lipschitz-2}), running OGD over the unary loss achieves 
\begin{equation}
  \label{eq:new-static-bound}
  \sum_{t=1}^T f_t(\w_{t-m:t}) - \min_{\v \in \W} \sum_{t=1}^T \f_t(\v) \leq 2\eta G^2 T + \frac{2D^2}{\eta} + 2\bar{L} m^{\frac{3}{2}} \eta G T.
\end{equation}
Setting the step size optimally yields an $\O(\bar{L}^{1/2} m^{3/4}\sqrt{T})$ static policy regret. 
\end{myThm}

Therefore, when the online functions are only $L$-coordinate-wise Lipschitz as considered in this paper, applying above theorem immediately obtains an $\O(\bar{L}^{1/2} m^{3/4}\sqrt{T}) = \O((\sqrt{m} L)^{1/2} m^{3/4}\sqrt{T}) = \O( L^{1/2} m\sqrt{T})$, which exhibiting a linear memory dependence. 

\section{Omitted Details for Section~\ref{sec:online-control} (Non-stochastic Control)}
\label{sec:appendix-control}
In this section, we present omitted details for Section~\ref{sec:online-control} online non-stochastic control, including the proofs of Proposition~\ref{proposition:DAC-state}, Theorem~\ref{thm:main-result-control}, Theorem~\ref{thm:unknown-system}, and Corollary~\ref{corollary:static-regret-control}.

\subsection{Proof of Proposition~\ref{proposition:DAC-state}}
We will prove the following statement that gives the state recurrence for any $h \leq t$, which is essentially a strengthened result of Proposition~\ref{proposition:DAC-state}.
\begin{myProp}
\label{proposition:DAC-state-strengthen}
Suppose one chooses the DAC controller $\pi(M_t,K)$ at iteration $t$, the reaching state is 
\begin{equation}
  \label{eq:state-DAC}
  x_{t+1} = \tilde{A}_K^{h+1} x_{t-h} + \sum_{i=0}^{H + h} \Psi_{t,i}^{K,h}(M_{t-h:t})  w_{t-i},
\end{equation}
where $\tilde{A}_K = A - BK$, and $\Psi_{t,i}^{K,h}(M_{t-h:t})$ is the transfer matrix defined as 
\begin{equation}
  \label{eq:transfer-matrix}
  \Psi_{t,i}^{K,h}(M_{t-h:t}) = \tilde{A}_K^i \ind{i \leq h} + \sum_{j=0}^{h} \tilde{A}_K^j B M_{t-j}^{[i-j]} \ind{1 \leq i-j \leq H}.
\end{equation}
The evolving equation holds for any $h \in \{0,\dots,t\}$.
\end{myProp}

\begin{proof}
First, by substituting the DAC policy into the dynamics equation, we have 
\begin{align*}
  x_{t+1} = {} & A x_t + B u_t + w_t = (A - BK) x_t + \sum_{i=1}^{H} B M_t^{[i]} w_{t-i} + w_t \\
  = {} & \tilde{A}_K^{h+1} x_{t-h} + \sum_{j=0}^{h} \tilde{A}_K^j \left(\sum_{i=1}^{H} B M_{t-j}^{[i]} w_{t-j-i} + w_{t-j}\right) \\
  = {} & \tilde{A}_K^{h+1} x_{t-h} + \sum_{j=0}^{h} \sum_{i=1}^{H} \tilde{A}_K^j B M_{t-j}^{[i]} w_{t-j-i} + \sum_{j=0}^{h} \tilde{A}_K^j w_{t-j}.
\end{align*}
Exchanging the summation index yields,
\begin{align}
\sum_{j=0}^{h} \sum_{i=1}^{H} \tilde{A}_K^j B M_{t-j}^{[i]} w_{t-j-i} = & {} \sum_{i=1}^{H} \sum_{k=i}^{i+h} \tilde{A}_K^{k-i} B M_{t-k+i}^{[i]} w_{t-k} \label{eq:transfer-matrix-1}\\
= {} & \sum_{k=1}^{H+h} \sum_{i=k-h}^{k} \tilde{A}_K^{k-i} B M_{t-k+i}^{[i]}w_{t-k} \ind{1\leq i \leq H} \label{eq:transfer-matrix-2}\\
= {} & \sum_{k=1}^{H+h} \sum_{l=0}^{h} \tilde{A}_K^{h-l} B M_{t+l-h}^{[l+k-h]}w_{t-k} \ind{1\leq l+(k-h) \leq H}\label{eq:transfer-matrix-3}\\
= {} & \sum_{k=1}^{H+h} \sum_{m=0}^{h} \tilde{A}_K^{m} B M_{t-m}^{[k-m]} w_{t-k} \ind{1 \leq k-m \leq H} \label{eq:transfer-matrix-4}\\
= {} & \sum_{i=1}^{H+h} \sum_{j=0}^{h} \tilde{A}_K^{j} B M_{t-j}^{[i-j]} w_{t-i} \ind{1 \leq i-j \leq H} \label{eq:transfer-matrix-5},
\end{align}
where \eqref{eq:transfer-matrix-1} holds by defining a third variable $k=j+i$, and \eqref{eq:transfer-matrix-2} is obtained by exchanging the summation index $i$ and $k$ and the new range of $i$ is from inequality $i \le k \le i+h$. Moreover,~\eqref{eq:transfer-matrix-3} is obtained by another change of variable $l=i-k+h$,~\eqref{eq:transfer-matrix-4} is obtained by replacing $l$ by $h-m$, and~\eqref{eq:transfer-matrix-5} is true by setting $i=k,j=m$. Therefore, we obtain that
\begin{align*}
  x_{t+1} = {} & \tilde{A}_K^{h+1} x_{t-h} + \sum_{j=0}^{h} \sum_{i=1}^{H} \tilde{A}_K^j B M_{t-j}^{[i]} w_{t-j-i} + \sum_{j=0}^{h} \tilde{A}_K^j w_{t-j}\\
  = {} & \tilde{A}_K^{h+1} x_{t-h} + \sum_{i=0}^{H+h} \sum_{j=0}^{h} \tilde{A}_K^{j} B M_{t-j}^{[i-j]} w_{t-i} \ind{1 \leq i-j \leq H} + \sum_{i=0}^{h} \tilde{A}_K^i w_{t-i}\\
  = {} & \tilde{A}_K^{h+1} x_{t-h} + \sum_{i=0}^{H+h} \left(\tilde{A}_K^i \ind{i \leq h} + \sum_{j=0}^{h} \tilde{A}_K^{j} B M_{t-j}^{[i-j]} \ind{1 \leq i-j \leq H} \right)w_{t-i}  
\end{align*}
and hence complete the proof.
\end{proof}

\subsection{Proof of Theorem~\ref{thm:main-result-control}}

To prove the dynamic policy regret of online non-stochastic control (Theorem~\ref{thm:main-result-control}), we will first present  theoretical analysis of the reduction to OCO with memory in Appendix~\ref{sec:appendix-proof-reduction}, then give the dynamic regret analysis over the $\M$-space in Appendix~\ref{sec:appendix-proof-Reg}, and finally present the overall proof of Theorem~\ref{thm:main-result-control} in Appendix~\ref{sec:appendix-proof-final}.

\subsubsection{Approximation Error}
\label{sec:appendix-proof-reduction}
In Section~\ref{sec:control-reduction} of the main paper, we have presented how to reduce from online non-stochastic control to OCO with memory, by employing the DAC parameterization and introducing the truncated loss functions. In this part, we introduce the following theorem that discloses that the truncation loss $f_t$ approximates the original cost function $c_t$ well.

\begin{myThm}[{Theorem 5.3 of~\citet{ICML'19:online-control}}]
\label{thm:approximation}
Suppose the disturbance are bounded by $W$. For any $(\kappa,\gamma)$-strongly stable linear controller $K$, and any $\tau > 0$ such that the sequence of $M_1,\ldots,M_T$ satisfies $\opnorm{M_t^{[i]}} \leq \tau (1-\gamma)^i, \forall i \in [H]$, the approximation error between original loss and truncated loss is at most
\begin{equation}
  \label{eq:approximation}
  \left\vert \sum_{t=1}^{T} c_t(x_t^K(M_{0:t-1}), u_t^K(M_{0:t})) - \sum_{t=1}^{T} f_t(M_{t-1-H:t}) \right\vert \leq  2 T G_c D^2 \kappa^3 (1-\gamma)^{H+1},
\end{equation}
where 
\begin{equation}
  \label{eq:def-D}
  D \define  \frac{W\kappa^3(1+ H \kappa_B \tau)}{\gamma (1-\kappa^2(1-\gamma)^{H+1})} + \frac{W\tau}{\gamma}.
\end{equation}
\end{myThm}

\begin{proof}
By Lipschitzness and definition of the truncated loss, we get that
\begin{align*}
    {} & c_t(x_t^K(M_{0:t-1}), u_t^K(M_{0:t}))-f_t(M_{t-H-1:t})\\
  = {} & c_t(x_t^K(M_{0:t-1}), u_t^K(M_{0:t}))-c_t(y_t^K(M_{t-H-1:t-1}),v_t^K(M_{t-H-1:t}))\\
  \leq {} & G_c D\left(\norm{x_t^K(M_{0:t-1}) - y_t^K(M_{t-H-1:t-1})} + \norm{u_t^K(M_{0:t}) - v_t^K(M_{t-H-1:t})} \right) \\
  \leq {} & G_c D (\kappa^2 (1-\gamma)^{H+1} D + \kappa^3 (1-\gamma)^{H+1} D) \le 2G_cD^2 \kappa^3 (1-\gamma)^{H+1},
\end{align*}
where the last two inequalities use the Lipschitzness and the boundedness presented in Lemma~\ref{lemma:support-1-boundedness}. We complete the proof by summing over the iterations from $t=1,\ldots,T$.
\end{proof}

\subsubsection{Dynamic Regret Analysis over $\M$-space}
\label{sec:appendix-proof-Reg}
In previous sections, we have analyzed the dynamic regret of our method over the $\R^d$-space. However, after reducing online non-stochastic control to OCO with memory, we need to apply their results to the $\M$-space and thus require to generalize the arguments of previous sections from Euclidean norm for $\R^d$-space to Frobenius norm for $\M$-space. For completeness, we present the proof here.

At the first place, we analyze the dynamic regret of the online gradient descent (OGD)  algorithm over the $\R^d$-space. OGD begins with any $M_1 \in \mathcal{M}$ and performs the following update procedure,
\begin{equation}
  \label{eq:OGD-update}
  M_{t+1} = \Pi_{\mathcal{M}} [M_t - \eta \nabla_M \tilde{f}_t(M_t)]
\end{equation}
where $\eta>0$ is the step size and $\Pi_{\mathcal{M}}[\cdot]$ denotes the projection onto the nearest point in the feasible set $\M$. We have the following dynamic regret regarding its dynamic regret.

\begin{myThm}
\label{thm:OGD-dynamic-regret-Control}
Suppose the function $\f: \M \mapsto \R$ is convex, the gradient norm satisfies $\max_{M \in \M} \max_{t \in [T]}\Fnorm{\nabla_M \f_t(M)} \leq G_f$ and the Euclidean diameter of $\M$ is at most $D_f$, i.e., $\sup_{M,M' \in \M} \Fnorm{M - M'} \leq D_f$. Then, OGD with a step size $\eta >0$ as shown in~\eqref{eq:OGD-update} satisfies that
\begin{equation}
  \label{eq:OGD-dynamic-regret-Control}
  \lambda \sum_{t=2}^{T} \Fnorm{M_{t-1} - M_t} + \sum_{t=1}^T \tilde{f}_t(M_t) - \sum_{t=1}^{T} \tilde{f}_t(M^*_t) \leq \frac{\eta}{2} (G_f^2 + 2\lambda G_f) T + \frac{1}{2\eta}(D_f^2 + 2D_f P_T),
\end{equation}
which holds for any comparator sequence $M_1^*,\ldots,M_T^* \in \M$. Besides, the path length $P_T = \sum_{t=2}^{T} \Fnorm{M^*_{t-1} - M^*_t}$ measures the non-stationarity of the comparator sequence.
\end{myThm}

\begin{proof}
Denote the gradient by $G_t = \nabla_M \tilde{f}_t(M_t)$. The convexity of online surrogate loss functions implies that 
\begin{equation*}
  \sum_{t=1}^T \tilde{f}_t(M_t) - \sum_{t=1}^{T} \tilde{f}_t(M^*_t) \leq \sum_{t=1}^{T} \inner{G_t}{M_t - M^*_t}.
\end{equation*}

Thus, it suffices to bound the sum of $\inner{G_t}{M_t - M^*_t}$. From the OGD update rule and the non-expensive property, we have 
\begin{align*}
    \Fnorm{M_{t+1} - M^*_t}^2 = {} & \left\lVert\Pi_{\M}[M_t - \eta G_t] - M^*_t\right\rVert_{\mathrm{F}}^2 \leq \Fnorm{M_t - \eta G_t - M^*_t}^2\\
    = {} & \eta^2 \Fnorm{G_t}^2 - 2\eta \inner{G_t}{M_t - M^*_t} + \Fnorm{M_t - M^*_t}^2 
\end{align*}
After rearranging, we obtain
\[
  \inner{G_t}{M_t - M^*_t} \leq \frac{\eta}{2}\Fnorm{G_t}^2 + \frac{1}{2\eta}\left(\Fnorm{M_t - M^*_t}^2 - \Fnorm{M_{t+1} - M^*_t}^2 \right).
\]

Next, we turn to analyze the second term on the right-hand side. Indeed,
\begin{align*}
  & \sum_{t=1}^{T}\left(\Fnorm{M_t - M^*_t}^2 - \Fnorm{M_{t+1} - M^*_t}^2 \right) \leq \sum_{t=1}^{T} \Fnorm{M_t - M^*_t}^2 - \sum_{t=2}^{T} \Fnorm{M_t - M^*_{t-1}}^2\\
  \leq {} & \Fnorm{M_1 - M^*_1}^2 + \sum_{t=2}^{T} \left(\Fnorm{M_t - M^*_{t}}^2 - \Fnorm{M_t - M^*_{t-1}}^2\right)\\
  = {} & \Fnorm{M_1 - M^*_1}^2 + \sum_{t=2}^{T} \inner{M^*_{t-1} - M^*_{t}}{2M_t - M^*_{t-1} - M^*_{t}} \leq D_f^2 + 2D_f \sum_{t=2}^{T} \Fnorm{M^*_{t-1} - M^*_{t}}.
\end{align*}
Hence, combining all above inequalities, we have 
\begin{align*}
    \sum_{t=1}^T \f_t(M_t) - \sum_{t=1}^{T} \f_t(M^*_t) \leq {} & \frac{\eta}{2}\sum_{t=1}^{T}\Fnorm{G_t}^2 +  \frac{1}{2\eta}\left(D_f^2 + 2D_f \sum_{t=2}^{T} \Fnorm{M^*_{t-1} - M^*_{t}}\right) \\
    \leq {} & \frac{\eta}{2} G_f^2 T + \frac{1}{2\eta}(D_f^2 + 2D_f P_T).
\end{align*}
On the other hand, the switching cost can be bounded by 
\begin{align*}
    \Fnorm{M_{t} - M_{t-1}} = \left\lVert\Pi_{\M}[M_{t-1} - \eta G_{t-1}] - M_{t-1}\right\rVert_{\mathrm{F}}^2 \leq \Fnorm{M_{t-1} - \eta G_{t-1} - M_{t-1}} \leq \eta G_f,
\end{align*}
which together with the previous dynamic regret bound yields the desired result.
\end{proof}

\subsubsection{Proof of Theorem~\ref{thm:main-result-control}}
\label{sec:appendix-proof-final}

\begin{proof}
We begin with the following dynamic policy regret decomposition,
\begin{align}
  & \sum_{t=1}^{T} c_t(x_t,u_t) - \sum_{t=1}^{T} c_t(x_t^{\pi_t},u_t^{\pi_t}) \notag \\
  = {} & \sum_{t=1}^{T} c_t(x_t^K(M_{0:t-1}), u_t^K(M_{0:t})) - \sum_{t=1}^{T} c_t(x_t^K(M^*_{0:t-1}), u_t^K(M^*_{0:t})) \notag\\
  = {} & \underbrace{\sum_{t=1}^{T} c_t(x_t^K(M_{0:t-1}), u_t^K(M_{0:t})) - \sum_{t=1}^{T} f_t(M_{t-1-H:t})}_{\define A_T} + \underbrace{\sum_{t=1}^{T} f_t(M_{t-1-H:t}) - \sum_{t=1}^{T} f_t(M^*_{t-1-H:t})}_{\define B_T} \notag\\
  {} & \qquad + \underbrace{\sum_{t=1}^{T} f_t(M^*_{t-1-H:t}) - \sum_{t=1}^{T} c_t(x_t^K(M^*_{0:t-1}), u_t^K(M^*_{0:t}))}_{\define C_T}. \label{eq:control-regret-decompose}
\end{align}
Notice that both $A_T$ and $C_T$ essentially represent the approximation error introduced by the truncated loss, so we can apply Theorem~\ref{thm:approximation} and obtain
\begin{align}
  \label{eq:control-regret-A-C}
  A_T + C_T \leq 4 T G_c D^2 \kappa^3 (1-\gamma)^{H+1}.
\end{align}
We now focus on the quantity $B_T$, which is the dynamic policy regret over the truncated loss functions $\{f_t\}_{t=1,\ldots,T}$. Indeed,
\begin{align}
  B_T = {} & \sum_{t=1}^{T} f_t(M_{t-1-H:t}) - \sum_{t=1}^{T} f_t(M^*_{t-1-H:t}) \nonumber \\
  \leq {} & \sum_{t=1}^{T} \f_t(M_{t}) - \sum_{t=1}^{T} \f_t(M^*_{t}) + \lambda \sum_{t=2}^{T} \Fnorm{M_{t-1} - M_t} + \lambda \sum_{t=2}^{T} \Fnorm{M^*_{t-1} - M^*_t}\nonumber \\
  \leq {} & \sum_{t=1}^{T} \inner{\nabla_M \f_t(M_{t})}{M_t - M^*_{t}} + \lambda \sum_{t=2}^{T} \Fnorm{M_{t-1} - M_t} + \lambda \sum_{t=2}^{T} \Fnorm{M^*_{t-1} - M^*_t} \nonumber \\
  = {} & \sum_{t=1}^{T} g_t(M_{t}) - \sum_{t=1}^{T} g_t(M^*_{t}) + \lambda \sum_{t=2}^{T} \Fnorm{M_{t-1} - M_t} + \lambda \sum_{t=2}^{T} \Fnorm{M^*_{t-1} - M^*_t}, \label{eq:control-regret-B}
\end{align}
where $\lambda = (H+2)^2L_f$ and $g_t(M) = \inner{\nabla_M \f_t(M_{t})}{M}$ is the surrogate linearized loss. As a consequence, we are reduced to proving an dynamic regret over the sequence of functions $\{g_t\}_{t=1,\ldots,T}$ with switching cost, namely, the first three terms on the right-hand side. We thus make use of the techniques developed in Appendix~\ref{sec:appendix-proof-OCOmemory} (dynamic policy regret minimization for OCO with memory) to decompose the terms into meta-regret and base-regret:
\begin{align*}
    {} & \sum_{t=1}^{T} g_t(M_{t}) - \sum_{t=1}^{T} g_t(M^*_{t}) + \lambda \sum_{t=2}^{T} \Fnorm{M_{t-1} - M_t}\\
  = {} & \underbrace{\left(\lambda \sum_{t=2}^{T} \Fnorm{M_{t-1} - M_t} + \sum_{t=1}^{T} g_t(M_t)\right) - \left(\lambda \sum_{t=2}^{T} \Fnorm{M_{t-1,i} - M_{t,i}} + \sum_{t=1}^{T} g_t(M_{t,i})\right)}_{\meta} \\
  {} & \qquad \qquad \qquad \qquad + \underbrace{\left(\lambda \sum_{t=2}^{T} \Fnorm{M_{t-1,i} - M_{t,i}} + \sum_{t=1}^{T} g_t(M_{t,i}) - \sum_{t=1}^{T} g_t(M^*_{t})\right)}_{\base}.
\end{align*}
We remark that the regret decomposition holds for any base-learner index $i \in [N]$. We now provide the upper bounds for the meta-regret and base-regret, respectively. First, Theorem~\ref{thm:OGD-dynamic-regret-Control} ensures the base-regret satisfies that 
\begin{align*}
\base \leq \frac{\eta_i}{2} (G_f^2 + 2\lambda G_f) T + \frac{1}{2\eta_i}(D_f^2 + 2D_f P_T),
\end{align*}
where $P_T = \sum_{t=2}^{T} \Fnorm{M^*_{t-1} - M^*_t}$ is the path length of the comparator sequence. On the other hand, similar to Lemma~\ref{lemma:sc-decompose} of Section~\ref{sec:sc-decompose}, we can show that the meta-regret satisfies
\begin{align*}
\meta \leq \lambda' \sum_{t=2}^{T} \norm{\p_{t-1} - \p_t}_1 + \sum_{t=1}^{T} \inner{\p_t}{\ellb_t} - \sum_{t=1}^T \ell_{t,i},
\end{align*}
where the surrogate loss vector $\ellb_t \in \Delta_N$ of the meta-algorithm is defined as 
\[
  \ell_{t,i} = \lambda \Fnorm{M_{t-1,i} - M_{t,i}} + g_t(M_{t,i}), \mbox{ for } i \in [N].
\]

Then, we can use the static regret with switching cost of online mirror descent for the prediction with expert advice setting (c.f. Corollary~\ref{corollary:Hedge} in Appendix~\ref{sec:appendix-proof-OMD}) and obtain that
\begin{align*}
  \meta \leq {} & \epsilon (2\lambda + G_f)(\lambda_f + G_f)D_f^2 T + \frac{\ln(1/p_{1,i})}{\epsilon}\\
  = {} &  D_f \sqrt{2(2\lambda + G_f)(\lambda + G_f) T}\big( 1 + \ln(1+i) \big),
\end{align*}
where the equation can be obtained by an appropriate setting of the learning rate $\epsilon$.

Since the above decomposition and the upper bounds of meta-regret and base-regret all hold for any base-learner index $i\in[N]$, we will choose the best index denoted by $i^*$ to make the regret bound tightest possible.  Specifically, from the construction of the step size pool, we can ensure that there exists a step size $\eta_{i^*}$ such that the optimal step size provably satisfies $\eta_{i^*} \leq \eta_* \leq 2\eta_{i^*}$. As a result, we have 
\begin{align*}
  {} & \sum_{t=1}^{T} g_t(M_{t}) - \sum_{t=1}^{T} g_t(M^*_{t}) + \lambda \sum_{t=2}^{T} \Fnorm{M_{t-1} - M_t}\\
  \leq {} & \frac{\eta_{i^*}}{2} (G_f^2 + 2\lambda G_f) T + \frac{1}{2\eta_{i^*}}(D_f^2 + 2D_f P_T) + D_f \sqrt{2(2\lambda + G_f)(\lambda + G_f) T}\big( 1 + \ln(1+i) \big)\\
  \leq {} & \frac{\eta_*}{2} (G_f^2 + 2\lambda G_f) T + \frac{1}{\eta_*}(D_f^2 + 2D_f P_T) + D_f \sqrt{2(2\lambda + G_f)(\lambda + G_f) T}\big( 1 + \ln(1+i) \big)\\
  \leq {} & \frac{3}{2} \sqrt{(G_f^2 + 2\lambda G_f)(D_f^2 + 2D_f P_T) T}\\
  & \qquad + D_f \sqrt{2(2\lambda + G_f)(\lambda + G_f) T} \left(1 + \ln( \lceil \log_2 (1 + 2P_T/D) \rceil +2) \right).
\end{align*}
Combining this result with the regret decomposition~\eqref{eq:control-regret-decompose} and the upper bounds~\eqref{eq:control-regret-A-C},~\eqref{eq:control-regret-B}, we have 
\begin{align*}
  & \sum_{t=1}^{T} c_t(x_t,u_t) - \sum_{t=1}^{T} c_t(x_t^{\pi_t},u_t^{\pi_t}) \\
  \leq {}&  4 T G_c D^2 \kappa^3 (1-\gamma)^{H+1} + \frac{3}{2} \sqrt{(G_f^2 + 2\lambda G_f)(D_f^2 + 2D_f P_T) T}\\
  \qquad {} & \qquad \qquad + D_f \sqrt{2(2\lambda + G_f)(\lambda + G_f) T} \left(1 + \ln( \lceil \log_2 (1 + 2P_T/D) \rceil +2) \right) + \lambda P_T.
\end{align*}
The specific values of $D, L_f, G_f, D_f$ can be found in Lemma~\ref{lemma:support-2-f-property}. By setting $H = \O(\log T)$, we obtain an $\Ot(\sqrt{T(1+P_T)})$ dynamic policy regret and hence complete the proof.
\end{proof}

\subsection{{Proof of Theorem~\ref{thm:unknown-system}}}
\label{sec:unknown-system}
In this part, we present the proof of Theorem~\ref{thm:unknown-system}. Specifically, we provide the main proof of Theorem~\ref{thm:unknown-system} in Appendix~\ref{subsub:unknown-system} and the proofs of some key lemmas in Appendix~\ref{subsub:lemma-known}.

\paragraph{Notations.} We define some notations for convenience. Define $\epsilon_w$ an upper bound for the gap between the true disturbance $w_t$ and the estimated one $\wh_t$, i.e., $\norm{w_t-\wh_t}_2 \le \epsilon_w$, and define a universal upper bound $W_0$ for $\epsilon_w$ and disturbance bound $W$ (cf. Assumption~\ref{assume:bound-noise}) as $W,\epsilon_w \le W_0$. We also define $d_{\min}=\min\{d_x,d_u\}, \tilde{A}_K=A-BK, \hat{A}_K = \Ah - \Bh K$ for notational convenience.

\subsubsection{Proof of Theorem~\ref{thm:unknown-system}}

\begin{proof}
  \label{subsub:unknown-system}
The overall dynamic regret is at most
\begin{equation*}
  \sum_{t=1}^{T} c_t(x_t,u_t) - \sum_{t=1}^{T} c_t(x_t^{\pi_t},u_t^{\pi_t}) \le \underbrace{\sum_{t=1}^{T_0} c_t(x_t,u_t)}_{\term{A}} + \underbrace{\sum_{t=T_0+1}^{T} c_t(x_t,u_t) - \sum_{t=T_0+1}^{T} c_t(x_t^{\pi_t},u_t^{\pi_t})}_{\term{B}},
\end{equation*}
where term~(A) is the cumulative cost during the system identification procedure and term~(B) is the dynamic regret caused by \SC~algorithm over the rest rounds. Note that term~(A) enjoys a trivial upper bound of $\O(T_0)$, and term~(B) can be decomposed into two parts:
\begin{align*}
  \term{B} = {} & \underbrace{\sum_{t=T_0+1}^{T} c_t(x_t,u_t) - \sum_{t=T_0+1}^{T} c_t(x_t^{\pi_t}(\Sh),u_t^{\pi_t}(\Sh))}_{\term{b\mbox{-}1}}\\
  & \quad + \underbrace{\sum_{t=T_0+1}^{T} c_t(x_t^{\pi_t}(\Sh),u_t^{\pi_t}(\Sh)) - \sum_{t=T_0+1}^{T} c_t\sbr{x_t^{\pi_t}(S),u_t^{\pi_t}(S)}}_{\term{b\mbox{-}2}}.
\end{align*}
Here, $(x_t^{\pi_t}(S),u_t^{\pi_t}(S))$ is the state-action pair produced by the policy $\pi_t$ on the true system $S=(A,B,\{w\})$, whereas $(x_t^{\pi_t}(\Sh),u_t^{\pi_t}(\Sh))$ is the state-action pair produced by the policy $\pi_t$ on the estimated system $\Sh=(\Ah,\Bh,\{\wh\})$. Summarizing, term~(b-1) is the dynamic regret on the estimated system and term~(b-2) is the gap between the cumulative cost of the true system and that of the estimated system. From Theorem~\ref{thm:main-result-control}, it holds that $\term{b\mbox{-}1} \leq \tilde{\O}(\sqrt{T(1+P_T)})$. From Lemma~\ref{lem:cumulative-cost-difference}, we can bound term~(b-2) as $\term{b\mbox{-}2} \le \O(\epsilon_{A,B}T)$. Overall, with probability at least $1-\delta$, the total dynamic regret is at most 
\begin{align*}
  \textnormal{D-Regret}_T \le {} & \O(T_0) + \tilde{\O}(\sqrt{T(1+P_T)}) + \O(\epsilon_{A,B}T)\\
  = {} & \O(\epsilon_{A,B}^{-2}+\epsilon_{A,B}T) + \tilde{\O}(\sqrt{T(1+P_T)})\\
  \le {} & \O(T^{2/3}) + \tilde{\O}(\sqrt{T(1+P_T)}).
\end{align*}
The second step makes use of the relationship between the system identification rounds $T_0$ and the estimation error $\opnorm{\Ah-A}, \opnorm{\Bh-B} \le \epsilon_{A,B}$, as demonstrated in Lemma~\ref{lem:system-recovery}. The last step holds by setting the rounds of exploration to ensure $\epsilon_{A,B} = \min\{10^{-3} \kappa^{-10} \gamma^2,T^{-1/3}\}$, which is realized when total time horizon is large enough, i.e., $T \ge 10^9 \kappa^{30} \gamma^{-6}$. 
\end{proof}

\subsubsection{Key Lemmas in Unknown Systems}
\label{subsub:lemma-known}
The proof of Theorem~\ref{thm:unknown-system} relies on the two key lemmas (Lemma~\ref{lem:system-recovery} and Lemma~\ref{lem:cumulative-cost-difference}). In the following, we provide the formal statements and corresponding proofs.

Lemma~\ref{lem:system-recovery} establishes the relationship between the estimation accuracy $\epsilon_{A,B}$ and the number of estimation rounds $T_0$. This lemma is firstly due to~\citet{ALT'20:control-Hazan} and is restated here for self-containedness.

\begin{myLemma}[Theorem 19 of \citet{ALT'20:control-Hazan}]
  \label{lem:system-recovery}
  Under Assumptions~\ref{assume:bound-noise}, \ref{assume:strongly-stable}, \ref{assume:strong-controllability}, when Algorithm~\ref{alg:system-identification} runs for $T_0$ rounds, if the output pair $(\Ah,\Bh)$ satisfies, with probability at least $1-\delta$, that $\opnorm{\Ah-A},\opnorm{\Bh-B} \le \epsilon_{A,B}$, then it holds that $T_0 = \O(\epsilon_{A,B}^{-2})$.
\end{myLemma}

\begin{proof}[of Lemma~\ref{lem:system-recovery}]
  Based on the observation, we have the following two equations:
  \begin{equation*}
      \tilde{A}_K C_k = (\tilde{A}_K C_k), \quad \Ah_K \hat{C}_0 = \hat{C}_1.
  \end{equation*}
  Using Lemma~\ref{suplem:perturbation-analysis}, it holds that
  \begin{equation}
      \label{eq:tilde-A-gap}
      \opnorm{\tilde{A}_K - \Ah_K} \le \frac{\opnorm{\tilde{A}_K C_k - \hat{C}_1} + \opnorm{C_k - \hat{C}_0}\opnorm{\tilde{A}_K}}{\sigma_{\min}(C_k) - \opnorm{C_k - \hat{C}_0}}.
  \end{equation}
  Lemma~\ref{lem:moment-recovery} tells that with probability at least $1-\delta$, $\Fnorm{N_j - \tilde{A}_K^{j} B} \leq \epsilon$, where
  \begin{equation}
    \label{eq:concentration-acc}
       \epsilon \define 3 \kappa_B\kappa^2 d_u W \gamma^{-1} \sqrt{\frac{2d_{\min}\log\sbr{2e^2k\delta^{-1}}}{T_0-k}}.
  \end{equation}
  Owing to the benign high-probability guarantee, we only need to focus on the successful event, that is, under the case when $\Fnorm{N_j - \tilde{A}_K^{j} B} \leq \epsilon$ is true. We then try to bound $\opnorm{C_k - C_0}, \opnorm{\tilde{A}_K C_k - C_1}$,
  \begin{align}
      & 
      \begin{aligned}
          \label{eq:Ck-C0}
          \opnorm{C_k - \hat{C}_0}
          \le {} & \Fnorm{C_k - \hat{C}_0}
          = \left\| \mbr{N_0-B,\ldots,N_{k-1}-\tilde{A}_K^{k-1}B} \right\|_{\mathrm{F}}\\
          = {} & \sqrt{\sum_{i=0}^{k-1} \Fnorm{N_i-\tilde{A}_K^iB}^2} \le \sqrt{k \epsilon^2} = \epsilon \sqrt{k},
      \end{aligned}\\
      &
      \begin{aligned}
          \label{eq:ACk-C1}
          \opnorm{\tilde{A}_K C_k - \hat{C}_1}
          \le {} & \Fnorm{\tilde{A}_K C_k - \hat{C}_1}
          = \left\| \mbr{N_1-\tilde{A}_KB,\ldots,N_k-\tilde{A}_K^kB} \right\|_{\mathrm{F}}\\
          = {} & \sqrt{\sum_{i=1}^k \Fnorm{N_i-\tilde{A}_K^iB}^2} \le \sqrt{k \epsilon^2} = \epsilon \sqrt{k}.
      \end{aligned}
  \end{align}
  Using Lemma~\ref{suplem:strong-controllability} to upper-bound $\sigma_{\min}(C_k)$, and  plugging \eqref{eq:Ck-C0} and \eqref{eq:ACk-C1} into \eqref{eq:tilde-A-gap}, we have
  \begin{equation*}
      \opnorm{\tilde{A}_K - \Ah_K} \le \frac{\epsilon \sqrt{k} + \epsilon \sqrt{k} \cdot \kappa^2 (1-\gamma)}{1/\sqrt{\kappa_c}-\epsilon \sqrt{k}}.
  \end{equation*}
  The gap between $A$ and $\Ah$ can be bounded as 
  \begin{align*}
      \opnorm{A - \Ah} = {} & \opnorm{\tilde{A}_K + BK - \Ah_K - \Bh K}\\
      \le {} & \opnorm{\tilde{A}_K - \Ah_K} + \opnorm{K}\opnorm{B-\Bh}\\
      \le {} & \frac{\epsilon \sqrt{k} + \epsilon \sqrt{k} \cdot \kappa^2 (1-\gamma)}{1/\sqrt{\kappa_c}-\epsilon \sqrt{k}} + \kappa \epsilon \le \frac{3\epsilon \kappa^{5/2}}{\sqrt{1/\kappa_c}-\epsilon \sqrt{\kappa}}.
  \end{align*}
  If we want $\Fnorm{\Ah-A},\Fnorm{\Bh-B} \le \epsilon_{A,B}$, the following equations should hold:
  \begin{equation}
      \label{eq:eps-epsAB}
      \begin{aligned}
      \Fnorm{\Ah-A} \le {} & \sqrt{d_x} \opnorm{\Ah-A} \le \sqrt{d_x} \sbr{\frac{3\epsilon \kappa^{5/2}}{\sqrt{1/\kappa_c}-\epsilon \sqrt{\kappa}}} \define \epsilon_A \le \epsilon_{A,B},\\
      \Fnorm{\Bh-B} \le {} & \sqrt{d_{\min}} \opnorm{\Bh-B} \le \sqrt{d_{\min}} \epsilon \define \epsilon_B \le \epsilon_{A,B}.
      \end{aligned}
  \end{equation}
  Besides, it is easy to see that $\epsilon_B = \sqrt{d_{\min}} \epsilon \le \sqrt{d_x} \epsilon \le \epsilon_A$, thus conditions in \eqref{eq:eps-epsAB} can be simplified as $\epsilon_A \le \epsilon_{A,B}$. Finally, combining the above inequality with the value of $\epsilon$ (c.f.~\eqref{eq:concentration-acc}), we can obtain that $T_0 = \O({\epsilon_{A,B}^{-2}})$.
  \end{proof}

Lemma~\ref{lem:cumulative-cost-difference} measures the difference of the cumulative costs of a policy between the true system and the estimated one. This result holds for both strongly stable linear controllers and non-stationary DAC policy and here we only give a proof of the latter, for the former result, we refer readers to {\citet[Lemma 16]{ALT'20:control-Hazan}}.
\begin{myLemma}[Identification Accuracy]
\label{lem:cumulative-cost-difference}
Under Assumptions~\ref{assume:bound-noise}-\ref{assume:strongly-stable}, suppose $\opnorm{\Ah-A},\opnorm{\Bh-B} \le \epsilon_{A,B} \le 0.25 \kappa^{-3} \gamma$ and let $K$ be any $(\kappa,\gamma)$-strongly stable linear controller with respect to $(A,B)$. Then for any non-stationary DAC policy $\pi_{1:T}$ parameterized via $M_{1:T}$,
  \begin{equation*}
      \left|\sum_{t=T_0+1}^{T} c_t\sbr{x_t^{\pi_t}(\Sh),u_t^{\pi_t}(\Sh)}- \sum_{t=T_0+1}^{T} c_t\sbr{x_t^{\pi_t}(S),u_t^{\pi_t}(S)}\right| \le \O\sbr{\epsilon_{A,B} T+\epsilon_{A,B}^2 T},
  \end{equation*}
  where $(x_t^{\pi_t}(S),u_t^{\pi_t}(S))$ is the state-action pair produced by policy $\pi_t$ on the true system $S=(A,B,\{w\})$ and $(x_t^{\pi_t}(\Sh),u_t^{\pi_t}(\Sh))$ is produced on the estimated system $\Sh=(\Ah,\Bh,\{\wh\})$.
\end{myLemma}

\begin{proof}[of Lemma~\ref{lem:cumulative-cost-difference}]
  If the policy is a non-stationary DAC policy parameterized via $M_{1:T}$, in system $(A,B,\{w\})$, it holds that
\begin{align*}
    \norm{x_{t+1}^{\pi_t}(S)}_2 \le {} & W\sum_{i=0}^{H+t} \opnorm{\Psi_{t,i}^{K,t}(M_{0:t})} \\
    = {}& W\sum_{i=0}^{H+t} \opnorm{\tilde{A}_K^i \ind{i \leq t} + \sum_{j=0}^{t} \tilde{A}_K^j B M_{t-j}^{[i-j]} \ind{1 \leq i-j \leq H}}\\
    \le {} & W \sbr{\kappa^2 \sum_{i=0}^{H+t} (1-\gamma)^i + \kappa_B^2\kappa^3 \sum_{i=0}^{H+t} \sum_{j=0}^{t} \opnorm{\tilde{A}_K^j \ind{1 \leq i-j \leq H}}}\\
    \le {} & W \sbr{\kappa^2 \gamma^{-1} + \kappa_B^2\kappa^3 \sum_{i=0}^{H+t} \sum_{j=i-H}^{i-1} \opnorm{\tilde{A}_K^{j}}\ind{0 \leq j \leq t}}\\
    \le {} & W \sbr{\kappa^2 \gamma^{-1} + \kappa_B^2\kappa^5 \sum_{i=0}^{H+t} \sum_{j=i-H}^{i-1} (1-\gamma)^{j}\ind{0 \leq j \leq t}}\\
    \le {} & W \sbr{\kappa^2 \gamma^{-1} + \kappa_B^2 \kappa^5 H \sum_{i=0}^{t} (1-\gamma)^{i}}\\
    \le {} & W\sbr{\kappa^2 \gamma^{-1} + \kappa_B^2 \kappa^5 H \gamma^{-1}} \\
    \le {} & 2W \kappa_B^2 \kappa^5 \gamma^{-1} H.
\end{align*}
By Lemma~\ref{suplem:stability-preserve}, a linear controller $K$ is $\sbr{\kappa, \gamma-2\kappa^3 \epsilon_{A,B}}$-strongly stable with respect to the estimated system $\Sh=(\Ah,\Bh,\{\wh\})$ if it is $(\kappa,\gamma)$-strongly stable for the true system $S=(A,B,\{w\})$. Thus it can be easily verified that 
\begin{align*}
    1 - \gamma + 2 \kappa^3 \epsilon_{A,B} \le 1 - \gamma + 2 \kappa^3 \cdot 0.25 \kappa^{-3} \gamma = 1 - \gamma/2.
\end{align*}
For simplicity, we can say that linear controller $K$ is $(\kappa, \gamma/2)$-strongly stable for the estimated system $\Sh$. Further, let $\opnorm{\Bh} \le \kappa_{\Bh}$, it holds that 
\begin{equation*}
  \kappa_{\Bh} = \opnorm{\Bh} = \opnorm{(\Bh-B)+B} \le \epsilon_{A,B} + \kappa_B \le 2\kappa_B.
\end{equation*}
As a result, we can bound $\norm{x_{t+1}^{\pi_t}(\Sh)}_2$ as 
\begin{equation*}
  \norm{x_{t+1}^{\pi_t}(\Sh)}_2 \le 2 (\epsilon_w+W)(2\kappa_B)^2 \kappa^5 (\gamma/2)^{-1} H = 32 W_0 \kappa_B^2 \kappa^5 \gamma^{-1} H.
\end{equation*}
As for the action $u_t^{\pi_t}(\Sh)$, we can bound it as
\begin{align*}
    \norm{u_t^{\pi_t}(\Sh)}_2 \le {} & \norm{-Kx_t^{\pi_t}(\Sh)}_2 + \left\|\sum_{i=1}^H M_{t}^{[i]} \wh_{t-i}\right\|_2 \le 32 W_0 \kappa_B^2 \kappa^6 \gamma^{-1} H + 2W_0 \kappa_B \kappa^3 \gamma^{-1}\\
    \le {} & 34 W_0 \kappa_B^2 \kappa^6 \gamma^{-1} H.
\end{align*}
Thus, the diameter of the state-action domain in the estimated system, denoted as $\hat{D}$, is at most $\hat{D}\define \max_{t \in [T]}\max\{\norm{x_t(\Sh)}_2,\norm{u_t(\Sh)}_2\}=34 W_0 \kappa_B^2 \kappa^6 \gamma^{-1} H$.
The gap of the cumulative costs between the true system and the estimated system can be bounded as
\begin{equation}
\label{eq:difference-between-system}
\begin{split}
    & \left| \sum_{t=T_0+1}^{T} c_t\sbr{x_t^{\pi_t}(\Sh),u_t^{\pi_t}(\Sh)}- \sum_{t=T_0+1}^{T} c_t\sbr{x_t^{\pi_t}(S),u_t^{\pi_t}(S)}\right|\\
    \le {} & G_c \hat{D} \sum_{t=1}^T \norm{x_t^{\pi_t}(\Sh)-x_t^{\pi_t}(S)}_2 + G_c \hat{D} \sum_{t=1}^T \norm{u_t^{\pi_t}(\Sh)-u_t^{\pi_t}(S)}_2.
\end{split}
\end{equation}
We start by analyzing $\norm{u_t^{\pi_t}(\Sh)-u_t^{\pi_t}(S)}_2$:
\begin{equation}
  \label{eq:action-gap}
  \begin{aligned}
    \norm{u_t^{\pi_t}(\Sh)-u_t^{\pi_t}(S)}_2 = {} & \left\|\sbr{-Kx_t^{\pi_t}(\Sh)+\sum_{i=1}^H M_t^{[i]}\wh_{t-i}} - \sbr{-Kx_t^{\pi_t}(S)+\sum_{i=1}^H M_t^{[i]}w_{t-i}} \right\|_2\\
    \le {} & \kappa \norm{x_t^{\pi_t}(\Sh)-x_t^{\pi_t}(S)}_2 + \sum_{i=1}^H \norm{M_t^{[i]}(\wh_{t-i}-w_{t-i})}\\
    \le {} & \kappa \norm{x_t^{\pi_t}(\Sh)-x_t^{\pi_t}(S)}_2 + \epsilon_w \kappa_B \kappa^3 \sum_{i=1}^H (1-\gamma)^i\\
    \le {} & \kappa \norm{x_t^{\pi_t}(\Sh)-x_t^{\pi_t}(S)}_2 + \epsilon_w \kappa_B \kappa^3 \gamma^{-1}.
  \end{aligned}
\end{equation}
Plugging \eqref{eq:action-gap} into \eqref{eq:difference-between-system}, it holds that
\begin{equation}
  \label{eq:cost-gap-2}
  \begin{split}
    & \left| \sum_{t=T_0+1}^{T} c_t\sbr{x_t^{\pi_t}(\Sh),u_t^{\pi_t}(\Sh)}- \sum_{t=T_0+1}^{T} c_t\sbr{x_t^{\pi_t}(S),u_t^{\pi_t}(S)}\right| \\
    \le {} & 2 \kappa G_c \hat{D} \sum_{t=1}^T \norm{x_t^{\pi_t}(\Sh)-x_t^{\pi_t}(S)}_2 + G_c \hat{D} \epsilon_w \kappa_B \kappa^3 \gamma^{-1} T.
  \end{split}
\end{equation}
This motivates the need to analyze $\norm{x_t^{\pi_t}(\Sh)-x_t^{\pi_t}(S)}_2$. To begin with, we define $\Psih_{t,i}^{K,h}(M_{t-h:t}) = \Ah_K^i \mathbf{1}_{i \leq h} + \sum_{j=0}^{h} \Ah_K^j \Bh M_{t-j}^{[i-j]} \mathbf{1}_{1 \leq i-j \leq H}$, where $\Ah_K\define \Ah-\Bh K$. Expanding $x_t^{\pi_t}(\Sh)$ and $x_t^{\pi_t}(S)$ using Proposition~\ref{proposition:DAC-state}, it holds that
\begin{equation}
\label{eq:term-total}
\begin{split}
    & \norm{x_t^{\pi_t}(\Sh)-x_t^{\pi_t}(S)}_2 = \left\|\sum_{i=0}^{H+t}\Psi_{t,i}^{K,t}(M_{1:t})w_{t-i}-\sum_{i=0}^{H+t}\Psih_{t,i}^{K,t}(M_{1:t})\wh_{t-i}\right\|_2\\
    \le {} & \underbrace{\left\| \sum_{i=0}^{H+t}\Psi_{t,i}^{K,t}(M_{1:t})w_{t-i}-\sum_{i=0}^{H+t}\Psi_{t,i}^{K,t}(M_{1:t})\wh_{t-i} \right\|_2}_{\mathtt{term~(i)}} + \underbrace{\left\| \sum_{i=0}^{H+t}\Psi_{t,i}^{K,t}(M_{1:t})\wh_{t-i}-\sum_{i=0}^{H+t}\Psih_{t,i}^{K,t}(M_{1:t})\wh_{t-i} \right\|_2}_{\mathtt{term~(ii)}}.
\end{split}
\end{equation}
First, we analyze term~(i):
\begin{equation}
\label{eq:term-i}
    \mathtt{term~(i)} \le \epsilon_w \sum_{i=0}^{H+t} \opnorm{\Psi_{t,i}^{K,t}(M_{1:t})}
    \le 2\epsilon_w \kappa_B^2\kappa^5 \gamma^{-1} H.
\end{equation}
Second, we investigate term~(ii):
\begin{align}
    \mathtt{term~(ii)} \le {} & (W+\epsilon_w) \sum_{i=0}^{H+t} \left\| \Psi_{t,i}^{K,t}(M_{1:t})-\Psih_{t,i}^{K,t}(M_{1:t}) \right\|_{\mathrm{op}} \notag\\
    \le {} & 2W_0\sum_{i=0}^{H+t}\sbr{\left\| \sbr{\tilde{A}_K^i - \Ah_K^i} \ind{i \leq t} \right\|_{\mathrm{op}} + \kappa_B\kappa^3\sum_{j=0}^{t}\opnorm{\tilde{A}_K^j B  - \Ah_K^j \Bh} \ind{1 \leq i-j \leq H}} \notag\\
    \le {} & 2W_0\kappa^2 \underbrace{\sum_{i=0}^t \opnorm{L^i-\hat{L}^i}}_{\term{a}} + 2W_0\kappa_B\kappa^3 \underbrace{\sum_{i=0}^{H+t}\sum_{j=0}^{t}\opnorm{\tilde{A}_K^j B  - \Ah_K^j \Bh} \ind{1 \leq i-j \leq H}}_{\term{b}} \label{eq:term-ii}.
\end{align}
For term~(a), using Lemma~\ref{suplem:matrix-minus-norm-sum}, it holds that
\begin{equation*}
  \sum_{i=0}^t \opnorm{L^i-\hat{L}^i} \le 3\gamma^{-2} \opnorm{L-\hat{L}} \le 3\gamma^{-2} \cdot 2\kappa^3 \epsilon_{A,B} = 6\kappa^3 \gamma^{-2} \epsilon_{A,B}.
\end{equation*}
For term~(b), by inserting an intermediate term, we have
\begin{align*}
  & \sum_{i=0}^{H+t}\sum_{j=0}^{t}\opnorm{\tilde{A}_K^j B  - \Ah_K^j \Bh} \ind{1 \leq i-j \leq H} \\ 
  \le {} & \sum_{i=0}^{H+t}\sum_{j=0}^{t}\opnorm{\tilde{A}_K^j B  - \tilde{A}_K^j \Bh} \ind{1 \leq i-j \leq H} + \sum_{i=0}^{H+t}\sum_{j=0}^{t}\opnorm{\tilde{A}_K^j \Bh  - \Ah_K^j \Bh} \ind{1 \leq i-j \leq H}\\
  \le {} & \epsilon_{A,B}\sum_{i=0}^{H+t}\sum_{j=0}^{t} \opnorm{\tilde{A}_K^j} \ind{1 \leq i-j \leq H} + \kappa_{\Bh} \sum_{i=0}^{H+t}\sum_{j=0}^{t}\opnorm{\tilde{A}_K^j - \Ah_K^j} \ind{1 \leq i-j \leq H}\\
  \le {} & \epsilon_{A,B} H \gamma^{-1} + 2\kappa_B \kappa^2 H \sum_{i=0}^t \opnorm{L^i-\hat{L}^i}\\
  \le {} & \epsilon_{A,B} H \gamma^{-1} + 2\kappa_B \kappa^2 H \cdot 6\kappa^3 \gamma^{-2} \epsilon_{A,B}.
\end{align*}
Plugging term~(a) and term~(b) into \eqref{eq:term-ii}, we have
\begin{align*}
  \term{ii} \le {} & 2W_0 \kappa^2 \cdot 6\kappa^3 \gamma^{-2} \epsilon_{A,B} + 2W_0\kappa_B\kappa^3 \cdot (\epsilon_{A,B} H \gamma^{-1} + 2\kappa_B \kappa^2 H \cdot 6\kappa^3 \gamma^{-2} \epsilon_{A,B})\\
  \le {} & 38W_0\kappa_B^2 \kappa^8 \gamma^{-2} H \epsilon_{A,B}.
\end{align*}
Plugging the bounds of \eqref{eq:term-i} and \eqref{eq:term-ii} into \eqref{eq:term-total}, we have
\begin{equation}
    \label{eq:x-gap-two-system}
    \norm{x_t^{\pi_t}(\Sh)-x_t^{\pi_t}(S)}_2 \le 2\epsilon_w \kappa_B^2\kappa^5 \gamma^{-1} H + 38W_0\kappa_B^2 \kappa^8 \gamma^{-2} H \epsilon_{A,B}.
\end{equation}
Furthermore, by Lemma~\ref{suplem:W-bound}, we have
\begin{equation*}
    W_0 \le 2 \sqrt{d_u} \kappa^3 \gamma^{-1} W, \quad \epsilon_w \le 42 \sqrt{d_u} \kappa^{12} \gamma^{-3} W \epsilon_{A,B}
\end{equation*}
Plugging $W_0$ and $\epsilon_w$ into \eqref{eq:x-gap-two-system}, it holds that
\begin{equation*}
    \norm{x_t^{\pi_t}(\Sh)-x_t^{\pi_t}(S)}_2 \le \O(\epsilon_{A,B} + \epsilon_{A,B}^2).
\end{equation*}
Plugging the above bound into \eqref{eq:cost-gap-2}, we have
\begin{equation*}
  \left|\sum_{t=T_0+1}^{T} c_t\sbr{x_t^{\pi_t}(\Sh),u_t^{\pi_t}(\Sh)}- \sum_{t=T_0+1}^{T} c_t\sbr{x_t^{\pi_t}(S),u_t^{\pi_t}(S)}\right| \le \O\sbr{\epsilon_{A,B} T+\epsilon_{A,B}^2 T},
\end{equation*}
which finishes the proof.
\end{proof}

\subsection{Proof of Corollary~\ref{corollary:static-regret-control}}
We now present the proof of Corollary~\ref{corollary:static-regret-control}, i.e., the static policy regret of the controller. Corollary~\ref{corollary:static-regret-control} states that when the system dynamics are known, \SC~enjoys the following static policy regret,
\begin{equation}
  \label{eq:cor-static-regret-control}
  \sum_{t=1}^{T} c_t(x_t,u_t) - \min_{\pi \in \Pi} \sum_{t=1}^{T} c_t(x_t^{\pi},u_t^{\pi}) \leq \Ot(\sqrt{T}),
\end{equation}
where the comparator set $\Pi$ can be chosen as either the set of DAC policies or the set of strongly linear controllers. Let us denote the two comparator sets as $\Pi_{\mathrm{DAC}}$ and $\Pi_{\mathrm{SLC}}$, respectively. Moreover, when the system dynamics are unknown, using the identification algorithm of~\citet{ALT'20:control-Hazan}, we can achieve an $\Ot(T^{2/3})$ static regret, which also holds for either the set of DAC policies or the set of strongly linear controllers. Therefore, in the following we will prove the statement for two comparator sets separately.~\\

\begin{proof}[{of Corollary~\ref{corollary:static-regret-control}}]
When the comparator set $\Pi$ is chosen as the set of DAC policies, i.e., $\pi \in \Pi_{\mathrm{DAC}} = \{\pi(K,M) \given M \in \M \}$, the result of~\eqref{eq:cor-static-regret-control} can be easily obtained from Theorem~\ref{thm:main-result-control} by setting $\pi_1 = \ldots = \pi_T = \pi_* \in \argmin_{\pi \in \Pi} \sum_{t=1}^{T} c_t(x_t^{\pi},u_t^{\pi})$. Under such a case, the path length $P_T = \sum_{t=2}^{T} \Fnorm{M_{t-1} - M_t} = 0$, and thus
\[
  \sum_{t=1}^{T} c_t(x_t,u_t) - \min_{\pi \in \Pi_{\mathrm{DAC}}} \sum_{t=1}^{T} c_t(x_t^{\pi},u_t^{\pi}) \leq \Ot(\sqrt{T}).
\]

On the other hand, when choosing the comparator set $\Pi$ as $\Pi_{\mathrm{SL}}$, i.e., $\pi = K \in \Pi_{\mathrm{SL}} = \{K \mid  K\mbox{ is } (\kappa,\gamma)\mbox{-strongly stable}\}$, we will need some efforts to prove the statement. 

We show that the statement can be obtained by further incorporating Lemma~\ref{lemma:sufficiency}, which demonstrates that minimizing static policy regret over the DAC class is sufficient to deliver a policy regret competing with the strongly linear controller class~\citep[Lemma 5.2]{ICML'19:online-control}. In fact, denote by $\pi^* = K^\star = \argmin_{K \in \Pi_{\mathrm{SL}}} \sum_{t=1}^{T} c_t(x_t^{K},u_t^{K})$ , and we have
\begin{align*}
  {} & \sum_{t=1}^{T} c_t(x_t,u_t) - \min_{\pi \in \Pi_{\mathrm{SLC}}} \sum_{t=1}^{T} c_t(x_t^{\pi},u_t^{\pi})\\
  = {} & \sum_{t=1}^{T} c_t(x_t,u_t) - \min_{\pi \in \Pi_{\mathrm{DAC}}} \sum_{t=1}^{T} c_t(x_t^{\pi},u_t^{\pi}) + \min_{\pi \in \Pi_{\mathrm{DAC}}} \sum_{t=1}^{T} c_t(x_t^{\pi},u_t^{\pi}) - \sum_{t=1}^{T} c_t(x_t^{K^*},u_t^{K^*})\\
  \leq {} & \Ot(\sqrt{T}) + \sum_{t=1}^{T} c_t(x_t^{\pi(M_{\Delta}, K)},u_t^{\pi(M_{\Delta}, K)}) - \sum_{t=1}^{T} c_t(x_t^{K^*},u_t^{K^*})\\
  \leq {} & \Ot(\sqrt{T}) + T \cdot 4G_c D W H \kappa_B^2 \kappa^6 (1-\gamma)^{H-1}\gamma^{-1} \leq \Ot(\sqrt{T}),
\end{align*}
where the first inequality uses the optimality of $\argmin_{\pi \in \Pi_{\mathrm{DAC}}} \sum_{t=1}^{T} c_t(x_t^{\pi},u_t^{\pi})$ and $\pi(M_{\Delta}, K)$ is a DAC policy with $M_{\Delta} = (M_{\Delta}^{[1]},\ldots,M_{\Delta}^{[H]})$ defined by $M_{\Delta}^{[i]} = (K-K^\star)(A - BK^\star)^i$. The second inequality holds by Lemma~\ref{lemma:sufficiency}, and the final inequality sets $H = \O(\log T)$.

The above arguments hold for the known system setting. On the other hand, when the system dynamics are unknown, using the system identification yields an additional  estimation overhead of order $\Ot(T^{2/3})$ no matter which comparator set is chosen. Therefore, the overall regret remains $\Ot(T^{2/3})$ for unknown systems. Hence, we complete the proof.
\end{proof} 
\subsection{Supporting Lemmas}
\label{appendix-sec:support-lemmas}
In this part, we provide several supporting lemmas used frequently in the analysis of online non-stochastic control. Most of them are due to the pioneering works~\citep{ICML'19:online-control,ALT'20:control-Hazan}, and we adapt them to our notations and provide the proofs to achieve self-containedness. Specifically,
\begin{itemize}
  \item Lemma~\ref{lemma:norm-relation} establishes the norm relations between the $\ell_1, \mathrm{op}$ norm and Frobenius norm used in the $\M$-space. 
  \item Lemma~\ref{lemma:support-1-boundedness} checks the boundedness of several variables of interest. 
  \item Lemma~\ref{lemma:support-2-f-property} shows several properties of the truncated functions $\{f_t\}_{t=1}^T$ and the feasible set $\M$. 
\item Lemma~\ref{lemma:sufficiency} connects the DAC class and the strongly linear controller class.
  \item Lemma~\ref{lem:moment-recovery} -- Lemma~\ref{suplem:perturbation-analysis} are useful for analysis in unknown systems.
\end{itemize}

\begin{myLemma}[Norm Relations]
\label{lemma:norm-relation}
For any $M = (M^{[1]},\ldots,M^{[H]}) \in \M \subseteq (\R^{d_u \times d_x})^H$, its $\ell_1, \mathrm{op}$ norm and Frobenius norm are defined by
\[
  \Loneopnorm{M} \define  \sum_{i=1}^{H} \opnorm{M^{[i]}}, \mbox{ and } \Fnorm{M} \define  \sqrt{\sum_{i=1}^{H} \Fnorm{M^{[i]}}^2}. 
\]
Denoting by $d = \min\{d_u, d_x\}$, we then have the following inequalities on their relations:
\begin{equation*}
  \Loneopnorm{M} \leq \sqrt{H} \Fnorm{M}, \mbox{ and }  \Fnorm{M} \leq \sqrt{d} \Loneopnorm{M}.
\end{equation*}
\end{myLemma}

\begin{proof}[{of Lemma~\ref{lemma:norm-relation}}]
We know that for any matrix $X \in \R^{m\times n}$, $\opnorm{X} \leq \Fnorm{X} \leq \sqrt{d} \opnorm{X}$. Therefore, by definition and Cauchy-Schwarz inequality, we obtain
\[
  \Loneopnorm{M} = \sum_{i=1}^{H} \opnorm{M^{[i]}} \leq \sum_{i=1}^{H} \Fnorm{M^{[i]}} \leq \sqrt{H} \Fnorm{M}.
\]
On the other hand, we have
\[
  \Fnorm{M} = \sqrt{\sum_{i=1}^{H} \Fnorm{M^{[i]}}^2} \leq \sum_{i=1}^{H} \Fnorm{M^{[i]}} \leq \sum_{i=1}^{H} \sqrt{d} \opnorm{M^{[i]}} = \sqrt{d} \Loneopnorm{M},
\]
which completes the proof.
\end{proof}

\begin{myLemma}[Lemma 5.5 of \citet{ICML'19:online-control}]
\label{lemma:support-1-boundedness}
Suppose $K$ and $K^\star$ are two $(\kappa, \gamma)$-strongly stable linear controllers (cf. Definition~\ref{def:controller-set}). Define
\begin{equation}
  \label{eq:diameter-D}
  D \define  \frac{W(\kappa^3 + H \kappa_B \kappa^3\tau)}{\gamma (1 - \kappa^2(1-\gamma)^{H+1})} + \frac{W\tau}{\gamma}.
\end{equation}
Suppose there exists a $\tau > 0$ such that for all $i \in [H]$ and $t \in [T]$, $\Fnorm{M_t^{[i]}} \leq \tau(1-\gamma)^i$. Then, we have
\begin{itemize}
  \item $\norm{x_t^K(M_{0:t-1})} \leq D$, $\norm{y_t^K(M_{t-H-1:t-1})} \leq D$, and $\norm{x_t^{K^\star}} \leq D$.
  \item $\norm{u_t^K(M_{0:t})} \leq D$, and $\norm{v_t^K(M_{t-H-1:t})} \leq D$.
  \item $\norm{x_t^K(M_{0:t-1}) - y_t^K(M_{t-1-H:t-1})} \leq \kappa^2 (1-\gamma)^{H+1}D$.
  \item $\norm{u_t^K(M_{0:t}) - v_t^K(M_{t-1-H:t})} \leq \kappa^3 (1-\gamma)^{H+1}D$.
\end{itemize}
In above, the definitions of state $x_t^K(M_{0:t-1})$ and corresponding DAC control $u_t^K(M_{0:t})$ can be found in Proposition~\ref{proposition:DAC-state}, and the definitions of truncated state $x_t^K(M_{0:t-1})$ and corresponding DAC control $v_t^K(M_{0:t})$ can be found in Definition~\ref{def:truncated-loss}. The definitions of state $x_t^{K^\star}$ can be found (and will be used) in Lemma~\ref{lemma:sufficiency}.
\end{myLemma}
\begin{proof}[{of Lemma~\ref{lemma:support-1-boundedness}}]
We first study the state.
\begin{align}
  \norm{x_t^K(M_{0:t-1})} =    {} & \left\|\tilde{A}_K^{H+1} x_{t-H-1}^K(M_{0:t-H-2}) + \sum_{i=0}^{2H} \Psi_{t-1,i}^{K,H} (M_{t-H-1:t-1}) w_{t-1-i}\right\|\nonumber\\
  \leq {} & \kappa^2 (1-\gamma)^{H+1} \norm{x_{t-H-1}^K(M_{0:t-H-2})} + W \sum_{i=0}^{2H} \norm{\Psi_{t-1,i}^{K,H} (M_{t-H-1:t-1})}\nonumber\\
  \leq {} & \kappa^2 (1-\gamma)^{H+1} \norm{x_{t-H-1}^K(M_{0:t-H-2})} + W \sum_{i=0}^{2H} \left( \kappa^2 (1-\gamma)^i + H \kappa_B\kappa^2 \tau (1-\gamma)^{i-1}\right)\nonumber\\
  \leq {} & \kappa^2 (1-\gamma)^{H+1} \norm{x_{t-H}^K(M_{0:t-H-1})} + W(\kappa^2 + H \kappa_B \kappa^2 \tau)/\gamma \nonumber\\
  \le {} & \frac{W (\kappa^2 + H \kappa_B \kappa^2 \tau)}{\gamma (1 - \kappa^2 (1-\gamma)^{H+1})} \leq D, \label{eq:support-range-1}
\end{align}
where inequality~\eqref{eq:support-range-1} is a summation of geometric series and the ratio of this series is $\kappa^2 (1-\gamma)^{H+1}$. Similarly,
\begin{align*}
  \norm{y_t^K(M_{t-1-H:t-1})} = {} & \left\|\sum_{i=0}^{2H} \Psi_{t-1,i}^{K,H} (M_{t-1-H : t-1}) w_{t-1-i}\right\| \\
  \leq {} & W \sum_{i=0}^{2H} \norm{\Psi_{t-1,i}^{K,H} (M_{t-1-H : t-1})} \\
  \leq {} & W \sum_{i=0}^{2H} \left( \kappa^2 (1-\gamma)^i + H \kappa_B\kappa^2 \tau (1-\gamma)^{i-1}\right)\\
  \leq {} & W \left(\frac{\kappa^2 + H \kappa_B \kappa^2 \tau}{\gamma}\right) \leq D.
\end{align*}
Besides,
\[
  \norm{x_t^{K^\star}} = \left\| \sum_{i=0}^{t-1} \tilde{A}_{K^\star}^i w_{t-1-i}\right\| \leq W \sum_{i=0}^{t-1} \kappa^2 (1-\gamma)^i \leq \frac{W \kappa^2}{\gamma} \leq D.
\]
So the difference can be evaluated as follows:
\[
  \norm{x_t^K(M_{0:t-1}) - y_t^K(M_{t-H-1:t-1})} = \norm{\tilde{A}_K^{H+1} x_{t-H-1}^K(M_{0:t-H-1})} \leq \kappa^2 (1-\gamma)^{H+1} D.
\]
We now consider the action (or control signal). 
\begin{align*}
  \norm{u_t^K(M_{0:t})} = {} & \left\| - K x_t^K(M_{0:t-1}) + \sum_{i=1}^{H} M_t^{[i]} w_{t-i} \right\|\\
  \leq {} &  \kappa \norm{x_t^K(M_{0:t-1})} + \sum_{i=1}^{H} W \tau (1-\gamma)^{i-1}\\
  \leq {} & \frac{W(\kappa^3 + H \kappa_B \kappa^3 \tau)}{\gamma(1 - \kappa^2(1-\gamma)^{H+1})}+\frac{W\tau}{\gamma} \leq D.
\end{align*}
Similarly,
\begin{align*}
  \norm{v_t^K(M_{t-H-1:t})} \leq \kappa \norm{y_t^K(M_{t-H-1:t-1})} + \sum_{i=1}^{H} W \tau (1-\gamma)^{i-1} \leq D.
\end{align*}
The difference of the actions is
\[
  \norm{u_t^K(M_{0:t-1}) - v_t^K(M_{t-H-1:t-1})} = \norm{- K(x_t^K(M_{0:t-1}) - y_t^K(M_{t-H-1:t-1}))} \leq \kappa^3 (1-\gamma)^{H+1} D,
\]
which finishes the proof.
\end{proof}

To reduce the online non-stochastic control to OCO with memory, in Definition~\ref{def:truncated-loss} we define the truncated loss $f_t: \M^{H+2} \mapsto \R$ as 
\begin{equation*}
  f_t(M_{t-1-H:t}) = c_t(y_t^K(M_{t-1-H:t-1}), v_t^K(M_{t-1-H:t})),
\end{equation*}
where $y_{t+1}^K(M_{t-H:t}) = \sum_{i=0}^{2H} \Psi_{t,i}^{K,H}(M_{t-H:t}) w_{t-i}$ and $v_{t+1}^K(M_{t-H:t+1}) = - K y_{t+1}(M_{t-H:t}) + \sum_{i=1}^{H} M_{t+1}^{[i]} w_{t+1-i}$. In the following lemma, we show several properties of the truncated functions $\{f_t\}_{t=1}^T$ and the feasible set $\M$ such that we can further apply the results of OCO with memory.

\begin{myLemma}
\label{lemma:support-2-f-property}
The truncated loss $f_t: \M^{H+2} \mapsto \R$ and the feasible set $\M$ satisfy the following properties. For notational convenience, we first let $D$ be defined the same as~\eqref{eq:def-D}, and we restate it below 
\[
  D \define  \frac{W\kappa^3(1+ H \kappa_B \tau)}{\gamma (1-\kappa^2(1-\gamma)^{H+1})} + \frac{W\tau}{\gamma}.
\]
\begin{enumerate}
  \item[(i)] The function is $L_f$-coordinate-wise Lipschitz with respect to the Euclidean (i.e., Frobenius) norm, namely, 
  \[
    \abs{f_t(M_{t-H-1},\ldots,M_{t-k},\ldots,M_t)} - \abs{f_t(M_{t-H-1},\ldots,\tilde{M}_{t-k},\ldots,M_t)} \leq L_f \Fnorm{M_{t-k} - \tilde{M}_{t-k}},
  \]
  where $L_f \leq 3\sqrt{H} G_c D W \kappa_B \kappa^3$.
  \item[(ii)] The gradient norm of surrogate loss $\f_t: \M \mapsto \R$ is bounded by $G_f$, i.e., $\Fnorm{\nabla_M \f_t(M)} \leq G_f$ holds for any $M \in \M$ and any $t \in [T]$, where $G_f \leq 3Hd^2G_cW\kappa_B \kappa^3 \gamma^{-1}$.
  \item[(iii)] The diameter of the feasible set is at most $D_f$, namely, $\Fnorm{M - M'} \leq D_f$ holds for any $M, M' \in \M$, where $D_f \leq 2\sqrt{d} \kappa_B \kappa^3 \gamma^{-1}$.
\end{enumerate}
\end{myLemma}

\begin{proof}[{of Lemma~\ref{lemma:support-2-f-property}}]
We first prove the claim (i), i.e., the $L_f$-coordinate-wise Lipschitz continuity.
For simplicity, we use the following definitions in the following arguments.
\begin{gather*}
  M_{t-H-1:t} \define  \{M_{t-H-1} \dots M_{t-k} \dots M_t\}, \quad M_{t-H-1:t-1} \define  \{M_{t-H-1} \dots M_{t-k} \dots M_{t-1}\}, \\
  \tilde{M}_{t-H-1:t} \define  \{M_{t-H-1} \dots \tilde{M}_{t-k} \dots M_t\}, \quad
  \tilde{M}_{t-H-1:t-1} \define  \{M_{t-H-1} \dots \tilde{M}_{t-k} \dots M_{t-1}\}.
\end{gather*}
By representing $f_t$ using $c_t$, we have
\begin{align}
  {} & f_t(M_{t-H-1:t}) - f_t(\tilde{M}_{t-H-1:t}) \notag\\
  = {} & c_t\left(y_t^K(M_{t-H-1:t-1}),v_t^K(M_{t-H-1:t})\right) - c_t\left(y_t^K(\tilde{M}_{t-H-1:t-1}),v_t^K(\tilde{M}_{t-H-1:t})\right) \notag\\
  \leq {} & G_cD\norm{y_t^K-\tilde{y}_t^K} + G_cD\norm{v_t^K-\tilde{v}_t^K}, \label{eq:oco-lip-1}
\end{align}
where for convenience we use the notations $y_t^K \define  y_t^K(\tilde{M}_{t-H-1:t-1}), \tilde{y}_t^K \define  y_t^K(\tilde{M}_{t-H-1:t-1})$ and $v_t^K \define v_t^K(M_{t-H-1:t}), \tilde{v}_t^K \define \tilde{v}_t^K(M_{t-H-1:t})$. Besides, the last inequality holds because the norm of $\norm{y_t^K}$, $\norm{\tilde{y}_t^K}$, $\norm{v_t^K}$, $\norm{\tilde{v}_t^K}$ are all bounded by $D$, as shown in Lemma~\ref{lemma:support-1-boundedness}.

Then we try to bound $\norm{y_t^K-\tilde{y}_t^K}$ and $\norm{v_t^K-\tilde{v}_t^K}$.
\begin{align}
  \norm{y_t^K-\tilde{y}_t^K} = {} & \left\| \sum_{i=0}^{2H} \left( \Psi_{t-1,i}^{K,H}(M_{t-H-1:t-1})-\Psi_{t-1,i}^{K,H}(\tilde{M}_{t-H-1:t-1}) \right) w_{t-1-i} \right\| \notag\\
  = {} & \left\| \tilde{A}_K^kB \sum_{i=0}^{2H} \left( M_{t-k}^{[i-k]}-\tilde{M}_{t-k}^{[i-k]} \right) \ind{i-k \in [H]} w_{t-1-i} \right\| \notag\\
  \le {} & \kappa_B \kappa^2 (1-\gamma)^k W\sum_{i=1}^{H}\norm{M_{t-k}^{[i]}-\tilde{M}_{t-k}^{[i]}} \notag\\
  \le {} & \kappa_B \kappa^2 W \norm{M_{t-k}-\tilde{M}_{t-k}}, \label{eq:oco-lip-2}
\end{align}
and we have
\begin{align}
  \norm{v_t^K-\tilde{v}_t^K}  = {} & \left\|-K(y_t^K-\tilde{y}_t^K) +\ind{k=0} \sum_{i=1}^{H}\left(M_{t-k}^{[i]}-\tilde{M}_{t-k}^{[i]}\right) \right\| \notag\\
  \le {} & (\kappa_B \kappa^3 W+1) \norm{M_{t-k}-\tilde{M}_{t-k}} \notag\\
  \le {} & 2\kappa_B \kappa^3 W \norm{M_{t-k}-\tilde{M}_{t-k}}. \label{eq:oco-lip-3}
\end{align}
Combining~\eqref{eq:oco-lip-1},~\eqref{eq:oco-lip-2}, and~\eqref{eq:oco-lip-3}, we obtain  
\begin{align*}
  f_t(M_{t-H-1:t}) - f_t(\tilde{M}_{t-H-1:t}) \le {} & G_cD\norm{y_t^K-\tilde{y}_t^K} + G_cD\norm{v_t^K-\tilde{v}_t^K}\\
  \le {} & G_cD\kappa_B \kappa^2 W \norm{M_{t-k}-\tilde{M}_{t-k}}+G_cD 2\kappa_B \kappa^3 W \norm{M_{t-k}-\tilde{M}_{t-k}}\\
  \le {} & 3G_cD\kappa_B \kappa^3 W \norm{M_{t-k}-\tilde{M}_{t-k}}.
\end{align*}
So we have $L_f \leq 3G_cDW\kappa_B\kappa^3$.

Next, we prove the claim (ii), i.e., the boundedness of the gradient norm. Indeed, we will try to bound $\nabla_{M^{[r]}_{p,q}} \tilde{f}_t(M)$ for every $p \in [d_u], q\in [d_x]$ and $r\in \{0,\dots,H-1\}$,
\begin{equation}
  \label{eq:oco-grad-1}
  \left| \nabla_{M^{[r]}_{p,q}} \tilde{f}_t(M)\right| \le G_c \left\| \frac{\partial y_t^K(M)}{\partial M^{[r]}_{p,q}} \right\|_{\mathrm{F}} + G_c \left\| \frac{\partial v_t^K(M)}{\partial M^{[r]}_{p,q}}\right\|_{\mathrm{F}}.
\end{equation}
So we will bound the two terms of the right-hand side respectively.
\begin{align}
  \left\|\frac{\partial y_t^K(M)}{\partial M^{[r]}_{p,q}}\right\|_{\mathrm{F}} \leq {} & \left\|\sum_{i=0}^{2H} \sum_{j=0}^H \left[ \frac{\partial \tilde{A}_K^j B M^{[i-j]}}{\partial M^{[r]}_{p,q}} \right] w_{t-1-i}\ind{i-j \in [H]}\right\|_{\mathrm{F}} \notag\\
  \leq {} & \sum_{i=r+1}^{r+H+1} \left\|\frac{\partial \tilde{A}_K^{i-r-1} B M^{[r]}}{\partial M^{[r]}_{p,q}} w_{t-1-i}\right\|_{\mathrm{F}} \notag\\
  \leq {} &  W \kappa_B \kappa^2 \left\|\frac{\partial M^{[r]}}{\partial M^{[r]}_{p,q}}\right\|_{\mathrm{F}} \sum_{i=r+1}^{r+H+1} (1-\gamma)^{i-r-1} \notag\\
  \leq {} & \frac{W\kappa_B \kappa^2}{\gamma} \left\|\frac{\partial M^{[r]}}{\partial M^{[r]}_{p,q}}\right\|_{\mathrm{F}} \leq \frac{W\kappa_B \kappa^2}{\gamma} \label{eq:oco-grad-2}
\end{align}

\begin{align}
  \left\|\frac{\partial v_t^K(M)}{\partial M^{[r]}_{p,q}}\right\|_{\mathrm{F}} \leq {} & \kappa \left\|\frac{\partial y_t^K(M)}{\partial M^{[r]}_{p,q}}\right\|_{\mathrm{F}} + \sum_{i=1}^H\left\|\frac{\partial M^{[i]}}{\partial M^{[r]}_{p,q}} w_{t-i}\right\|_{\mathrm{F}} \notag\\
  \leq {} & \frac{W\kappa_B \kappa^3}{\gamma} + W \left\|\frac{\partial M^{[r]}}{\partial M^{[r]}_{p,q}}\right\|_{\mathrm{F}} \leq W\left(\frac{\kappa_B \kappa^3}{\gamma}+1\right) \label{eq:oco-grad-3}
\end{align}
Combining~\eqref{eq:oco-grad-1},~\eqref{eq:oco-grad-2}, and~\eqref{eq:oco-grad-3}, we obtain  
\[
  \left| \nabla_{M^{[r]}_{p,q}} \tilde{f}_t(M)\right| \le G_c \frac{W\kappa_B \kappa^2}{\gamma} + G_c W\left(\frac{\kappa_B \kappa^3}{\gamma}+1\right) \leq 3G_cW\kappa_B \kappa^3\gamma^{-1}.
\]
Thus, $\Fnorm{\nabla_M \tilde{f}_t(M)}$ is at most $3Hd^2G_cW\kappa_B \kappa^3\gamma^{-1}$.

Finally, we prove the claim (iii), i.e., the upper bound of diameter of the feasible set. Actually, the construction of feasible set $\M$ ensures that $\forall i \in [H]$, $\opnorm{M}^{[i]}\le \kappa_B \kappa^3 (1-\gamma)^i$. Therefore, we have
\begin{align*}
  & \max_{M_1,M_2 \in \M} \Fnorm{M_1-M_2} \overset{(\textnormal{Lemma}~\ref{lemma:norm-relation})}{\leq} \sqrt{d} \max_{M_1,M_2 \in \M} \Loneopnorm{M_1-M_2}\\
  \leq {} & \sqrt{d} \max_{M_1,M_2 \in \M}(\Loneopnorm{M_1}+\Loneopnorm{M_2}) = \sqrt{d} \max_{M_1,M_2 \in \M}\left(\sum_{i=1}^{H} \opnorm{M_1^{[i]}} + \opnorm{M_2^{[i]}}\right)\\
  \leq {} & \sqrt{d} \max_{M_1,M_2 \in \M} \left(2\sum_{i=1}^{H}\kappa_B \kappa^3 (1-\gamma)^i\right) = 2\sqrt{d}\kappa_B \kappa^3 \sum_{i=1}^{H}(1-\gamma)^i \le 2\sqrt{d}\kappa_B \kappa^3 \gamma^{-1}.
\end{align*}
Hence, we finish the proof of all three claims in the statement.
\end{proof}

In the following, we show that minimizing the static policy regret over the DAC class is sufficient to deliver a policy regret competing with the strongly linear controller class.
\begin{myLemma}[Lemma 5.2 of \citet{ICML'19:online-control}]
\label{lemma:sufficiency}
With $K,K^\star$ chosen as the $(\kappa,\gamma)$-strongly stable linear controllers as defined in Definition~\ref{def:controller-set} and under Assumption~\ref{assume:cost-bound}, there exists a DAC policy $\pi(M_{\Delta}, K)$ with $M_{\Delta} = (M_{\Delta}^{[0]},\ldots,M_{\Delta}^{[H-1]})$ defined by
\begin{equation*}
  \label{eq:M-star-def}
  M_{\Delta}^{[i]} = (K-K^\star)(A - BK^\star)^i
\end{equation*}
such that
\begin{equation*}
  \label{eq:sufficiency}
  \sum_{t=1}^{T} c_t(x_t^K(M_{\Delta}), u_t^K(M_{\Delta})) - \sum_{t=1}^{T} c_t(x_t^{K^\star}, u_t^{K^\star}) \leq T \cdot 4G_c D W H \kappa_B^2 \kappa^6 (1-\gamma)^{H-1}\gamma^{-1},
\end{equation*}
where $x_t^{K^\star}$ is the state attained by executing a linear controller $K^\star$ which chooses the action $u_t^{K^\star} = - K^\star x_t^{K^\star}$.
\end{myLemma}
\begin{proof}[{of Lemma~\ref{lemma:sufficiency}}]
The coordinate-wise Lipschitzness of the cost functions implies that
\begin{align*}
  c_t\left(x_t^K(M_{\Delta}), u_t^K(M_{\Delta})\right) - c_t\left(x_t^{K^\star}, u_t^{K^\star}\right) \leq G_cD\left\| x_t^K(M_{\Delta})-x_t^{K^\star} \right\| + G_c D\left\| u_t^K(M_{\Delta})- u_t^{K^\star} \right\|.
\end{align*}
By the linear dynamical equation~\eqref{eq:dynamics}, we have 
\begin{equation}
  \label{eq:proof-K*-state}
  x_{t+1}^{K^\star} = \sum_{i=0}^{t} (A - BK^\star)^i w_{t-i}= \sum_{i=0}^{t} {\tilde{A}_{K^\star}}^i w_{t-i}
\end{equation}

By the property of the DAC policy (Proposition~\ref{proposition:DAC-state}), we have
\[
  x_{t+1}^K(M_{\Delta}) = \tilde{A}_K^{h+1} x_{t-h}^K(M_{\Delta}) + \sum_{i=0}^{H + h} \Psi_{t,i}^{K,h}(M_{\Delta})w_{t-i}.
\]
Setting $h = t$ and combining the assumption that the starting state $x_0 = \mathbf{0}$, we achieve the following equation,
\[
  x_{t+1}^K(M_{\Delta}) = \sum_{i=0}^{H} \Psi_{t,i}^{K,t}(M_{\Delta}) w_{t-i} + \sum_{i=H+1}^{t} \Psi_{t,i}^{K,t}(M_{\Delta}) w_{t-i}.
\]
Now we turn to calculate the transfer matrix $\Psi_{t,i}^{K,h}(M_{\Delta})$ explicitly. Actually, for any $i \in \{0,\ldots,H\}$, $h \geq H$, i.e., $0\le i \le H \le h$, by definition we have
\begin{align}
  \Psi_{t,i}^{K,h}(M_{\Delta}) = {} & \tilde{A}_K^i \ind{i \leq h} + \sum_{j=0}^{h} \tilde{A}_K^j B M_{\Delta}^{[i-j]} \ind{i - j \in [H]}\nonumber\\
  = {} & \tilde{A}_K^i + \sum_{k=1}^{i} \tilde{A}_K^{i-k} B M_{\Delta}^{[k]} \label{eq:sufficiency-support-1}\\
  = {} & \tilde{A}_K^i + \sum_{k=1}^{i} \tilde{A}_K^{i-k} B (K - K^\star) \tilde{A}_{K^\star}^{k-1}\label{eq:sufficiency-support-2}\\
  = {} & \tilde{A}_K^i + \sum_{k=1}^{i} \tilde{A}_K^{i-k} (\tilde{A}_{K^\star} - \tilde{A}_{K}) \tilde{A}_{K^\star}^{k-1}\nonumber\\
  = {} & \tilde{A}_K^i + \sum_{k=1}^{i} \tilde{A}_K^{i-k}\tilde{A}_{K^\star}^k - \tilde{A}_K^{i-k+1}\tilde{A}_{K^\star}^{k-1}\nonumber\\  
  = {} & \tilde{A}_K^i + \tilde{A}_{K^\star}^{i} - \tilde{A}_K^{i}\nonumber\\
  = {} & \tilde{A}_{K^\star}^{i}\nonumber,
\end{align}
where~\eqref{eq:sufficiency-support-1} holds by introducing a new index $k = i-j$ and~\eqref{eq:sufficiency-support-2} can be obtained by plugging the construction of $M_{\Delta}^{[i]}$~\eqref{eq:M-star-def}. So we achieve the conclusion that
\begin{equation}
  \label{eq:proof-M*-state}
  x_{t+1}^K(M_{\Delta}) = \sum_{i=0}^{H} \tilde{A}_{K^\star}^{i} w_{t-i} + \sum_{i=H+1}^{t} \Psi_{t,i}^{K,t}(M_{\Delta}) w_{t-i}.
\end{equation}
Combining~\eqref{eq:proof-K*-state} and~\eqref{eq:proof-M*-state} yields 
\begin{align*}
\left\| x_{t+1}^{K^\star} - x_{t+1}^K(M_{\Delta})\right\| = {} & \left\| \sum_{i=H+1}^{t} \big(\Psi_{t,i}^{K,t}(M_{\Delta}) - \tilde{A}_{K^\star}^{i}\big) w_{t-i} \right\|\\
  \leq {} & W \left(\sum_{i=H+1}^{t} \norm{\Psi_{t,i}^{K,t}(M_{\Delta})} + \sum_{i=H+1}^{t} \norm{\tilde{A}_{K^\star}^{i}}\right)\\
  \le {} & W\left(\sum_{i=H+1}^t \left(2\kappa^2(1-\gamma)^i+H\kappa_B^2 \kappa^5 (1-\gamma)^{i-1}\right)\right)\\
  \le {} & W\left(2\kappa^2 (1-\gamma)^{H+1}\gamma^{-1}+H\kappa_B^2 \kappa^5 (1-\gamma)^H\gamma^{-1}\right)\\
  \le {} & \kappa^2W(1-\gamma)^H\gamma^{-1}\left(2(1-\gamma)+H\kappa_B^2 \kappa^3\right)\\
  \le {} & H\kappa_B^2 \kappa^5 W(1-\gamma)^H\gamma^{-1} (2(1-\gamma)+1)\\
  \leq {} & 2W H \kappa_B^2 \kappa^5 (1-\gamma)^H\gamma^{-1},
\end{align*}
where the second inequality makes use of Lemma~\ref{lemma:support-1-boundedness}. Next, we investigate the difference between the control signals,
\begin{align*}
  \norm{u_{t+1}^{K^\star} - u_{t+1}^K(M_{\Delta})} =  {} & \left\|-K^\star x_{t+1}^{K^\star} - \left( - K x_{t+1}^K(M_{\Delta}) + \sum_{i=1}^{H} M_{\Delta}^{[i]} w_{t+1-i}\right)\right\| \\
  = {} & \left\|-K^\star x_{t+1}^{K^\star} + K x_{t+1}^K(M_{\Delta}) - \sum_{i=1}^{H}(K-K^\star) \tilde{A}_{K^\star}^{i-1} w_{t+1-i}\right\| \\
  = {} & \left\|-K^\star\left(x_{t+1}^{K^\star}-\sum_{i=0}^{H-1} \tilde{A}_{K^\star}^{i} w_{t-i} \right)+K\left(x_{t+1}^K(M_{\Delta}) -\sum_{i=0}^{H-1} \tilde{A}_{K^\star}^{i}w_{t-i} \right)\right\| \\
  = {} & \left\|-K^\star \sum_{i=H}^{t} \tilde{A}_{K^\star}^i w_{t-i} + K \sum_{i=H}^{t} \Psi_{t,i}^{K,h}(M_{\Delta})w_{t-i}\right\| \\
  \leq {} & 2WH \kappa_B^2 \kappa^6 (1-\gamma)^{H-1}\gamma^{-1}.
\end{align*}
Using above inequalities and Lipschitz assumption as well as the boundedness result (Lemma~\ref{lemma:support-1-boundedness}), we complete the proof.
\end{proof}

The remaining part of this section lists useful supporting lemmas for studying non-stochastic control in unknown systems. 
Lemma~\ref{lem:moment-recovery} gives a high-probability bound about the estimation accuracy in unknown systems.
\begin{myLemma}[Moment Recovery~{\citep[Lemma 21]{ALT'20:control-Hazan}}]
  \label{lem:moment-recovery}
  Under Assumption~\ref{assume:strongly-stable}, Algorithm~\ref{alg:system-identification} satisfies for all $j \in [k]$, with probability at least $1-\delta$, it holds that
  \begin{equation}
    \label{eq:concentration}
      \Fnorm{N_j - \tilde{A}_K^{j} B} \le 3 \kappa_B\kappa^2 d_u W \gamma^{-1} \sqrt{\frac{2d_{\min}\log\sbr{2e^2k\delta^{-1}}}{T_0-k}}.
  \end{equation}
\end{myLemma}
\begin{proof}[of Lemma~\ref{lem:moment-recovery}]
  When the control inputs are chosen as $u_t = -K x_t + \tilde{u}_t$, using the transition equation of linear dynamical systems, it holds that
  \begin{equation*}
    \begin{split}
        x_{t+1} = {} & A x_t + B u_t + w_t = A x_t + B \sbr{-K x_t + \tilde{u}_t} + w_t = \tilde{A}_K x_t + B \tilde{u}_t + w_t\\
        = {} & \tilde{A}_K \sbr{A x_{t-1} + B u_{t-1} + w_{t-1}} = \tilde{A}_K \sbr{\tilde{A}_K x_{t-1} + B \tilde{u}_{t-1} + w_{t-1}} + B \tilde{u}_t + w_t\\
        = {} & \tilde{A}_K^2 x_{t-1} + \tilde{A}_K \sbr{B \tilde{u}_{t-1} + w_{t-1}} + \sbr{B \tilde{u}_t + w_t}= \ldots\\
        = {} & \sum_{i=0}^t \tilde{A}_K^{t-i} \sbr{B \tilde{u}_i + w_i}.
    \end{split}
  \end{equation*}
  Let $N_{j,t} = x_{t+j+1}\tilde{u}_t^\T$, we can prove that 
  \begin{align*}
    \E\mbr{N_{j,t}} = {} & \E\mbr{x_{t+j+1}\tilde{u}_t^\T} = \E\mbr{\sum_{i=0}^{t+j} \tilde{A}_K^{t+j-i} \sbr{B \tilde{u}_i + w_i}\tilde{u}_t^\T} \\
    = {} & \sum_{i=0}^{t+j} \tilde{A}_K^{t+j-i} \cdot \E\mbr{\sbr{B \tilde{u}_i + w_i}\tilde{u}_t^\T} = \tilde{A}_K^{j} \cdot \E\mbr{\sbr{B \tilde{u}_t + w_t}\tilde{u}_t^\T}\\
    = {} & \tilde{A}_K^{j} B \cdot \E\mbr{\tilde{u}_t \tilde{u}_t^\T} + \tilde{A}_K^{j} w_t \cdot \E\mbr{\tilde{u}_t^\T} =  \tilde{A}_K^{j} B,
  \end{align*}
  where the second last equation is due to the fact that $\tilde{u}_i$ and $\tilde{u}_j$ are independent when $i \ne j$, and the last step is true because $\E_{\tilde{u}_t}\mbr{\tilde{u}_t \tilde{u}_t^\T} = I, \E_{\tilde{u}_t}\mbr{\tilde{u}_t} = \mathbf{0}$. Consequently, we can prove that $\E[N_j] = \frac{1}{T_0-k}\sum_{t=0}^{T_0-k-1} \E\mbr{N_{j,t}} = \tilde{A}_K^{j} B$. Note that for $0 \le t_1,t_2 \le T_0-k-1$ and $t_1 \ne t_2$, $N_{j,t_1}$ and $N_{j,t_2}$ are not independent because they contains the same random variables $\eta$, so we cannot use Hoeffding's inequality here.

  For each index $j\in [k]$, we can define a sequence of variables $\tilde{N}_{j,t} \define N_{j,t} - \tilde{A}_K^{j} B$, we can prove that $\{\tilde{N}_{j,t}\}_{t=0}^{T_0-k-1}$ is a \emph{martingale difference sequence} w.r.t. the sequence $\{\tilde{u}_t\}_{t=0}^{T_0-k-1}$:
  \begin{align*}
    & \E\mbr{\tilde{N}_{j,t} \givenn\tilde{u}_{0:t-1}}
    = \E\mbr{N_{j,t}\givenn\tilde{u}_{0:t-1}} - \tilde{A}_K^{j} B\\
    = {} & \E\mbr{\sum_{i=0}^{t+j} \tilde{A}_K^{t+j-i} \sbr{B \tilde{u}_i + w_i}\tilde{u}_t^\T \givenn \tilde{u}_{0:t-1}} - \tilde{A}_K^{j} B\\
    = {} & \E\mbr{\sum_{i=0}^{t-1} \tilde{A}_K^{t+j-i} \sbr{B \tilde{u}_i + w_i}\tilde{u}_t^\T \givenn \tilde{u}_{0:t-1}} + \E\mbr{\sum_{i=t}^{t+j} \tilde{A}_K^{t+j-i} \sbr{B \tilde{u}_i + w_i}\tilde{u}_t^\T} - \tilde{A}_K^{j} B\\
    = {} & \E\mbr{\tilde{A}_K^{j} \sbr{B \tilde{u}_t + w_t}\tilde{u}_t^\T} - \tilde{A}_K^{j} B = \mathbf{0}.
  \end{align*}
  For all $j \in [k], t=0,\ldots,T_0-k-1$, the operator norm of $N_{j,t}$ can be bounded by
  \begin{equation*}
      \opnorm{N_{j,t}} \le \opnorm{x_{t+j+1}}\opnorm{\tilde{u}_t} \le \norm{x_{t+j+1}}_2\norm{\tilde{u}_t}_2 \le 2 \kappa_B\kappa^2 \sqrt{d_u} W \gamma^{-1} \cdot \sqrt{d_u} = 2 \kappa_B\kappa^2 d_u W \gamma^{-1}.
  \end{equation*}
  Also, for $\tilde{N}_{j,t}$, we can prove that 
  \begin{gather*}
    \opnorm{\tilde{N}_{j,t}} \le \opnorm{N_{j,t}} + \opnorm{\tilde{A}_K^{j} B} \le 2 \kappa_B\kappa^2 d_u W \gamma^{-1} + \kappa_B\kappa^2 (1-\gamma)^j \le 3 \kappa_B\kappa^2 d_u W \gamma^{-1}, \\
    \Fnorm{\tilde{N}_{j,t}} \le \sqrt{d_{\min}} \opnorm{\tilde{N}_{j,t}} \le 3 \sqrt{d_{\min}} \kappa_B\kappa^2 d_u W \gamma^{-1} \define D_N.
  \end{gather*}
  Using Lemma~\ref{suplem:vector-azuma}, we have $\Pr\mbr{\Fnorm{ \sum_{t=0}^{T_0-k} \tilde{N}_{j,t}} \ge x} \le 2 e^2\exp\sbr{\frac{-x^2}{2(T_0-k) D_N^2}}$. By substituting $\tilde{N}_{j,t}$ by $N_{j,t}-\tilde{A}_K^{j} B$, it holds that $\Pr\mbr{\Fnorm{N_j - \tilde{A}_K^{j} B} \ge \frac{x}{T_0-k}} \le 2 e^2\exp\sbr{\frac{-x^2}{2(T_0-k) D_N^2}}$. Finally, let $\epsilon = \frac{x}{T_0-k}$, we have
  \begin{equation*}
    \Pr\mbr{\Fnorm{N_j - \tilde{A}_K^{j} B} \ge \epsilon} \le 2 e^2\exp\sbr{\frac{-(T_0-k)\epsilon^2}{2D_N^2}}
  \end{equation*}
  We set $2 e^2\exp\sbr{\frac{-(T_0-k)\epsilon^2}{2D_N^2}} = \frac{\delta}{k}$ to make above concentration inequality holds for each $j \in [k]$ with probability at least $1-\delta$, which implies that 
  \begin{equation*}
    \epsilon = 3 \kappa_B\kappa^2 d_u W \gamma^{-1} \sqrt{\frac{2d_{\min}\log\sbr{2e^2k\delta^{-1}}}{T_0-k}}.
  \end{equation*}
Hence, we complete the proof.
\end{proof}

\begin{myLemma}[Preservation of Stability]
  \label{suplem:stability-preserve}
  Under Assumption~\ref{assume:strongly-stable}, if $K$ is $(\kappa,\gamma)$-strongly stable for a linear dynamical system $S=(A,B,\{w\})$, i.e., $A-BK = QLQ^{-1}$, and $\|A - \Ah\|_{\mathrm{F}}, \|A - \Ah\|_{\mathrm{F}} \le \epsilon_{A,B}$, then the same linear controller $K$ is $\sbr{\kappa, \gamma-2\kappa^3 \epsilon_{A,B}}$-strongly stable for the estimated system $\Sh = (\Ah, \Bh, \{\wh\})$, i.e., $\Ah - \Bh K = Q \hat{L} Q^{-1}$, where $\norm{\hat{L}} \le 1 - \gamma + 2\kappa^3 \epsilon_{A,B}$.
\end{myLemma}
\begin{proof}[of Lemma~\ref{suplem:stability-preserve}]
  First, we try to express the strong stability of $K$ with respect to $(\Ah,\Bh)$ as
  \begin{align*}
      \Ah - \Bh K = {} & A - BK + (\Ah-A) - (\Bh - B)K\\
      = {} & Q L Q^{-1} + (\Ah-A) - (\Bh - B)K \\
      = {} & Q \sbr{L + Q^{-1}\sbr{(\Ah-A) - (\Bh - B)K}Q } Q^{-1} \define \hat{Q} \hat{L} \hat{Q}^{-1},
  \end{align*}
  where the last equality is by defining $\hat{L} = L + Q^{-1}((\Ah-A) - (\Bh - B)K)Q$. Further, the operator norm of $\hat{L}$ can be bounded as
  \begin{align*}
      \|\hat{L}\|_{\mathrm{op}} = {} & \opnorm{L + Q^{-1}\sbr{(\Ah-A) - (\Bh - B)K}Q}\\
      \le {} & \opnorm{L} + \opnorm{Q^{-1}} \sbr{\|\Ah-A\|_{\mathrm{op}} + \opnorm{K}\|\Bh-B\|_{\mathrm{op}}} \opnorm{Q}\\
      \le {} & (1 - \gamma) + \kappa \cdot \sbr{\epsilon_{A,B}+\kappa \cdot \epsilon_{A,B}} \cdot \kappa \le 1 - \gamma + 2\kappa^3 \epsilon_{A,B}.
  \end{align*}
  By definition of strong stability, it holds that $K$ is $\sbr{\kappa, \gamma-2\kappa^3 \epsilon_{A,B}}$-strongly stable for the estimated system $\Sh = (\Ah, \Bh, \{\wh\})$.
\end{proof}

Lemma~\ref{suplem:W-bound} below provides boundedness results in the fictitious system.
\begin{myLemma}[Lemma 18 of \citet{ALT'20:control-Hazan}]
  \label{suplem:W-bound}
  Under Assumption~\ref{assume:bound-noise} and Assumption~\ref{assume:strongly-stable}, if it holds that $\epsilon_{A,B} \le 10^{-3} \kappa^{-10} \gamma^2$, then for any $t \ge T_0 + 1$, we have 
  \begin{equation*}
      \norm{x_t}_2 \le 20 \sqrt{d_u} \kappa^{11} \gamma^{-3} W,\quad \norm{w_t - \wh_t}_2 \le 42 \sqrt{d_u} \kappa^{12} \gamma^{-3} W \epsilon_{A,B},\quad \norm{\wh_{t-1}}_2 \le 2 \sqrt{d_u} \kappa^3 \gamma^{-1} W.
  \end{equation*}
\end{myLemma}

\begin{myLemma}
  \label{suplem:strong-controllability}
  Under Assumption~\ref{assume:strong-controllability}, $\sigma_{\min}(C_k) \ge 1/\sqrt{\kappa_c}$, where $C_k$ is defined in~\eqref{eq:C_k}.
\end{myLemma}
\begin{proof}[of Lemma~\ref{suplem:strong-controllability}]
  Under Assumption~\ref{assume:strong-controllability}, it holds that $\opnorm{(C_kC_k^\T)^{-1}}\le \kappa_c$, i.e., 
  \begin{equation*}
      \sigma_{\max}((C_kC_k^\T)^{-1}) \le \kappa_c.
  \end{equation*}
  It is apparent that $\sbr{(C_kC_k^\T)^{-1}}^\T = \sbr{(C_kC_k^\T)^\T}^{-1} = (C_kC_k^\T)^{-1}$, i.e., $(C_kC_k^\T)^{-1}$ is a symmetric matrix. Then we have
  \begin{align*}
      \sigma_{\max}((C_kC_k^\T)^{-1}) = {} & \lambda_{\max}\sbr{(C_kC_k^\T)^{-1} \sbr{(C_kC_k^\T)^{-1}}^\T} = \lambda_{\max}\sbr{(C_kC_k^\T)^{-1}(C_kC_k^\T)^{-1}}\\
      = {} & \lambda_{\max}^2 \sbr{(C_kC_k^\T)^{-1}} \le \kappa_c.
  \end{align*}
  Finally we have $\sigma_{\min}(C_k) = \lambda_{\min}(C_kC_k^\T) \ge 1/\sqrt{\kappa_c}$, which finishes the proof.
\end{proof}

\begin{myLemma}[{Lemma 17 of~\citet{ALT'20:control-Hazan}}]
  \label{suplem:matrix-minus-norm-sum}
  For any matrix pair $L,\hat{L}$, such that $\opnorm{L}, \opnorm{\hat{L}} \le 1- \gamma, \gamma \in (0,1)$, we have $\sum_{t=0}^\infty \opnorm{L^t - \hat{L}^t} \leq 3\gamma^{-2} \opnorm{L - \hat{L}}$.
\end{myLemma}

\begin{myLemma}[Perturbation Analysis~{\citep[Lemma 22]{ALT'20:control-Hazan}}]
  \label{suplem:perturbation-analysis}
  Let $x^\star$ be the solution to linear system $Ax=b$, and $\hat{x}$ be the solution to $(A+\Delta A) x = b + \Delta b$, then if it holds that $\norm{\Delta A} \le \sigma_{\min}(A)$, it is true that
  \begin{equation*}
      \norm{x^\star - \hat{x}} \le \frac{\norm{\Delta b} + \norm{\Delta A}\norm{x^\star}}{\sigma_{\min}(A) - \opnorm{\Delta A}}.
  \end{equation*}
\end{myLemma} 
\bibliography{online_learning}

\end{document}